\pdfoutput=1
\documentclass{article}
\usepackage{iclr2027_conference,times}

\usepackage[utf8]{inputenc}
\usepackage[T1]{fontenc}
\usepackage{microtype}
\usepackage{amsmath,amssymb}
\usepackage{amsthm}
\newtheorem{proposition}{Proposition}
\theoremstyle{remark}
\newtheorem{remark}{Remark}
\usepackage{xspace}
\newcommand{\bench}{RealCompanion\xspace}
\usepackage{graphicx}
\usepackage{booktabs}
\usepackage{enumitem}
\usepackage{multirow}
\usepackage{tabularx}
\usepackage{xurl}
\usepackage{xcolor}
\usepackage{tikz}
\usepackage{pgfplots}\pgfplotsset{compat=1.17}\usetikzlibrary{patterns}
\usetikzlibrary{arrows.meta,positioning,fit,backgrounds,calc,shapes.geometric,decorations.pathreplacing}
\usepackage{titletoc}
\usepackage{listings}
\usepackage{hyperref}
\definecolor{cData}{HTML}{2C5F8A}
\definecolor{cProf}{HTML}{6B4E9B}
\definecolor{cGT}{HTML}{1F7A5C}
\definecolor{cBad}{HTML}{B03A2E}
\definecolor{cGrey}{HTML}{6E6E6E}
\definecolor{cWash}{HTML}{F4F6F9}
\hypersetup{colorlinks=true, linkcolor=cData, citecolor=cData, urlcolor=cData,
  pdftitle={RealCompanion: Benchmarking Human Understanding from Reasoning over Longitudinal Real-World Conversations},
  pdfauthor={Arman Behnam, Sunglyoung Kim, Liangwei Yang}}
\lstdefinestyle{rcitem}{
  basicstyle=\ttfamily\scriptsize,
  frame=none,
  showstringspaces=false, breaklines=false,
  columns=fullflexible, keepspaces=true,
  moredelim=**[is][\color{cBad}\bfseries]{@}{@},
}

\title{\centering RealCompanion: Benchmarking Human\\ Understanding from Reasoning over\\ Longitudinal Real-World Conversations}

\author{}

\iclrfinalcopy
\begin{document}

\maketitle
\lhead{Preprint}

\vspace{-2.2em}
\begin{center}
\begin{tabular}{@{}c@{\hspace{2.5em}}c@{\hspace{2.5em}}c@{}}
\textbf{Arman Behnam} & \textbf{Sunglyoung Kim} & \textbf{Jiayi Yu} \\
Quis Lab, USA & Quis Lab, USA & Quis Lab, USA \\
\texttt{arman@quis.ai} & \texttt{sk@quis.ai} & \texttt{jy@quis.ai}
\end{tabular}

\vspace{0.8em}

\begin{tabular}{@{}c@{\hspace{2.5em}}c@{}}
\textbf{Eric Huang} & \textbf{Liangwei Yang} \\
Quis Lab, USA & Independent Researcher, USA \\
\texttt{eric@quis.ai} & \texttt{liangwei\_yang@outlook.com}
\end{tabular}
\end{center}
\vspace{0.5em}

\begin{abstract}
An AI companion that talks with someone for months should come to understand them. It should know who they are, remember what they said, and recognize when something from the past matters now. Testing this needs real conversations, but real conversations are private, so existing benchmarks use invented people and invented questions. We release \bench, ten real relationships between people and an AI companion, with 27{,}218 messages over up to 120 days. For each person we release the full conversation, a profile, a persona, chat test items and question test items. Every label points to the messages that support it, and every chat label comes with the reasoning that produced it. The real data shows three things. First, people rarely refer back. Only 3.4\% of their messages depend on something said earlier, and when one does, the earlier message is usually far away (a median of 2{,}157 messages back). Averages hide this. Looking at the most recent messages finds the needed one 95.9\% of the time overall, but only 2.2\% of the time when it is far back. Second, AI systems cannot tell when the past matters. The detectors we tested barely beat chance on real messages, and when the same earlier messages are labeled ``memories'' instead of ``earlier messages'', models bring up the past 10 to 14 percentage points more often, even when nothing from the past is needed. Third, AI systems read more into a person than the person revealed. Three agent systems rebuild each persona equally well (F1 0.71). They see the person, and then imagine more. Understanding a person depends on knowing when their past matters and where what they shared ends, and only real conversations can test it.
\end{abstract}
\section{Introduction}
\label{sec:intro}

% P1 - understanding a person means answering two questions
People now spend a long time talking to AI companions. Understanding such a person means answering two questions. The first is who they are, and that answer holds from one message to the next. The second is which of the things they have said matters now, and that answer changes with every message. A companion has to get both right. Both answers are claims about a real person, and a claim about a person can only be checked against what that person actually said. Those conversations are private, held by the companies whose products produced them, so the conversations used to test this have been generated instead. \bench\footnote{Data: \url{https://huggingface.co/datasets/Quislab/RealCompanion}. Code and reproduction map: \url{https://anonymous.4open.science/r/realcompanion-54CC/}.} releases ten real ones, with both answers written down beside them. Figure~\ref{fig:understanding} shows the two questions on one of them.

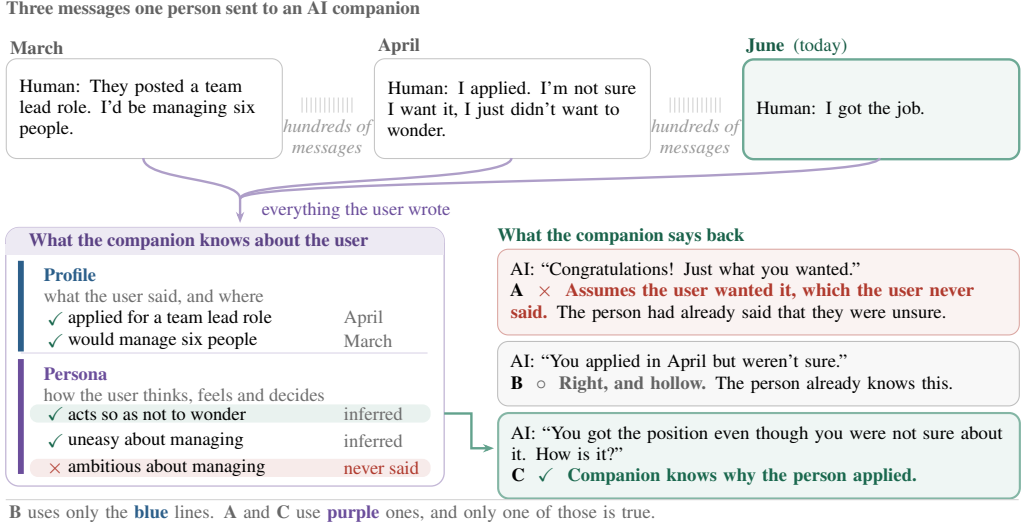
\begin{figure}[t]
\centering
\begin{tikzpicture}[
  font=\scriptsize,
  msg/.style  ={anchor=north west, inner sep=5pt, draw=cGrey!40, fill=white,
                rounded corners=4pt, align=left, text=black, text width=3.30cm,
                minimum height=1.32cm},
  now/.style  ={anchor=north west, inner sep=5pt, draw=cGT!65, fill=cGT!6,
                line width=0.9pt, rounded corners=4pt, align=left, text=black,
                text width=3.30cm, minimum height=1.32cm},
  when/.style ={anchor=south west, inner sep=0pt, font=\scriptsize\bfseries,
                text=cGrey!85!black},
  gap/.style  ={anchor=north, inner sep=0pt, font=\scriptsize\itshape, text=cGrey!70},
  ttl/.style  ={anchor=north west, inner sep=0pt, align=left,
                font=\scriptsize\bfseries, text=cGrey!90!black},
  sml/.style  ={anchor=north west, inner sep=0pt, align=left,
                font=\scriptsize, text=cGrey},
  merge/.style={draw=cProf!55, line width=0.9pt},
  rep/.style  ={anchor=north west, inner sep=5pt, rounded corners=4pt,
                align=left, text=black, text width=6.55cm}
]

% ======================= BAND 1 : her messages ======================
\node[ttl] at (0,3.40) {Three messages one person sent to an AI companion};
% \node[sml] at (6,3.40) {months apart, with hundreds of others in between};

\node[when] at (0.05,2.68) {March};
\node[msg] at (0,2.64)
  {Human: They posted a team lead role. I'd be managing six people.};

\node[when] at (4.92,2.68) {April};
\node[msg] at (4.87,2.64)
  {Human: I applied. I'm not sure I want it, I just didn't want to wonder.};

\node[when, text=cGT!75!black] at (9.79,2.68) {June\ \ \textnormal{(today)}};
\node[now] at (9.74,2.64) {Human: I got the job.};

\foreach \x in {3.92,8.79}{
  \foreach \j in {0,...,11}{
    \draw[cGrey!28, line width=0.45pt] ({\x+\j*0.060},1.92) -- ({\x+\j*0.060},2.12);}
  \node[gap, align=center] at ({\x+0.33},1.84) {hundreds of\\messages};
}

% ======================= BAND 2 : gathering =========================
\draw[merge] (1.82,1.32) .. controls (1.82,1.02) and (3.10,1.06) .. (3.10,0.78);
\draw[merge] (6.69,1.32) .. controls (6.69,0.98) and (3.10,1.06) .. (3.10,0.78);
\draw[merge] (11.56,1.32) .. controls (11.56,0.98) and (3.10,1.06) .. (3.10,0.78);
\draw[-{Stealth[length=4.2pt]}, merge] (3.10,0.78) -- (3.10,0.46);
\node[sml, anchor=west, text=cProf] at (3.38,0.62) {everything the user wrote};

% ======================= BAND 3 LEFT : the model of her =============
\def\bt{0.40}\def\bb{-3.02}
\draw[draw=cProf!40, fill=white, rounded corners=5pt, line width=0.5pt]
  (0,\bt) rectangle (5.80,\bb);
\fill[cProf!12, rounded corners=5pt] (0,\bt) rectangle (5.80,{\bt-0.36});
\fill[cProf!12] (0,{\bt-0.22}) rectangle (5.80,{\bt-0.36});
\node[anchor=west, font=\scriptsize\bfseries, text=cProf!85!black]
  at (0.18,{\bt-0.18}) {What the companion knows about the user};

% ---- profile ----
\draw[cData, line width=2.2pt] (0.20,-0.04) -- (0.20,-1.24);
\node[anchor=north west, font=\scriptsize\bfseries, text=cData] at (0.38,-0.02) {Profile};
\node[anchor=north west, font=\scriptsize, text=cGrey] at (0.38,-0.30)
  {what the user said, and where};
\foreach \t/\src/\k in {%
  {applied for a team lead role}/{April}/0,
  {would manage six people}/{March}/1}{
  \node[anchor=west, font=\scriptsize, text=cGT!85!black] at (0.42,{-0.80-\k*0.30}) {\checkmark};
  \node[anchor=west, font=\scriptsize] at (0.70,{-0.80-\k*0.30}) {\t};
  \node[anchor=west, font=\scriptsize, text=cGrey] at (4.35,{-0.80-\k*0.30}) {\src};
}

\draw[cProf!25, line width=0.4pt] (0.20,-1.26) -- (5.58,-1.26);

% ---- persona ----
\draw[cProf, line width=2.2pt] (0.20,-1.34) -- (0.20,-2.86);
\node[anchor=north west, font=\scriptsize\bfseries, text=cProf] at (0.38,-1.32) {Persona};
\node[anchor=north west, font=\scriptsize, text=cGrey] at (0.38,-1.60)
  {how the user thinks, feels and decides};

\fill[cGT!10, rounded corners=3pt] (0.32,-1.94) rectangle (5.62,-2.18);
\node[anchor=west, font=\scriptsize, text=cGT!85!black] at (0.42,-2.06) {\checkmark};
\node[anchor=west, font=\scriptsize] at (0.70,-2.06) {acts so as not to wonder};
\node[anchor=west, font=\scriptsize, text=cGrey] at (4.35,-2.06) {inferred};

\node[anchor=west, font=\scriptsize, text=cGT!85!black] at (0.42,-2.42) {\checkmark};
\node[anchor=west, font=\scriptsize] at (0.70,-2.42) {uneasy about managing};
\node[anchor=west, font=\scriptsize, text=cGrey] at (4.35,-2.42) {inferred};

\fill[cBad!10, rounded corners=3pt] (0.32,-2.66) rectangle (5.62,-2.90);
\node[anchor=west, font=\scriptsize, text=cBad] at (0.42,-2.78) {$\times$};
\node[anchor=west, font=\scriptsize] at (0.70,-2.78) {ambitious about managing};
\node[anchor=west, font=\scriptsize, text=cBad] at (4.35,-2.78) {never said};

% ======================= BAND 3 RIGHT : three replies ===============
\node[ttl, text=cGT!80!black] at (6.50,\bt) {What the companion says back};

\node[rep, draw=cBad!45, fill=cBad!6] at (6.50,0.12)
  {AI: ``Congratulations! Just what you wanted.''\\
  \textbf{A}\ \ \textcolor{cBad}{\textbf{$\times$\ \ Assumes the user wanted it, which the user never said.}} The person had already said that they were unsure.\\[2pt]};

\node[rep, draw=cGrey!45, fill=cGrey!5] at (6.50,-1.10)
  {AI: ``You applied in April but weren't sure.''\\
  \textbf{B}\ \ \textcolor{cGrey!90!black}{\textbf{$\circ$\ \ Right, and hollow.}} The person already knows this.\\[2pt]};

\node[rep, draw=cGT!55, fill=cGT!7, line width=0.8pt] at (6.50,-2.04)
  {AI: ``You got the position even though you were not sure about it. How is it?''\\
  \textbf{C}\ \ \textcolor{cGT!85!black}{\textbf{\checkmark\ \ Companion knows why the person applied.}}\\[2pt]};

\draw[-{Stealth[length=4.2pt]}, draw=cGT!70, line width=0.9pt]
  (5.80,-2.06) -- (6.16,-2.06) -- (6.16,-2.52) -- (6.46,-2.52);

% ======================= footer =====================================
\draw[cGrey!30, line width=0.4pt] (0,-3.20) -- (13.40,-3.20);
\node[anchor=north west, inner sep=0pt, font=\scriptsize, text=cGrey!95!black,
      align=left, text width=13.30cm]
  at (0.04,-3.26) {\textbf{B} uses only the \textcolor{cData}{\textbf{blue}} lines.
  \textbf{A} and \textbf{C} use \textcolor{cProf}{\textbf{purple}} ones, and only one
  of those is true.};

\end{tikzpicture}
\caption{A user's messages to an AI companion, and what the companion built from them. Reply A uses a persona line the companion got wrong, so it congratulates the user on a job they were unsure about. Reply B uses only the profile, so it repeats what the user already said. Reply C uses a correct persona line and the profile, so it speaks to why the user applied.}

\label{fig:understanding}
\end{figure}

% P2 - what generated conversations can measure
Generated conversations support one kind of measurement. They place a fact in a dialog and score whether a system returns it, which is memorization. Three properties of how they are built set that limit. A generated dialog is written around the memory it tests, so nearly every message in it needs the past \citep{maharana2024locomo,wu2024longmemeval,hu2026clonemem}. A prompted persona declares its traits in advance, so a question about the person is answered by reading the declaration \citep{zhang2018personachat,xu2022goldfish,jiang2025personamem, kim2025cupid}. A scripted history never revises itself, so nothing said in month nine changes what was established in month four \citep{jiang2025personamemv2,li2026horizonbench}. These benchmarks answer both questions before a system sees them. The dialog is built so that the past always matters. The persona states in advance what the person is like.

% P3 - four challenges
Answering those two questions means meeting four challenges. Challenge 1 is \emph{knowing when the past matters}, since most messages need nothing from earlier. Challenge 2 is \emph{knowing what the messages add up to}, since who a person is emerges across all of them together. Challenge 3 is \emph{drawing on everything at once}. A reply built on the last few messages misses what the person has carried for months. A reply built on the facts they stated hands those facts back to them. A reply built on a trait inferred wrongly is confident and wrong. Challenge 4 is \emph{telling whether any of this was done right}, and that takes the person's own record.

% P4 - what RealFriend supplies
\bench is built from ten people in long-term relationships with AI companions, released with their consent. These conversations were never written to test a system, so most messages ask nothing of the past (challenge 1). No one described these people before they spoke, so who they are has to be worked out from what they said (challenge 2). We release five files for each participant. The conversation is the source, and the other four come from it. A \emph{profile} records what the person stated, and a \emph{persona} records how they think, feel and decide. The \emph{chat ground truth} records what a reply had to know and what it should have said, and the \emph{question set} poses the same decisions as written questions. Together these give a system the recent messages, the distant ones, the facts the person stated and the state they are in (challenge 3). Every claim points back at the messages behind it, so a reader can check any of it (challenge 4). Three tracks run on these files. The \emph{reconstruction track} asks a system for the profile and persona. The \emph{chat track} scores its replies to the person's own messages. The \emph{question track} scores its answers to the written questions. The derivation is done by models and audited by people, and we report how often the two agree (Appendix~\ref{app:gt}).

% P5 - what is new in the release
The four derived files divide the work of knowing a person. The \emph{profile} holds every claim the conversation establishes, one record at a time, each pointing at the messages behind it. The \emph{persona} records who they are, how they think and feel, and how they decide. Each entry is filled only where the conversation supports it. Between them, the profile and the persona hold what it takes to understand the person. The \emph{question set} is written from the profile, so its 3{,}312 questions across 16 categories reach every claim, fact and trait the conversation established. Every item in the \emph{chat ground truth} carries a \textbf{reasoning trace}, and the trace names what the person was pointing at, where it lives, which messages the reply was built from, and whether those messages support it. A label can therefore be contested at the step that produced it. \bench makes all four challenges measurable on relationships that people actually had. We also define the task over any record of a person and state seven conditions a corpus must meet to measure it (Appendix~\ref{app:task}).
\section{Background}
\label{sec:bg}

% B1 - conversational memory
Work on conversational memory has converged on one shape. A long history is constructed, a question is posed at the end, and a system is scored on recovering the fact the question needs. LoCoMo builds very long generated conversations and queries them \citep{maharana2024locomo}, and LongMemEval scales the context past a million tokens \citep{wu2024longmemeval}. CloneMem moves to diaries and messages generated from Big Five profiles into multi-year arcs, links each question to evidence units, and reports that flat retrieval beats consolidation-based memory systems \citep{hu2026clonemem}. The shape has produced real machinery, and the memory layers now shipping inside products were tuned against it. It also decides one thing in advance. A question states what should be retrieved, so a benchmark built from questions never has to judge whether retrieval was warranted. Two benchmarks that test selection directly report how hard it is. CUPID curates histories in which a preference holds only in one context, and finds that no model it evaluates exceeds 50\% precision at identifying which prior context a request depends on \citep{kim2025cupid}. HorizonBench evolves preferences through a mental state graph and reports that all 25 models it evaluates disproportionately select the value from before the change, positioning itself as preparation in advance of longitudinal human data \citep{li2026horizonbench}.

\begin{table}[t]
\centering
\scriptsize
\setlength{\tabcolsep}{3.5pt}
\renewcommand{\arraystretch}{0.95}
\newcommand{\yes}{\tikz[baseline=-0.5ex]\fill[cData] (0,0) circle (1.9pt);}
\newcommand{\no}{\tikz[baseline=-0.5ex]\draw[cGrey!55, line width=0.45pt] (0,0) circle (1.8pt);}
\newcommand{\na}{\tikz[baseline=-0.5ex]{\fill[cGrey!30] (0,0) circle (1.9pt);
  \draw[cGrey!55, line width=0.45pt] (0,0) circle (1.8pt);}}
\newcommand{\grp}[1]{\addlinespace[1pt]%
  \multicolumn{6}{@{}l}{\itshape\color{cGrey}#1}\\[0.5pt]}
\newcommand{\bmk}[1]{\hspace{0.7em}#1}
\begin{tabular}{@{}lccccc@{}}
\toprule
Benchmark & Longitudinal & Probes & Evidence & Abstention & Checkable traces \\
\midrule
\grp{Crowdworker dialogs}
\bmk{PersonaChat} \citeyearpar{zhang2018personachat}  & \no  & authored & \no  & \no  & \no \\
\bmk{MSC} \citeyearpar{xu2022goldfish}                & \yes & authored & \no  & \no  & \no \\
\grp{Public histories}
\bmk{LaMP} \citeyearpar{salemi2024lamp}               & \no  & authored & \no  & \no  & \no \\
\grp{Generated dialogs}
\bmk{LoCoMo} \citeyearpar{maharana2024locomo}         & \yes & authored & \yes & \yes & \no \\
\bmk{LongMemEval} \citeyearpar{wu2024longmemeval}     & \yes & authored & \yes & \yes & \no \\
\bmk{PersonaMem} \citeyearpar{jiang2025personamem}    & \yes & authored & \no  & \no  & \no \\
\bmk{PersonaMem-v2} \citeyearpar{jiang2025personamemv2}& \yes & authored & \yes & \no  & \no \\
\bmk{CUPID} \citeyearpar{kim2025cupid}                & \yes & authored & \yes & \no  & \no \\
\bmk{CloneMem} \citeyearpar{hu2026clonemem}           & \yes & authored & \yes & \yes & \no \\
\bmk{HorizonBench} \citeyearpar{li2026horizonbench}   & \yes & authored & \yes & \no  & \no \\
\bmk{ANCHOR} \citeyearpar{venkit2026best}             & \yes & authored & \no  & \no  & \no \\
\bmk{CIMemories} \citeyearpar{mireshghallah2026cimemories}& \no & authored & \na & \yes & \no \\
\grp{Real conversations}
\bmk{PRISM} \citeyearpar{kirk2024prism}               & \no  & authored & \na  & \no  & \no \\
\bmk{CompanionBench} \citeyearpar{liu2026companionbench}& \no & synthesized & \na & \no & \no \\
\bmk{AlpsBench} \citeyearpar{xiao2026alpsbench}       & \yes & authored & \yes & \no  & \no \\
\midrule
\textbf{\bench} & \yes & \textbf{verbatim} & \yes & \yes & \yes \\
\bottomrule
\end{tabular}
\caption{Memory and personalization benchmarks. Filled circles mark yes, open
circles no, and half-filled not applicable.}
\label{tab:comparison}
\end{table}

% B3 - companionship
Companion applications are a setting of their own. A person returns to the same agent daily for months, and the relationship is what they are there for. ANCHOR evaluates whether the companion stays itself across 2{,}008 conversations, 27 personas and three memory configurations, and reports trajectory accuracy of 44.4\% with recall of the user's state near chance under every configuration \citep{venkit2026best}. CompanionBench works from de-identified real companion conversations and anchors its rubric in psychological theory, and states directly that a 20-turn session cannot reach the depth of disclosure that develops over a longer relationship \citep{liu2026companionbench}. Outside machine learning the same relationship is studied through interviews and surveys, which describe how the bond develops over months and document the distress people report when a system's behavior changes \citep{skjuve2021companion,brandtzaeg2022aifriend,ta2020socialsupport,laestadius2024toohuman}. Every one of those studies works from what people say about the interaction afterward.

% B4 - real conversation, and the comparison
Three efforts remove the substitution at the data. AlpsBench curates 2{,}500 long-term sequences from WildChat and evaluates seven memory systems on them \citep{xiao2026alpsbench}, which makes it the closest prior work to ours on data. PRISM collects feedback from 1{,}500 people, the largest body of real human preference data in this area \citep{kirk2024prism}. OmniBehavior and LUNAR build from real behavioral traces outside conversation \citep{chen2026towards,zhang2026lunar}, and OmniBehavior reports that models evaluated on real traces converge toward an average agreeable person and lose the long tail. A parallel line supplies a written description of the person and scores agreement with it, from PersonaChat through PersonaMem \citep{zhang2018personachat,xu2022goldfish,salemi2024lamp,jiang2025personamem,jiang2025personamemv2}. A separate measurement shows that demographics explain 1.5\% of the variance in how two people respond \citep{venkit2026need}. Table~\ref{tab:comparison} places each benchmark on five properties.

% B5 - what RealFriend adds
\bench releases ten people's conversations to an AI companion for months, and the corpus is every message they sent. Their traits were never declared in advance, so a system's account of who each person is can be wrong, and the conversation is there to show where. Which messages mattered was never marked, so a system has to judge for itself when the past is worth reaching for, and that judgment is what we score. Every item names the messages it rests on, so a reader who doubts a label can go find the evidence and argue with it.

% B6 - units
The rest of this paper uses the following terms. A \emph{message} is one utterance by one party, with a speaker and a timestamp. A participant's \emph{span} is the number of calendar days from their first message to their last, and their \emph{active days} are the days on which they sent at least one. An \emph{item} is one unit of evaluation, and every item names the messages it rests on, which we call its \emph{evidence}. Chat items ask for the reply that should follow a given message, and question items ask for a fact about the participant. Each chat item carries a \emph{tier}, the shape of the context its reply may draw on, and each question item carries a \emph{category}, the kind of fact it asks for. An item is \emph{scoreable} when its evidence survived release processing intact. Appendix~\ref{app:related} traces each line of work to the substitution it makes and what that substitution costs, and states the criteria behind every cell of Table~\ref{tab:comparison}.
\section{The Dataset}
\label{sec:dataset}

% D1 - what this records, scale and shape
The corpus records relationships with an AI companion from the first message to the last. The conversations come from the application's operational store, so every released message stands where the participant sent it. The release is fully anonymized. Every message was rewritten, with direct identifiers replaced by consistent surrogates: the words change, while what the participant said, why they said it and the context it answers are kept. Nobody was recruited to a protocol, so the corpus holds what daily use of a product produced. Table~\ref{tab:peruser} gives the per-participant composition. \bench is designed per relationship. Each participant is a self-contained instance with its own five files, every task is posed within one history, and every result can be reported for one person (Appendix~\ref{app:persubject}). The two heaviest relationships hold 73\% of all messages, nearly three times as many as the other eight combined, and supply the depth a longitudinal benchmark exists to test. Depth does not add across people. The eight shorter relationships reflect the lighter use most people make of a companion, and it must understand them from far less of what the person has shared.

\begin{table}[t]
\centering
\small
\setlength{\tabcolsep}{6pt}
\begin{tabular}{@{}lrrrrr@{}}
\toprule
User & Span & Active days & Messages & Chat items & Question items \\
\midrule
U01 &  36 &   9 &   115 &  28 &   23 \\
U02 & 101 &  19 &   173 &  35 &   27 \\
U03 &  46 &  12 &   173 &  26 &   36 \\
U04 &  71 &  19 &   195 &  39 &   35 \\
U05 &  46 &  23 &   420 &  37 &   47 \\
U06 & 120 &  53 &  1349 &  81 &  110 \\
U07 &  42 &  39 &  2005 & 150 &  405 \\
U08 &  68 &  50 &  2918 & 221 &  304 \\
U09 & 111 & 111 &  7243 & 434 &  991 \\
U10 & 115 &  95 & 12627 & 482 & 1257 \\
\midrule
Total & & 430 & 27218 & 1533 & 3235 \\
\bottomrule
\end{tabular}
\caption{Per-participant composition.}
\label{tab:peruser}
\end{table}

% D2 - the two records of a person
Understanding is not only remembering what they told you, but also reasoning, organizing the memories, and tracing what that says about them. The profile records what someone stated, and every claim in it carries the messages that establish it, a confidence, and the window over which it held. Extraction is adversarial toward its own output and throws away more than a third of what it proposes, and the claims that later failed audit are kept, so the audit can itself be audited. The persona has no counterpart in prior work, because it holds each reading of a person with the messages behind it and a stated strength, down to named working hypotheses that carry their grounds and how firmly they are held. Every reference either file makes into the conversation resolves, so a reading of a person traces back to what was said and can be contested. That is what makes reply A in Figure~\ref{fig:understanding} a scoreable error. The persona holds only what these ten people revealed. People tell a companion about their family, their health, what is worrying them, and almost never state their census category.

% D3 - the two tracks
Knowing someone shows up in what you can say about them and in what you say to them. The corpus scores both. The \emph{question track} asks the first, across sixteen kinds of fact authored from the profile, and every item names the messages its answer rests on except the 63 marked unanswerable, which name none by design. The \emph{chat track} asks the second, and it had to be derived. Its probes are user messages taken verbatim, and its reference reply is written under the recorded context, so no item can depend on anything it does not name. Every label carries the reasoning trace that produced it, written in the same call as the decision it explains, and each stage checks against the conversation, so no claim serves as its own warrant.  Appendix~\ref{app:corpus} specifies the rest: the settings, the schema of each released file, the artifacts of the medium, what the shifted timestamps preserve, the per-subject census behind every macro-average, and how the profiles were built and audited.
\section{Reasoning over Conversations to Construct The Ground Truth}
\label{sec:chatgt}

% O - why the labels have to be derived, and why they can be trusted
Every label here had to be derived, and a derived label deserves trust only if a reader can check each step. The probes are observed, every label comes from a stated procedure that records each decision it makes, and the finished corpus was graded by an audit that procedure does not control.

\subsection{Probes and Items}
\label{sec:contain}

% P1 - what an item records
An item pairs a probe with the context a correct reply may draw on, a reference reply written from that context alone, and the reasoning trace that produced both. The context lists the messages just before the probe, any claims from the profile, and any specific earlier exchanges. Its shape sets the item's \emph{tier}. Preceding messages alone make an item \texttt{basic}, a profile claim makes it \texttt{intermediate}, an earlier exchange makes it \texttt{hard}, and an empty context makes it \texttt{edge}, so the tier can be recomputed from the lists without trusting the label. The \emph{locus} names where the thing the participant is pointing at lives. Of the 1{,}533 scoreable items, 1{,}009 point into the current thread and 120 at nothing prior. The other 404 point outside the thread, to a profile claim (223), an earlier exchange (116) or the companion's own past behavior (65), and on 167 of them the reference reply uses what it cites, which makes them \emph{dependent}. Appendix~\ref{app:labels} gives the locus distribution by stratum.

% P2 - which probes, and why three strata
Every probe is a message a participant sent, taken verbatim from \texttt{source}, and nothing is rewritten, paraphrased or authored. Choosing which messages to use poses a problem, because messages that need the past are rare. In a random sample only 3.4\% of items need anything outside the current thread, which leaves 40 items to study. The probes come in three strata, and each answers one question. The \emph{proportional} stratum is that random sample, 1{,}169 scoreable items drawn from every participant, and it alone answers how often the past is needed (condition V3 of Appendix~\ref{app:validity}). The \emph{enriched} stratum answers what a system does when the past is needed. The \emph{abstention} control answers whether a system holds back when nothing is needed. Its 434 session openers need nothing prior, and a benchmark on which retrieval helps cannot detect a system that retrieves too much.

\subsection{Traced Reasoning}
\label{sec:derivation}

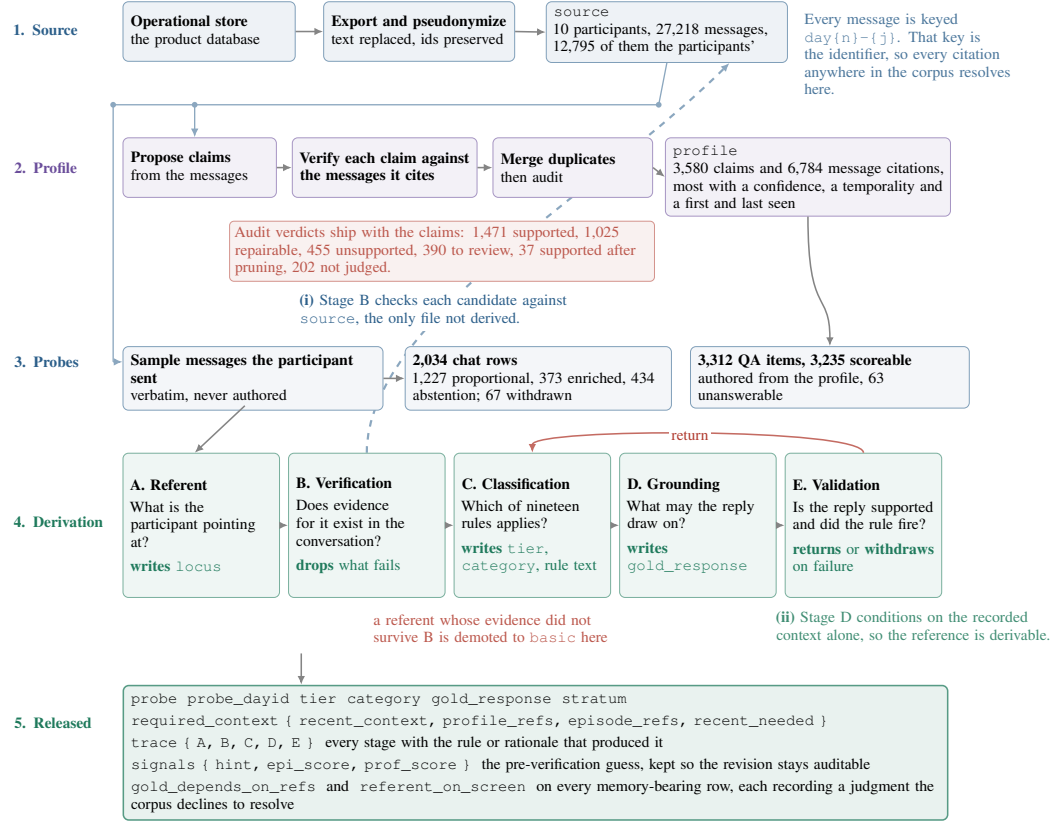
\begin{figure}[t]
\centering
\resizebox{\linewidth}{!}{%
\begin{tikzpicture}[
  font=\scriptsize,
  >={Latex[length=4pt,width=3.5pt]},
  lane/.style={anchor=west, font=\scriptsize\bfseries, text width=1.85cm, align=left,
              execute at begin node={\hyphenpenalty=10000}},
  bx/.style={rounded corners=2.5pt, draw, inner sep=3.5pt, align=left,
             minimum height=0.95cm, anchor=north west,
             execute at begin node={\hyphenpenalty=10000}},
  src/.style={bx, draw=cData!60, fill=cData!8},
  prf/.style={bx, draw=cProf!60, fill=cProf!8},
  prb/.style={bx, draw=cData!60, fill=cWash},
  drv/.style={bx, draw=cGT!60, fill=cGT!7, text width=2.25cm, minimum height=2.3cm},
  rel/.style={bx, draw=cGT!75, fill=cGT!10, line width=0.7pt},
  rej/.style={rounded corners=2pt, draw=cBad!55, fill=cBad!7, inner sep=2.5pt,
              align=left, font=\scriptsize, text=cBad!85, anchor=north west,
              execute at begin node={\hyphenpenalty=10000}},
  flow/.style={->, cGrey!85, line width=0.7pt},
  rail/.style={cData!55, line width=0.6pt},
  note/.style={font=\scriptsize, align=left,
              execute at begin node={\hyphenpenalty=10000}},
]
\node[lane, text=cData] at (-0.30,0) {1. Source};
\node[src, text width=2.5cm] (o1) at (1.55,0.45)
  {\textbf{Operational store}\\the product database};
\node[src, text width=2.8cm] (o2) at (4.75,0.45)
  {\textbf{Export and pseudonymize}\\text replaced, ids preserved};
\node[src, text width=3.6cm] (o3) at (8.30,0.45)
  {\texttt{source}\\10 participants, 27{,}218 messages, 12{,}795 of them the participants'};
\draw[flow] (o1.east) -- (o2.west);
\draw[flow] (o2.east) -- (o3.west);
\node[note, anchor=north west, text width=3.5cm, text=cData!85] at (12.30,0.38)
  {Every message is keyed \texttt{day\{n\}-\{j\}}. That key is the identifier,
   so every citation anywhere in the corpus resolves here.};
\draw[rail] (o3.south) -- (10.10,-1.20);
\draw[rail] (1.40,-1.20) -- (10.10,-1.20);
\foreach \x in {1.40,2.70,10.10} \fill[cData!55] (\x,-1.20) circle (1.3pt);
\draw[->, rail] (2.70,-1.20) -- (2.70,-1.72);
\draw[rail] (1.40,-1.20) -- (1.40,-5.30);
\draw[->, rail] (1.40,-5.30) -- (1.55,-5.30);
\node[lane, text=cProf] at (-0.30,-2.2) {2. Profile};
\node[prf, text width=2.2cm] (c1) at (1.55,-1.72)
  {\textbf{Propose claims}\\from the messages};
\node[prf, text width=2.7cm] (c2) at (4.25,-1.72)
  {\textbf{Verify each claim against the messages it cites}};
\node[prf, text width=2.3cm] (c3) at (7.45,-1.72)
  {\textbf{Merge duplicates}\\then audit};
\node[prf, text width=4.3cm] (c4) at (10.20,-1.72)
  {\texttt{profile}\\3{,}580 claims and 6{,}784 message citations, most with
   a confidence, a temporality and a first and last seen};
\draw[flow] (c1.east) -- (c2.west);
\draw[flow] (c2.east) -- (c3.west);
\draw[flow] (c3.east) -- (c4.west);
\node[rej, text width=6.6cm] at (3.25,-3.05)
  {Audit verdicts ship with the claims: 1{,}471 supported, 1{,}025 repairable, 455
   unsupported, 390 to review, 37 supported after pruning, 202 not judged.};
\node[lane, text=cData] at (-0.30,-5.30) {3. Probes};
\node[prb, text width=3.9cm] (b1) at (1.55,-5.05)
  {\textbf{Sample messages the participant sent}\\verbatim, never authored};
\node[prb, text width=4.0cm] (b2) at (6.05,-5.05)
  {\textbf{2{,}034 chat rows}\\1{,}227 proportional, 373 enriched,
   434 abstention; 67 withdrawn};
\node[prb, text width=4.2cm] (b3) at (10.60,-5.05)
  {\textbf{3{,}312 QA items, 3{,}235 scoreable}\\authored from the profile, 63 unanswerable};
\draw[flow] (b1.east) -- (b2.west);
\draw[flow] (c4.south) .. controls +(0,-1.6) and +(0,0.9) .. (b3.north);
\node[lane, text=cGT] at (-0.30,-7.85) {4. Derivation};
\node[drv] (dA) at (1.55,-6.75)
  {\textbf{A. Referent}\\[2pt] What is the participant pointing at?\\[3pt]
   {\color{cGT}\textbf{writes} \texttt{locus}}};
\node[drv] (dB) at (4.19,-6.75)
  {\textbf{B. Verification}\\[2pt] Does evidence for it exist in the
   conversation?\\[3pt] {\color{cGT}\textbf{drops} what fails}};
\node[drv] (dC) at (6.83,-6.75)
  {\textbf{C. Classification}\\[2pt] Which of nineteen rules applies?\\[3pt]
   {\color{cGT}\textbf{writes} \texttt{tier}, \texttt{category}, rule text}};
\node[drv] (dD) at (9.47,-6.75)
  {\textbf{D. Grounding}\\[2pt] What may the reply draw on?\\[3pt]
   {\color{cGT}\textbf{writes} \texttt{gold\_response}}};
\node[drv] (dE) at (12.11,-6.75)
  {\textbf{E. Validation}\\[2pt] Is the reply supported and did the rule
   fire?\\[3pt] {\color{cGT}\textbf{returns} or \textbf{withdraws} on failure}};
\foreach \a/\b in {dA/dB, dB/dC, dC/dD, dD/dE} \draw[flow] (\a.east) -- (\b.west);
\draw[flow] (3.50,-6.00) -- (2.70,-6.75);
\draw[->, cBad!75, line width=0.7pt]
  (dE.north) .. controls +(0,0.40) and +(0,0.40) .. (dC.north);
\node[text=cBad, fill=white, inner sep=1pt] at (10.60,-6.47) {return};
\node[note, anchor=north west, text width=4 cm, text=cBad!85] at (5.43,-9.20)
  {a referent whose evidence did not survive B is demoted to \texttt{basic} here};
\node[lane, text=cGT] at (-0.30,-11.05) {5. Released};
\node[rel, text width=12.9cm] (item) at (1.55,-10.45)
  {\texttt{probe}\ \ \texttt{probe\_dayid}\ \ \texttt{tier}\ \
   \texttt{category}\ \ \texttt{gold\_response}\ \ \texttt{stratum}\\[2pt]
   \texttt{required\_context \{ recent\_context, profile\_refs, episode\_refs,
     recent\_needed \}}\\[2pt]
   \texttt{trace \{ A, B, C, D, E \}}\ \ every stage with the rule or rationale
   that produced it\\[2pt]
   \texttt{signals \{ hint, epi\_score, prof\_score \}}\ \ the pre-verification
   guess, kept so the revision stays auditable\\[2pt]
   \texttt{gold\_depends\_on\_refs}\ \ and \ \texttt{referent\_on\_screen}\ \ on
   every memory-bearing row, each recording a judgment the corpus declines to
   resolve};
\draw[flow] (4.40,-9.95) -- (4.40,-10.45);
\begin{scope}[on background layer]
\draw[->, cData!55, line width=1.0pt, dashed]
  (dB.north) .. controls +(0,2.1) and +(-2.0,-2.0) .. ($(o3.south)+(1.0,0)$);
\end{scope}
\node[note, anchor=north west, text width=4.9cm, text=cData, fill=white,
      inner sep=2pt] at (4.30,-4.15)
  {\textbf{(i)} Stage B checks each candidate against \texttt{source}, the only
   file not derived.};
\node[note, anchor=north west, text width=4.5cm, text=cGT!85, fill=white,
      inner sep=2pt] at (11.9,-9.20)
  {\textbf{(ii)} Stage D conditions on the recorded context alone, so the
   reference is derivable.};
\end{tikzpicture}%
}
\caption{How the chat ground truth is built.}
\label{fig:pipeline}
\end{figure}

% P3 - the five stages
Figure~\ref{fig:pipeline} shows the five stages. Stage A resolves what the probe points at and names the store it would live in. Stage B retrieves candidate records for that referent and checks each against \texttt{source}, dropping what fails. Stage C assigns the tier and category by one of nineteen fixed rules and records the rule that fired. Stage D writes the reference reply and names the messages it rests on. Stage E checks that the reply is supported by its evidence and that the category fits, and demotes the item when either check fails. Each stage reads only what the stages before it wrote, and writes its decision and its reason in the same call, so no rationale is reconstructed after the decision it explains. That record is the item's reasoning trace. Appendix~\ref{app:phases} states what each stage may and may not revise.

% P4 - how the ordering protects the labels
A label can go wrong in three ways, and the order of the stages closes each one. The first is invented evidence. By Stage B, an invented profile claim cannot confirm a probe and enter the gold as a memory a system is expected to retrieve (V5). The second is a reply that depends on something the item does not name. By Stage D, the context is sufficient by construction, and the gap between the oracle condition and any retriever measures retrieval alone (V1). The recorded grounding confirms it. Because this reply is constructed, the reply the participant actually received is kept as well, as the message after each probe in \texttt{source}, so every result can be re-run against it. The third is an error that makes memory look more needed than it is. Verification can remove a reference and never add one, so a verification error can only lower the demand rate (V6, Proposition~\ref{prop:onesided}). Three invariants confirm that the ordering held on every item: every identifier the four derived files cite resolves to a message in \texttt{source}, no item cites a message after its own probe, and the tier recomputes from the lists.

% P5 - what the derivation threw away
The derivation discarded much of what it first proposed. Stage A proposed a distant referent for 806 of the unused messages, and stage B verified 422. In its final pass stage B matched 1{,}544 profile claims to referents and kept 752, dropping those whose evidence did not support the referent. The final tier differs from the pre-verification signal on 685 of the 1{,}533 items. That signal is a surface guess a cheap retrieval gate would make, and it is wrong on nearly half the items (Appendix~\ref{app:lowdemand}). Appendix~\ref{app:gt} gives the full derivation: what each stage may revise, how the probes were sampled, the on-screen judgment, every withdrawal and repair, the release gates, and the audit of both tracks.
\section{Benchmark Setup}
\label{sec:setup}

% P1 - what is benchmarked, and why
\bench benchmarks the two questions about a person that Section~\ref{sec:intro} separates. The first is who the participant is. The reconstruction track tests it: a system reads the participant's full history and writes the profile and persona files. The second is which of the things they have said bears on the message in front of a system. The chat and question tracks test it by scoring whether a system reaches for the past at all, which earlier messages it finds, and what its reply makes of them. The reasoning trace of Section~\ref{sec:chatgt} records how each chat label was decided, and the chat track reads its outputs: the reference set, whether memory is needed, the tier and the reference reply. Table~\ref{tab:tasks} summarizes the three tracks.

% P1b - scope
\paragraph{Scope.} The benchmark covers ten relationships with one companion application, each 36 to 120 days long, and every chat probe is a message the participant actually sent. On the second question it scores three decisions: restraint, retrieval, and application. Empathy, the prediction of future behavior, and a person's values are out of scope. Chat and question rates are reported pooled over items with the macro-average over participants beside them, and comparisons carry a bootstrap over participants. Rates of occurrence come from the proportional stratum alone, and comparisons are reported on it.

\begin{table}[t]
\centering\footnotesize
\setlength{\tabcolsep}{4pt}
\begin{tabularx}{\linewidth}{@{}l
  >{\raggedright\arraybackslash\hsize=0.8\hsize}X
  >{\raggedright\arraybackslash\hsize=0.8\hsize}X
  >{\raggedright\arraybackslash\hsize=1.4\hsize}X r@{}}
\toprule
Track & Input & Output & Scored against & Items \\
\midrule
Reconstruction & a participant's full history & profile and persona files, each entry
  citing messages & the released files, field by field (precision, recall, F1, run
  agreement) & 10 participants \\
Chat & a message the participant sent, with the history before it & a ranking of
  earlier messages; a reply & reference set (hit@$k$, MRR); reference reply (content
  match, evidence coverage, appropriate application rate); need for memory (misfire
  rate, AUROC) & 1{,}533 + 434 \\
Question & an authored question over the whole history & a ranking of earlier
  messages; an answer or an abstention & reference set (hit@$k$, MRR); answerability
  (abstention, AUROC); reference answer (judged accuracy) & 3{,}235 \\
\bottomrule
\end{tabularx}
\caption{The three tracks. Item counts are scoreable items; the chat count adds 434 cold opens.}
\label{tab:tasks}
\end{table}

% P2 - the retrieval methods and the controls
\paragraph{Retrieval.} Five methods run over both tracks: \textsc{Random}; \textsc{Recency}, which returns the $k$ messages preceding the probe and performs no retrieval; \textsc{Recency-user}, which returns the user's own preceding messages and so separates adjacency from speaker; \textsc{BM25}; and \textsc{Oracle}, which returns the labeled gold set. Each method sees one participant's history at a time. Three controls substitute a different gold set and rescore every method against it: \emph{offset permuted} preserves the distances from probe to gold and destroys the content relation, \emph{position randomized} destroys both, and \emph{broken oracle} supplies messages known to lie outside the gold set and must score zero.

% P3 - five conditions differing only in what precedes the probe
\paragraph{The context ablation.} One generator answers every probe under five conditions that differ only in what precedes the probe in its input. Every message is shown truncated to 220 characters. C0 supplies the probe alone, C1 the three preceding messages, C2, the oracle condition, the messages the item records as required, C3 the ten highest-scoring \textsc{BM25} messages preceding the probe, and C4 the ten preceding messages. C2 against C1 asks what the recorded memory adds to the recent thread, and Section~\ref{sec:results} splits that difference by whether the probe needed memory (Proposition~\ref{prop:nonident}). C3 against C4 prices an imperfect retriever against recency at an equal budget of ten messages. As shipped these two also differ in how the block is labeled and in whether it is ordered by score or by time, so a matched pair, C3m and C4m, repeats both under one neutral label in chronological order, and Section~\ref{sec:results} reports the matched difference.

\vspace{-2mm}
% P4 - what each number means
\paragraph{Metrics.} Retrieval is scored by whether a required message appears in the top $k$, and by MRR. Recall over the whole gold set is reported in the appendix, where an item citing more than $k$ messages cannot reach one. Responses are scored by content match against the reference, by restraint, and by evidence coverage, which uses no judge and cross-checks the other two. One minus restraint, on the 1{,}129 items with an empty reference set, is the misfire rate of \citet{yoon2026benchpres}; the fraction of memory-bearing items whose response draws on the past is that work's appropriate application rate. We report both, because neither constrains the other. Whether memory is needed at all is scored twice, by abstention on the unanswerable questions and the cold opens and by the area under the ROC curve of a lexical score and of a detector model shown the probe and its three preceding messages, the latter on the proportional stratum alone since every enriched item carries a reference by construction. One judge scores the five conditions in a single call, and the matched pair the same way so that they are ranked against one another. Appendix~\ref{app:setup} specifies the protocol: the context ablation (\ref{app:conditions}), the retrieval controls (\ref{app:controls}) and baselines (\ref{app:baselines}), the response judge and its agreement with human adjudication (\ref{app:judge}), how each quantity is computed (\ref{app:metrics}), and the models behind each role (\ref{app:models}).

\vspace{-2mm}
% P5 - reconstructing the standing face
\paragraph{Reconstruction.} Three agent systems, Claude Opus 5.5, Codex GPT-5.6-sol and Antigravity running Gemini 3.8 Flash, each run three times on each of the ten participants. Each field is scored by exact match, presence or recall according to its type (Appendix~\ref{app:recon}). F1 is computed per run over the fields that either the run or the released file fills, pooled over the ten participants, and averaged over the three runs. Run agreement is the same overlap between two runs of one system, so a difference that the runs share is systematic rather than noise.
\vspace{-2mm}
\section{Results}
\label{sec:results}

% P0 - What is aggregated, over what, under which reading
Every rate below divides by a stated denominator. The proportional stratum (1{,}169 scoreable probes of 1{,}227) carries every population rate on the chat track. The enriched stratum (364 of 373) is a census of the deep-history items, so it supports conditional statements and no rates. The abstention control (434 cold opens) is scored only for abstention. Two readings run throughout: the \emph{recorded} reading keeps all 404 memory-bearing probes, 40 proportional and 364 enriched. The \emph{strict} reading removes the 154 whose referent is visible on screen (Appendix~\ref{app:refs}). Each result is reported by population, because a pooled number reports the corpus composition of as much as the capability (Proposition~\ref{prop:nonident}).

% P1 - Memory is rarely required, and the requirement is a property of depth
Memory is required far less often than benchmark construction implies. On the natural-rate stratum, 3.4\% [2.5, 4.6] of the messages participants sent carry a verified memory demand under the recorded reading and 1.3\% [0.8, 2.1] under the strict one. Verification is monotone, so each of these bounds the true rate from below (Proposition~\ref{prop:onesided}), and two procedures over the 11{,}369 participant messages outside that sample agree. The census of Section~\ref{sec:chatgt} gives 3.2\% [2.9, 3.5] and a re-reading of it under the Stage~B instruction gives 1.73\% [1.45, 2.06] (Appendix~\ref{app:extra-demand}). The requirement is also uneven. Three of the ten participants contribute no memory-bearing item under the recorded reading, and five under the strict one. Memory demand is a property of relationship depth. 

% P2 - The required material sits far back in the relationship
When a message draws on the past, the material it needs usually sits far back. The furthest message a memory-bearing probe requires sits a median of 2{,}157 messages back, with an upper quartile of 5{,}114 and a maximum of 12{,}542. Even the nearest required message sits a median of 450 messages back. The length of the relationship fixes the horizon a system must cover.

\begin{table}[t]
\caption{Hit@5 on the distant items.}
\label{tab:retrieval}
\centering\small
\begin{tabular}{lccc}
\toprule
Method & All items & Averaged by participant & Dependent items \\
\midrule
Random                          & 0.005 & 0.004 & 0.000 \\
Recency over participant messages & 0.015 & 0.014 & 0.006 \\
Recency                         & 0.022 & 0.029 & 0.024 \\
BM25                            & 0.248 & 0.424 & 0.347 \\
Oracle                          & 1.000 & 1.000 & 1.000 \\
Broken oracle                   & 0.000 & 0.000 & 0.000 \\
\bottomrule
\end{tabular}
\end{table}

% P3 - Ranking fails at this depth, and pooled retrieval scores hide it
Locating that material defeats the scorers a memory system usually reaches for. Table~\ref{tab:retrieval} reports how often a required message appears in the top five for the distant items. Recency, the order a sliding context window imposes, finds one for 0.022 of them against 0.005 for a random ranking, and for none once on-screen referents are removed (Appendix~\ref{app:controls}). Lexical matching finds one eleven times as often, at 0.248, and still misses three in four. The broken oracle scores 0.000 throughout, as a control must. Pooling reverses this order. Over the 1{,}477 chat probes with a reference set, of which 72.6\% are answerable from the recent thread, recency finds a required message for 0.959 of them and lexical matching for 0.236. Lexical matching barely moves between the two populations, while recency falls by a factor of forty. A pooled retrieval score reports the composition of the corpus. Coverage of the required message rises with the context budget, from 0.109 at 8k tokens to 0.453 at 128k and 1.000 at one million, where every history fits (Figure~\ref{fig:window}, Appendix~\ref{app:extra-coverage}). What a model does with the context it is given is a separate question, and the next three results answer it.
\begin{table}[t]
\caption{Content match by context condition.}
\label{tab:ablation}
\centering\small
\begin{tabular}{llcc}
\toprule
Condition & Context supplied & All items & Averaged by participant \\
\midrule
C0 & None                              & 1.431 & 1.286 \\
C1 & Three preceding messages          & 1.622 & 1.587 \\
C2 & Recorded messages (oracle)        & 1.763 & 1.736 \\
C3 & Ten messages retrieved by BM25    & 1.426 & 1.352 \\
C4 & Ten preceding messages            & 1.625 & 1.584 \\
\bottomrule
\end{tabular}
\end{table}

% P4 - Supplying context helps, and most of the help reaches probes that need no memory
Supplying context helps, and most of the help reaches probes that need no memory. Table~\ref{tab:ablation} reports mean content match, on a 0 to 2 scale, over the 1{,}533 scoreable probes of both strata. Every contrast below is a mean of per-probe differences. Replacing the last three messages (C1) with the recorded messages (C2) raises content match by $+0.140$ [0.126, 0.167] over both strata and by $+0.108$ [0.077, 0.159] on the proportional stratum. Write $\pi$ for the share of proportional probes that need memory, and $\gamma_1$ and $\gamma_0$ for the gain on probes that do and do not. The proportional gain splits into $\pi\gamma_1 = +0.004$ and $(1-\pi)\gamma_0 = +0.104$, so 96\% of it lands on probes that need no memory. On the 167 items whose reference reply depends on the recorded messages, the same replacement is worth $+0.455$ [0.319, 0.588]. Proposition~\ref{prop:nonident} predicts this. Where demand is rare, a pooled ablation difference measures the composition of the corpus as much as the use of memory, and only the split recovers the second. Which messages are supplied matters as much as how many. With budget, heading and order held fixed (C3m and C4m, Appendix~\ref{app:conditions}), ten messages chosen by BM25 score 0.189 [0.156, 0.247] below the ten preceding ones, and in Table~\ref{tab:ablation} they match no context at all (1.426 against 1.431). The strict reading strengthens this, raising the gain on intermediate items from 0.039 [$-0.040$, 0.108] to 0.106 [0.017, 0.196], an interval that excludes zero only under the stricter reading.

% P5 - Deciding when memory is needed is where systems fail
The decision to look is where systems fail. On the proportional stratum, where memory is needed at its natural rate, a detector that reads the probe and its three preceding messages separates the cases at chance. Its AUROC is 0.530 [0.488, 0.572] against probes whose referent is absent and 0.548 [0.488, 0.609] against cold opens. Scored against its own target it reaches 0.601 [0.527, 0.676], and its threshold then catches 6 of the 40 probes that need memory while raising 40 false alarms. Authored questions are easier to sort. A detector shown three retrieved messages reaches 0.802 [0.742, 0.857] on the question track against 0.546 [0.490, 0.605] on chat probes, and a keyword scorer's own confidence reaches 0.696 [0.629, 0.758]. An authored question shares more words with its evidence than a message sent in conversation does. 
% Even so, the threshold of the 0.802 detector calls 40\% of answerable questions unanswerable.

% P6 - Models misuse the past they are given
Given no context, models never reach for a past that is absent. Across the 434 cold opens and the 120 probes whose verified locus is \emph{none}, the misfire rate without context is zero. With context supplied, it rises to 29.2\% when the messages were selected for the probe and to 60.8\% when ten were supplied without regard to it (Appendix~\ref{app:extra-selectivity}). Framing drives the decision. Labeling the same ten messages \emph{retrieved memories} instead of \emph{earlier turns from this chat} raises memory use by ten to fourteen points on every population we test, including probes that need no memory. Relabeling a recency window the same way changes nothing the intervals separate. C2 and C1 in Table~\ref{tab:ablation} also carry different headings, so their contrasts include whatever the heading adds, which we do not separate.

% P7 - Reconstruction recovers the person and attributes more
Three agent systems read each participant's full history and wrote the persona file, three runs each over the ten participants (Table~\ref{tab:reconstruction}). All three reach an F1 of 0.69 to 0.70 while the tokens they process differ about thirtyfold, and each agrees with its own repeated runs at 0.93 to 0.95. They also agree with one another at 0.86 to 0.88, well above their agreement with the file, although they span a frontier and a flash-tier model: scale does not close the gap. Precision of 0.56 to 0.59 against recall of 0.86 to 0.89 means a system recovers what the file holds.

\begin{table}[t]
\centering
\small
\begin{tabular}{@{}lccccc@{}}
\toprule
System & F1 & Precision & Recall & Tokens processed & Run agreement \\
\midrule
Claude Opus 5.5              & 0.686 & 0.557 & 0.893 & 546M & 0.935 \\
Codex GPT-5.6-sol            & 0.688 & 0.573 & 0.861 & 18M  & 0.932 \\
Antigravity Gemini 3.8 Flash & 0.701 & 0.591 & 0.861 & 36M  & 0.946 \\
\bottomrule
\end{tabular}
\caption{Persona reconstruction from the full history, pooled over the ten participants.}
\label{tab:reconstruction}
\end{table}
\vspace{-3mm}
\section{Conclusion}
\label{sec:conclusion}
\bench releases ten real relationships between people and an AI companion, and every chat item carries the reasoning trace that produced it, checked against the conversation it came from. The past is rarely needed and, when it is, far away: 3.4\% of sampled messages reach outside the current thread, a median of 2{,}157 messages back, and the demand concentrates in the longest relationships. Where it is needed, supplying the recorded evidence raises content match by 0.455, yet no system we evaluate decides when to reach back, and no detector we tried tells the two situations apart on real messages. On who the person is, three agent systems reconstruct the persona with the same F1 at a 31-fold difference in cost, and all three add about seventy to eighty further fields for every hundred they recover. Unlike an authored corpus, \bench knows how often memory is actually required.
\subsection*{AI use statement}

Language models are research instruments in this work, and each use is disclosed with its check. They run stages A, D and E of the derivation (Section~\ref{sec:chatgt}), generate and judge the replies of Section~\ref{sec:setup}, audit both tracks, and are the systems the reconstruction track evaluates; every per-stage decision ships in the released items. They also produced parts of the released data: the anonymization rewrote every message with a model, the question track was built with one, and the persona and profile layers were generated by a model pipeline and then reviewed by hand (Appendices~\ref{app:qagt} and~\ref{app:datasheet}). We did not take their output on trust. Twenty-six release gates and five families of mechanical check run over the corpus without a model, an independent model graded what the derivation produced, and where an instrument overstated itself we report it. We also used AI assistants to write and debug code, give feedback on experimental design, help interpret results, search for related work, check references, and draft and edit prose. Every number was recomputed by script, every sentence was read and revised by an author, and the authors take full responsibility for the content.

\subsection*{Ethics statement}

The conversations are by ten people using a companion application for their own purposes; all consented to research use and release. Every released file is anonymized: account identifiers, contact details, names and named entities are replaced with consistent surrogates across all five files, and what remains is what each participant said and the context it answers, in rewritten words, with the interaction structure, session rhythm and communicative style on which this research depends. The corpus contains material falling within the special categories of Article~9 of the GDPR \citep{gdpr2016}, and necessarily so: a reply that ignores what someone has said about their health is not the correct reply, and a benchmark that stripped such material would measure a different task while claiming to measure this one. Processing rests on each participant's explicit consent to research use and release, the basis Article~9(2)(a) provides for special-category data. The study and the release were approved by the data operator's internal ethics committee under its privacy policy and terms of service; that committee is not independent of the operator, and no institutional review board was involved. Removing identifiers does not remove authorship, and that is a property of conversational text. We report the measurement because a corpus described as anonymized without one is describing an assumption. The release terms prohibit profiling, identifying or targeting any individual, and any commercial use (Appendix~\ref{app:ds-uses}).

\subsection*{Reproducibility statement}

All code for constructing the corpus, auditing it, and producing every number in this paper is at \url{https://anonymous.4open.science/r/realcompanion-54CC/}, and the corpus at \url{https://huggingface.co/datasets/Quislab/RealCompanion}. The repository's reproduction map ties each reported quantity to the command that regenerates it and to the committed output it is checked against. Appendix~\ref{app:setup} gives the evaluation protocol and the model behind each role. Appendix~\ref{app:corpus} gives the field-level schema of all five files and the invariants a consumer can rely on. Every figure names the stratum it was computed on, every released item carries that field, and each analysis script takes it as an argument; Appendix~\ref{app:audit} reports both versions of the audit instruction, which was revised once during this work.

% \clearpage

\subsubsection*{Acknowledgments}
\bench is the work of the entire Quis Lab team, and we thank everyone who contributed to its data, benchmark, harness infrastructure, engineering, product, and design. We thank our advisors in research for advising on this research, and the OpenAI and Google research teams for their support. We are especially grateful to Jiayi Yu and Eric Huang, executives of Quis Lab, for their support of this research.
\clearpage
\bibliographystyle{iclr2027_conference}
\bibliography{references}

@inproceedings{salemi2024lamp,
  title     = {{LaMP}: When Large Language Models Meet Personalization},
  author    = {Salemi, Alireza and Mysore, Sheshera and Bendersky, Michael and Zamani, Hamed},
  booktitle = {Proceedings of the 62nd Annual Meeting of the Association for Computational Linguistics},
  year      = {2024}
}

@inproceedings{maharana2024locomo,
  title     = {Evaluating Very Long-Term Conversational Memory of {LLM} Agents},
  author    = {Maharana, Adyasha and Lee, Dong-Ho and Tulyakov, Sergey and Bansal, Mohit and Barbieri, Francesco and Fang, Yuwei},
  booktitle = {Proceedings of the 62nd Annual Meeting of the Association for Computational Linguistics},
  year      = {2024}
}

@article{wu2024longmemeval,
  title={Longmemeval: Benchmarking chat assistants on long-term interactive memory},
  author={Wu, Di and Wang, Hongwei and Yu, Wenhao and Zhang, Yuwei and Chang, Kai-Wei and Yu, Dong},
  journal={arXiv preprint arXiv:2410.10813},
  year={2024}
}

@inproceedings{jiang2025personamem,
  title     = {Know Me, Respond to Me: Benchmarking {LLM}s for Dynamic User Profiling and Personalized Responses at Scale},
  author    = {Jiang, Bowen and Hao, Zhuoqun and Cho, Young-Min and Li, Bryan and Yuan, Yuan and Chen, Sihao and Ungar, Lyle and Taylor, Camillo J. and Roth, Dan},
  booktitle = {Conference on Language Modeling},
  year      = {2025},
  note      = {arXiv:2504.14225}
}

@article{jiang2025personamemv2,
  title   = {{PersonaMem-v2}: Towards Personalized Intelligence via Learning Implicit User Personas and Agentic Memory},
  author  = {Jiang, Bowen and Yuan, Yuan and Shen, Maohao and Hao, Zhuoqun and Xu, Zhangchen and Chen, Zichen and Liu, Ziyi and Vijjini, Anvesh R. and He, Jiashu and Yu, Hanchao and Poovendran, Radha and Wornell, Gregory and Ungar, Lyle and Roth, Dan and Chen, Sihao and Taylor, Camillo J.},
  journal = {arXiv preprint arXiv:2512.06688},
  year    = {2025}
}

@inproceedings{kim2025cupid,
  title     = {{CUPID}: Evaluating Personalized and Contextualized Alignment of {LLM}s from Interactions},
  author    = {Kim, Tae Soo and Lee, Yoonjoo and Park, Yoonah and Kim, Jiho and Kim, Young-Ho and Kim, Juho},
  booktitle = {Conference on Language Modeling},
  year      = {2025},
  note      = {arXiv:2508.01674}
}

@inproceedings{hu2026clonemem,
  title={Clonemem: Benchmarking long-term memory for ai clones},
  author={Hu, Sen and Zhang, Zhiyu and Wei, Yuxiang and Han, Xueran and Tang, Zhenheng and Chen, Ronghao and Wang, Huacan},
  booktitle={Proceedings of the 64th Annual Meeting of the Association for Computational Linguistics (Volume 1: Long Papers)},
  pages={33571--33602},
  year={2026}
}

@inproceedings{mireshghallah2026cimemories,
  title     = {{CIMemories}: A Compositional Benchmark for Contextual Integrity in {LLM}s},
  author    = {Mireshghallah, Niloofar and Mangaokar, Neal and Kokhlikyan, Narine and Zharmagambetov, Arman and Zaheer, Manzil and Mahloujifar, Saeed and Chaudhuri, Kamalika},
  booktitle = {International Conference on Learning Representations},
  year      = {2026}
}

@article{xu2026survey,
  title   = {Toward Personalized {LLM}-Powered Agents: Foundations, Evaluation, and Future Directions},
  author  = {Xu, Yue and Chen, Qi'an and Ma, Zizhan and Liu, Dongrui and Wang, Wenxuan and Wang, Xiting and Xiong, Li and Wang, Wenjie},
  journal = {ACM Computing Surveys},
  year    = {2026},
  note    = {arXiv:2602.22680}
}

@article{wang2024aipersona,
  title   = {{AI} Persona: Towards Life-long Personalization of {LLM}s},
  author  = {Wang, Tiannan and Tao, Meiling and Fang, Ruoyu and Wang, Huilin and Wang, Shuai and Jiang, Yuchen Eleanor and Zhou, Wangchunshu},
  journal = {arXiv preprint arXiv:2412.13103},
  year    = {2024}
}

@article{chhikara2025mem0,
  title   = {{Mem0}: Building Production-Ready {AI} Agents with Scalable Long-Term Memory},
  author  = {Chhikara, Prateek and Khant, Dev and Aryan, Saket and Singh, Taranjeet and Yadav, Deshraj},
  journal = {arXiv preprint arXiv:2504.19413},
  year    = {2025}
}

@article{xu2025amem,
  title   = {{A-MEM}: Agentic Memory for {LLM} Agents},
  author  = {Xu, Wujiang and Liang, Zujie and Mei, Kai and Gao, Hang and Tan, Juntao and Zhang, Yongfeng},
  journal = {arXiv preprint arXiv:2502.12110},
  year    = {2025}
}

@article{ge2024personahub,
  title   = {Scaling Synthetic Data Creation with 1,000,000,000 Personas},
  author  = {Ge, Tao and Chan, Xin and Wang, Xiaoyang and Yu, Dian and Mi, Haitao and Yu, Dong},
  journal = {arXiv preprint arXiv:2406.20094},
  year    = {2024}
}

@inproceedings{mireshghallah2024confaide,
  title     = {Can {LLM}s Keep a Secret? Testing Privacy Implications of Language Models via Contextual Integrity Theory},
  author    = {Mireshghallah, Niloofar and Kim, Hyunwoo and Zhou, Xuhui and Tsvetkov, Yulia and Sap, Maarten and Shokri, Reza and Choi, Yejin},
  booktitle = {International Conference on Learning Representations},
  year      = {2024}
}

@inproceedings{camburu2018esnli,
  title     = {e-{SNLI}: Natural Language Inference with Natural Language Explanations},
  author    = {Camburu, Oana-Maria and Rockt{\"a}schel, Tim and Lukasiewicz, Thomas and Blunsom, Phil},
  booktitle = {Advances in Neural Information Processing Systems},
  year      = {2018}
}

@article{xiao2026alpsbench,
  title={AlpsBench: An LLM Personalization Benchmark for Real-Dialogue Memorization and Preference Alignment},
  author={Xiao, Jianfei and Yu, Xiang and Wang, Chengbing and Zheng, Wuqiang and Lin, Xinyu and Liu, Kaining and Ding, Hongxun and Zhang, Yang and Wang, Wenjie and Feng, Fuli and others},
  journal={arXiv preprint arXiv:2603.26680},
  year={2026}
}

@article{liu2026companionbench,
  title={CompanionBench: A Theory-Anchored, Real-World-Grounded Benchmark for AI Emotional Companionship},
  author={Liu, Yao and Chai, Guangjia and Huang, Yuming and Huang, Jihao and Wang, Lei and Wan, Junchen},
  journal={arXiv preprint arXiv:2608.02046},
  year={2026}
}

@article{li2026horizonbench,
  title={Horizonbench: Long-horizon personalization with evolving preferences},
  author={Li, Shuyue Stella and Paranjape, Bhargavi and Oktar, Kerem and Ma, Zhongyao and Zhou, Gelin and Guan, Lin and Zhang, Na and Park, Sem and Chen, Lin and Yang, Diyi and others},
  journal={arXiv preprint arXiv:2604.17283},
  year={2026}
}

@inproceedings{kirk2024prism,
  title     = {The {PRISM} Alignment Dataset: What Participatory, Representative and Individualised Human Feedback Reveals About the Subjective and Multicultural Alignment of Large Language Models},
  author    = {Kirk, Hannah Rose and Whitefield, Alexander and R{\"o}ttger, Paul and Bean, Andrew and Margatina, Katerina and Ciro, Juan and Mosquera, Rafael and Bartolo, Max and Williams, Adina and He, He and Vidgen, Bertie and Hale, Scott A.},
  booktitle = {Advances in Neural Information Processing Systems, Datasets and Benchmarks Track},
  year      = {2024}
}

@article{gebru2021datasheets,
  title   = {Datasheets for Datasets},
  author  = {Gebru, Timnit and Morgenstern, Jamie and Vecchione, Briana and
             Vaughan, Jennifer Wortman and Wallach, Hanna and
             Daum{\'e} III, Hal and Crawford, Kate},
  journal = {Communications of the ACM},
  volume  = {64},
  number  = {12},
  pages   = {86--92},
  year    = {2021}
}

@inproceedings{zhang2018personachat,
  title     = {Personalizing Dialogue Agents: {I} have a dog, do you have pets too?},
  author    = {Zhang, Saizheng and Dinan, Emily and Urbanek, Jack and
               Szlam, Arthur and Kiela, Douwe and Weston, Jason},
  booktitle = {Proceedings of the 56th Annual Meeting of the Association for
               Computational Linguistics},
  pages     = {2204--2213},
  year      = {2018}
}

@inproceedings{xu2022goldfish,
  title     = {Beyond Goldfish Memory: Long-Term Open-Domain Conversation},
  author    = {Xu, Jing and Szlam, Arthur and Weston, Jason},
  booktitle = {Proceedings of the 60th Annual Meeting of the Association for
               Computational Linguistics},
  pages     = {5180--5197},
  year      = {2022}
}

@article{packer2023memgpt,
  title   = {{MemGPT}: Towards {LLM}s as Operating Systems},
  author  = {Packer, Charles and Fang, Vivian and Patil, Shishir G. and
             Lin, Kevin and Wooders, Sarah and Gonzalez, Joseph E.},
  journal = {arXiv preprint arXiv:2310.08560},
  year    = {2023}
}

@inproceedings{park2023generative,
  title     = {Generative Agents: Interactive Simulacra of Human Behavior},
  author    = {Park, Joon Sung and O'Brien, Joseph C. and Cai, Carrie J. and
               Morris, Meredith Ringel and Liang, Percy and
               Bernstein, Michael S.},
  booktitle = {Proceedings of the 36th Annual ACM Symposium on User Interface
               Software and Technology (UIST)},
  year      = {2023}
}

@inproceedings{zheng2023judging,
  title     = {Judging {LLM}-as-a-Judge with {MT-Bench} and {Chatbot Arena}},
  author    = {Zheng, Lianmin and Chiang, Wei-Lin and Sheng, Ying and
               Zhuang, Siyuan and Wu, Zhanghao and Zhuang, Yonghao and
               Lin, Zi and Li, Zhuohan and Li, Dacheng and Xing, Eric P. and
               Zhang, Hao and Gonzalez, Joseph E. and Stoica, Ion},
  booktitle = {Advances in Neural Information Processing Systems 36,
               Datasets and Benchmarks Track},
  year      = {2023}
}

@article{skjuve2021companion,
  title   = {My Chatbot Companion: a Study of Human-Chatbot Relationships},
  author  = {Skjuve, Marita and F{\o}lstad, Asbj{\o}rn and
             Fostervold, Knut Inge and Brandtzaeg, Petter Bae},
  journal = {International Journal of Human-Computer Studies},
  volume  = {149},
  pages   = {102601},
  year    = {2021}
}

@article{brandtzaeg2022aifriend,
  title   = {My {AI} Friend: How Users of a Social Chatbot Understand Their
             Human-{AI} Friendship},
  author  = {Brandtzaeg, Petter Bae and Skjuve, Marita and
             F{\o}lstad, Asbj{\o}rn},
  journal = {Human Communication Research},
  volume  = {48},
  number  = {3},
  pages   = {404--429},
  year    = {2022}
}

@article{laestadius2024toohuman,
  title   = {Too human and not human enough: A grounded theory analysis of
             mental health harms from emotional dependence on the social
             chatbot {Replika}},
  author  = {Laestadius, Linnea and Bishop, Andrea and Gonzalez, Michael and
             Illen{\v{c}}{\'i}k, Diana and Campos-Castillo, Celeste},
  journal = {New Media \& Society},
  volume  = {26},
  number  = {10},
  year    = {2024},
  note    = {doi:10.1177/14614448221142007}
}

@article{ta2020socialsupport,
  title   = {User Experiences of Social Support From Companion Chatbots in
             Everyday Contexts: Thematic Analysis},
  author  = {Ta, Vivian and Griffith, Caroline and Boatfield, Carolynn and
             Wang, Xinyu and Civitello, Maria and Bader, Haley and
             DeCero, Esther and Loggarakis, Alexia},
  journal = {Journal of Medical Internet Research},
  volume  = {22},
  number  = {3},
  pages   = {e16235},
  year    = {2020}
}

@misc{gdpr2016,
  title        = {Regulation (EU) 2016/679 of the European Parliament and of the
                  Council of 27 April 2016 on the protection of natural persons
                  with regard to the processing of personal data and on the free
                  movement of such data (General Data Protection Regulation)},
  author       = {{European Parliament and Council of the European Union}},
  howpublished = {Official Journal of the European Union, L 119, pp. 1--88},
  year         = {2016}
}

@article{yoon2026benchpres,
  title   = {{BenchPreS}: A Benchmark for Context-Aware Personalized Preference Selectivity of Persistent-Memory {LLMs}},
  author  = {Yoon, Sangyeon and Kim, Sunkyoung and Hong, Hyesoo and Jeung, Wonje and Kim, Yongil and Seo, Wooseok and Yeen, Heuiyeen and No, Albert},
  journal = {arXiv preprint arXiv:2603.16557},
  year    = {2026}
}

@inproceedings{rajpurkar2018know,
  title={Know what you don’t know: Unanswerable questions for SQuAD},
  author={Rajpurkar, Pranav and Jia, Robin and Liang, Percy},
  booktitle={Proceedings of the 56th Annual Meeting of the Association for Computational Linguistics (Volume 2: Short Papers)},
  pages={784--789},
  year={2018}
}

@article{liu2024lost,
  title={Lost in the middle: How language models use long contexts},
  author={Liu, Nelson F and Lin, Kevin and Hewitt, John and Paranjape, Ashwin and Bevilacqua, Michele and Petroni, Fabio and Liang, Percy},
  journal={Transactions of the association for computational linguistics},
  volume={12},
  pages={157--173},
  year={2024}
}

@inproceedings{anantha2021open,
  title={Open-domain question answering goes conversational via question rewriting},
  author={Anantha, Raviteja and Vakulenko, Svitlana and Tu, Zhucheng and Longpre, Shayne and Pulman, Stephen and Chappidi, Srinivas},
  booktitle={Proceedings of the 2021 Conference of the North American Chapter of the Association for Computational Linguistics: Human Language Technologies},
  pages={520--534},
  year={2021}
}

@article{mo2025survey,
  title={A survey of conversational search},
  author={Mo, Fengran and Mao, Kelong and Zhao, Ziliang and Qian, Hongjin and Chen, Haonan and Cheng, Yiruo and Li, Xiaoxi and Zhu, Yutao and Dou, Zhicheng and Nie, Jian-Yun},
  journal={ACM Transactions on Information Systems},
  volume={43},
  number={6},
  pages={1--50},
  year={2025},
  publisher={ACM New York, NY}
}

@inproceedings{asai2024self,
  title={Self-rag: Learning to retrieve, generate, and critique through self-reflection},
  author={Asai, Akari and Wu, Zeqiu and Wang, Yizhong and Sil, Avi and Hajishirzi, Hannaneh},
  booktitle={International conference on learning representations},
  volume={2024},
  pages={9112--9141},
  year={2024}
}

@inproceedings{jeong2024adaptive,
  title={Adaptive-rag: Learning to adapt retrieval-augmented large language models through question complexity},
  author={Jeong, Soyeong and Baek, Jinheon and Cho, Sukmin and Hwang, Sung Ju and Park, Jong C},
  booktitle={Proceedings of the 2024 conference of the north american chapter of the association for computational linguistics: Human language technologies (volume 1: Long papers)},
  pages={7036--7050},
  year={2024}
}

@inproceedings{mallen2023not,
  title={When not to trust language models: Investigating effectiveness of parametric and non-parametric memories},
  author={Mallen, Alex and Asai, Akari and Zhong, Victor and Das, Rajarshi and Khashabi, Daniel and Hajishirzi, Hannaneh},
  booktitle={Proceedings of the 61st annual meeting of the association for computational linguistics (volume 1: Long papers)},
  pages={9802--9822},
  year={2023}
}

@article{venkit2026need,
  title={The need for a socially-grounded persona framework for user simulation},
  author={Venkit, Pranav Narayanan and Li, Yu and Pruksachatkun, Yada and Wu, Chien-Sheng},
  journal={arXiv preprint arXiv:2601.07110},
  year={2026}
}

@article{venkit2026best,
  title={Best Friends, Not Forever: Evaluating Long-Horizon Persona Collapse and Behavioral Drift in AI Companions},
  author={Venkit, Pranav Narayanan and Prabhakar, Akshara and Li, Yu and Lee, Daniel and Wu, Chien-Sheng},
  journal={arXiv preprint arXiv:2607.28818},
  year={2026}
}

@article{neplenbroek2026topics,
  title={Topics as Proxies for Sociodemographics: How Conversational Context Affects LLM Answers},
  author={Neplenbroek, Vera and Sarti, Gabriele and Bisazza, Arianna and Fern{\'a}ndez, Raquel},
  journal={arXiv preprint arXiv:2606.02776},
  year={2026}
}

@article{pi2026only,
  title={Only Time Will Tell: A Structured Survey of Longitudinal Studies on Social AI Companions},
  author={Pi, Yulu and Hunter, Rosco},
  journal={International Journal of Human--Computer Interaction},
  pages={1--22},
  year={2026},
  publisher={Taylor \& Francis}
}

@article{lukauskas2026plausible,
  title={Plausible but Not Valid: A Psychometric Audit of LLMs as Synthetic Survey Respondents},
  author={Lukauskas, Mantas and {\v{S}}arkauskait{\.e}, Viktorija},
  journal={arXiv preprint arXiv:2608.14606},
  year={2026}
}

@article{chen2026synthetic,
  title={When Synthetic Users Fail: A Cross-Domain Benchmark of LLM-Simulated Human Survey Responses},
  author={Chen, Zihan and Zhu, Di and Zheng, Lei Nico},
  journal={arXiv preprint arXiv:2607.26348},
  year={2026}
}

@article{maier2025llms,
  title={LLMs Reproduce Human Purchase Intent via Semantic Similarity Elicitation of Likert Ratings. doi: 10.48550},
  author={Maier, Benjamin F and Aslak, Ulf and Fiaschi, Luca and Rismal, Nina and Fletcher, Kemble and Luhmann, Christian C and Dow, Robbie and Pappas, Kli and Wiecki, Thomas V},
  journal={arXiv preprint arXiv.2510.08338},
  year={2025}
}

@inproceedings{ni2026survey,
  title={A survey on llm-based conversational user simulation},
  author={Ni, Bo and Wang, Yu and Wang, Leyao and Kveton, Branislav and Dernoncourt, Franck and Xia, Yu and Chen, Hongjie and Luera, Reuben and Basu, Samyadeep and Mukherjee, Subhojyoti and others},
  booktitle={Proceedings of the 19th Conference of the European Chapter of the Association for Computational Linguistics (Volume 1: Long Papers)},
  pages={4266--4301},
  year={2026}
}

@article{wang2026mind2dialogue,
  title={Mind2Dialogue: Training Human-Aware Language Models by Simulating User Mental States},
  author={Wang, Zixuan and Zhou, Yufan and Tang, Jinzhou and Yu, Xinle and Wu, Chengjun and Ye, Lyumanshan and Feng, Zhaoxiang and Peng, Letian and Patra, Adyasha and Bai, Fan and others},
  journal={arXiv preprint arXiv:2609.15972},
  year={2026}
}

@inproceedings{li2026rectom,
  title={Rectom: A benchmark for evaluating machine theory of mind in llm-based conversational recommender systems},
  author={Li, Mengfan and Shi, Xuanhua and Deng, Yang},
  booktitle={Proceedings of the AAAI Conference on Artificial Intelligence},
  volume={40},
  number={37},
  pages={31636--31644},
  year={2026}
}

@article{iyer2026heart,
  title={Heart: A unified benchmark for assessing humans and llms in emotional support dialogue},
  author={Iyer, Laya and Aggarwal, Kriti and Koyejo, Sanmi and Heyman, Gail and Ong, Desmond C and Mukherjee, Subhabrata},
  journal={arXiv preprint arXiv:2601.19922},
  year={2026}
}

@article{zhao2026stateless,
  title={From Stateless to Situated: Building a Psychological World for LLM-Based Emotional Support},
  author={Zhao, Boning and Hu, Clover and Li, Xinnuo},
  journal={arXiv preprint arXiv:2603.25031},
  year={2026}
}

@article{zou2026affective,
  title={Affective Flow Language Model for Emotional Support Conversation},
  author={Zou, Chenghui and Wang, Ning and Shen, Tiesunlong and Xiao, Luwei and Ma, Chuan and Li, Xiangpeng and Mao, Rui and Cambria, Erik},
  journal={arXiv preprint arXiv:2602.08826},
  year={2026}
}

@article{yadav2026dialtom,
  title={DialToM: A Theory of Mind Benchmark for Forecasting State-Driven Dialogue Trajectories},
  author={Yadav, Neemesh and Achananuparp, Palakorn and Jiang, Jing and Lim, Ee-Peng},
  journal={arXiv preprint arXiv:2604.20443},
  year={2026}
}

@article{bawatneh2026omnitom,
  title={OmniToM: Benchmarking Theory of Mind in LLMs via Explicit Belief Modeling},
  author={Bawatneh, Adam and Sapkota, Sagar and Bedi, Amrit Singh and Karmaker, Santu and Shah, Mubarak},
  journal={arXiv preprint arXiv:2605.26322},
  year={2026}
}

@article{zhao2026affectverse,
  title={AffectVerse: Emotional World Models for Multimodal Affective Computing},
  author={Zhao, Bo and Ye, Fanghua and Ji, Yixin and Zhao, Sicheng and Peng, Xiaojiang and Yu, Zitong},
  journal={arXiv preprint arXiv:2605.19950},
  year={2026}
}

@article{cheng2026using,
  title={Using Cognitive Models to Improve Language Model Simulation of Human Persuasion Games},
  author={Cheng, Zirui and Shen, Zeyu and Griffiths, Thomas L and Henderson, Peter},
  journal={arXiv preprint arXiv:2606.17657},
  year={2026}
}

@article{lin2506large,
  title={Large language models as psychological simulators: A methodological guide, 2025},
  author={Lin, Zhicheng},
  journal={URL https://arxiv. org/abs/2506.16702}
}

@article{zhang2025rise,
  title={The rise of AI companions: Interaction with AI companions and psychological well-being},
  author={Zhang, Yutong and Zhao, Dora and Hancock, Jeffrey T and Kraut, Robert and Yang, Diyi},
  journal={arXiv preprint arXiv:2506.12605},
  year={2025}
}

@article{tigre2026can,
  title={When Can You Trust Your Synthetic Users? Diagnostics and Corrections for LLM Consumer Panels},
  author={Tigre, Robson and Souto, Hugo Gobato},
  journal={arXiv preprint arXiv:2609.13148},
  year={2026}
}

@article{chen2026towards,
  title={Towards real-world human behavior simulation: Benchmarking large language models on long-horizon, cross-scenario, heterogeneous behavior traces},
  author={Chen, Jiawei and Xu, Ruoxi and Cao, Boxi and Pan, Ruotong and Zhang, Yunfei and Hu, Yifei and Du, Yong and Gao, Tingting and Lu, Yaojie and Sun, Yingfei and others},
  journal={arXiv preprint arXiv:2604.08362},
  year={2026}
}

@article{zhang2026lunar,
  title={LUNAR: Benchmarking Personalized Large Language Models on UNiversal User BehAvioR Logs},
  author={Zhang, Jiahao and Tong, Yongzhi and Fu, Zelin and Zhao, Pengde and Jiang, Yanmei and Feng, Jiang and Yang, Min},
  journal={arXiv preprint arXiv:2608.05246},
  year={2026}
}

@misc{meyer2025nemotron,
  author       = {Meyer, Yev and Corneil, Dane},
  title        = {{Nemotron-Personas-USA}: Synthetic Personas Aligned to Real-World Distributions},
  month        = jun,
  year         = {2025},
  howpublished = {\url{https://huggingface.co/datasets/nvidia/Nemotron-Personas-USA}}
}

@article{zhang2025community,
  title   = {Cultivating Pluralism in Algorithmic Monoculture: The Community Alignment Dataset},
  author  = {Zhang, Lily Hong and Milli, Smitha and Jusko, Karen and Smith, Jonathan and Amos, Brandon and Bouaziz, Wassim and Revel, Manon and Kussman, Jack and Sheynin, Yasha and Titus, Lisa and Radharapu, Bhaktipriya and Yu, Jane and Sarma, Vidya and Rose, Kris and Nickel, Maximilian},
  journal = {arXiv preprint arXiv:2507.09650},
  year    = {2025}
}

@article{omnibehavior2026,
  title   = {Towards Real-world Human Behavior Simulation: Benchmarking Large Language Models on Long-horizon, Cross-scenario, Heterogeneous Behavior Traces},
  author  = {Chen, Jiawei and Xu, Ruoxi and Cao, Boxi and Pan, Ruotong and Zhang, Yunfei and Hu, Yifei and Du, Yong and Gao, Tingting and Lu, Yaojie and Sun, Yingfei and Han, Xianpei and Sun, Le and Wu, Xiangyu and Lin, Hongyu},
  journal = {arXiv preprint arXiv:2604.08362},
  year    = {2026}
}

@article{lunar2026,
  title   = {{LUNAR}: Benchmarking Personalized Large Language Models on {UNiversal} User {BehAvioR} Logs},
  author  = {Zhang, Jiahao and Tong, Yongzhi and Fu, Zelin and Zhao, Pengde and Jiang, Yanmei and Jiang, Feng and Yang, Min},
  journal = {arXiv preprint arXiv:2608.05246},
  year    = {2026}
}

\clearpage
\normalsize
\part*{Appendix}
\startcontents[appendix]
\printcontents[appendix]{}{1}{\setcounter{tocdepth}{2}}
\clearpage

\appendix
\section{The Evaluation Task}
\label{app:task}

Section~\ref{sec:bg} states the task informally. This appendix gives it formally: the objects an instance supplies (\ref{app:objects}), how the three decisions are scored and aggregated (\ref{app:scoring}), what follows for measurement when the demand rate is low (\ref{app:lowdemand}), and the conditions a corpus must meet to count as an instance at all (\ref{app:validity}). The definitions are stated over an arbitrary subject record, because \bench is one instance of the task.

% ---------- A.1 ----------
\subsection{Subject Records and Probes}
\label{app:objects}

% P1 - the object an instance provides
The task is defined over a record of one person, independently of how that record was obtained. A \emph{subject record} is a pair $(M, \mathcal{S})$. The stream $M = (m_1, \ldots, m_N)$ is a sequence of messages ordered in time, each carrying a sender, a timestamp, and its text, with senders partitioned into those attributable to the person and those attributable to anything else. The family $\mathcal{S} = \{S_1, \ldots, S_L\}$ is a finite set of \emph{stores}, each a set of records derived from $M$ and each standing for one place a system might look for something the person said before; for \bench, $L = 2$, the profile claims and, as the episode store, the messages of $M$ themselves. An \emph{instance} of the task is a finite collection of subject records together with the annotation defined below. \bench is one instance; Appendix~\ref{app:corpus} describes it, and Appendix~\ref{app:validity} states the conditions any instance must satisfy.

% P2 - probes, history, window
A \emph{probe} is a message $p = m_i$ attributable to the person. Its \emph{history} is $H(p) = (m_1, \ldots, m_{i-1})$ and its \emph{recent window} of size $k$ is $W_k(p) = (m_{i-k}, \ldots, m_{i-1})$, truncated at the start of the stream; the complement $\overline{W}_k(p) = H(p) \setminus W_k(p)$ is the \emph{distant history}. Each probe carries a recorded window size $k(p)$, a \emph{reference set} $R(p) \subseteq \bigcup_l S_l$ of records whose evidence precedes $p$ and lies outside $W_{k(p)}(p)$, and a \emph{gold response} $g(p)$. The reference set names the records the gold response used and nothing else. It is disjoint from the recorded window by construction, and whether a referent was on screen in a wider sense is a separate judgment, recorded per probe in Appendix~\ref{app:onscreen}.

% P3 - the sufficiency axiom
The annotation is required to satisfy one axiom, which is what makes a failure attributable to a system:
\begin{equation}
\label{eq:suff}
g(p) \ \text{is derivable from} \ \bigl(p,\ W_{k(p)}(p),\ R(p)\bigr) \quad \text{for every probe } p . \end{equation} 
The conditioning is on the triple and not on $R(p)$ alone, because a gold response is always entitled to use the probe itself and the messages of the recorded window. An instance that violates~\eqref{eq:suff} is asking for a response that cannot be produced from what it supplies, and no score computed on it is interpretable. Derivable means that a procedure shown only the triple can produce $g(p)$, which is how stage D writes it. On \bench the check is mechanical: the stage D grounding citations name a message outside the recorded context on 419 of the 1{,}533 scoreable items, and on every one of them that message is the probe itself.

% P4 - the three decisions
Three decisions are evaluated, in this order because each is defined on the output of the last. \emph{Restraint} is the prediction $\hat{\rho}(p) \in \{0,1\}$ of whether the probe requires anything beyond the recent window, against the ground truth $\rho^\star(p) = \mathbf{1}[R(p) \neq \emptyset]$. \emph{Locus} is, for probes with $\rho^\star(p) = 1$, the prediction of which store or stores the references lie in, against $\lambda^\star(p) = \{ l : R(p) \cap S_l \neq \emptyset \}$. \emph{Response} is the reply $\hat{y}(p)$, scored against $g(p)$. A system that never retrieves makes no restraint error on probes with $\rho^\star = 0$ and scores zero on locus. A system that always retrieves errs on every such probe and can still score on locus. Neither degenerate strategy is excluded by construction, which is the point of reporting all three.

% P5 - the base rate is a property of the instance
The difficulty of restraint is governed by a single quantity, the \emph{demand rate}
\begin{equation}
\label{eq:pi}
\pi \;=\; \Pr\bigl[\rho^\star(p) = 1\bigr]
\end{equation}
taken over probes drawn as the instance specifies. This is a property of the instance, not of the task, and it is where corpora written for evaluation and corpora drawn from use separate. In a corpus whose probes were composed so that a planted fact would be needed, $\pi = 1$ by construction, the restraint decision is vacuous, and only locus and response carry information. In a corpus whose probes are the person's own turns taken without selection, $\pi$ is whatever it is, and estimating it is itself a measurement. Because $\pi$ is small in a record of ordinary use, restraint dominates the error
budget, and authored probes cannot pose it at the natural rate. A system that is correct on every probe with $\rho^\star = 1$ and retrieves on every other probe makes a restraint error on a $(1-\pi)$ fraction of its inputs. The response score does not register that cost, since supplied context can raise it even where nothing was needed (Table~\ref{tab:decomp}).

% P6 - the two tracks
The formalism is indifferent to where probes come from, and \bench uses two sources over the same subject records. In the \emph{chat track} a probe is a message of the stream itself, so $H(p)$ is the true history of that message, and the gold response is a reference reply written under~\eqref{eq:suff}. The reply the participant actually received sits in $M$ as the next message. In the \emph{question track} a probe is a question posed over the record, so $H(p)$ is the whole stream and $\rho^\star$, $\lambda^\star$ and $g$ are annotated the same way. Only the proportional stratum of the chat track supports claims about $\pi$, because only there are probes sampled from what the person said without reference to demand.

% ---------- A.2 ----------
\subsection{Scoring and Aggregation}
\label{app:scoring}

% P1 - populations are fixed by the ground truth
Four quantities are reported: one for restraint, one for locus, and two for the response, which ask whether it draws on the past where it should (attribution) and whether it agrees with the gold (consistency). Retrieval itself is scored separately in Appendix~\ref{app:baselines}. Each is defined on a population fixed by the annotation and never by the system under test, so that a system cannot enlarge or shrink the set of probes it is graded on by changing its behavior. Table~\ref{tab:metrics} collects the definitions. Throughout, $\hat\rho(p)$, $\hat\lambda(p)$, $\hat{R}(p)$ and $\hat{y}(p)$ denote a system's restraint decision, locus prediction, retrieved record set and response, and $\rho^\star$, $\lambda^\star$, $R$ and $g$ the corresponding ground truth of Appendix~\ref{app:objects}.

\begin{table}[h]
\centering
\caption{Scored quantities and their populations.}
\label{tab:metrics}
\begin{tabular}{llrl}
\toprule
Quantity & Scored on & Items & Aggregation \\
\midrule
Quantity & Scored on & Items & Aggregation \\
\midrule
Misfire rate~\eqref{eq:mr}          & $\rho^\star = 0$ & 1{,}129   & pooled and macro \\
Attribution rate~\eqref{eq:aar}     & $\rho^\star = 1$ & 404 / 250 & pooled and macro \\
Locus accuracy~\eqref{eq:la}        & $\rho^\star = 1$ & 404 / 250 & pooled and macro \\
Response consistency~\eqref{eq:rc}  & all scoreable    & 1{,}533   & pooled and macro \\
\bottomrule
\end{tabular}
\end{table}

% P2 - restraint
Restraint is scored as a false-alarm rate, following the misfire convention of
\citet{yoon2026benchpres}:
\begin{equation}
\label{eq:mr}
\mathrm{MR} \;=\; \Pr\bigl[\hat\rho(p) = 1 \;\bigm|\; \rho^\star(p) = 0\bigr],
\end{equation}
the rate at which a system reaches beyond the recent window when nothing beyond it is required. The population is every probe with $R(p) = \emptyset$, which is 1{,}129 of the 1{,}533 scoreable probes. We take this population, because the decision a deployed system actually faces is whether to retrieve beyond what is already on screen, and that is exactly $\rho^\star$. The narrower figure, computed on the 120 probes whose locus is \texttt{none}, is reported alongside it in Appendix~\ref{app:onscreen} for comparability with prior work, and the two differ substantially; a reader comparing across papers should check which population a reported misfire rate was taken on.

% P3 - attribution and locus
Attribution is scored only where retrieval is required:
\begin{equation}
\label{eq:aar}
\mathrm{AAR} \;=\; \Pr\bigl[\,\mathrm{uses}(\hat y(p)) = 1 \;\bigm|\; \rho^\star(p) = 1\,\bigr],
\end{equation}
where $\mathrm{uses}(\cdot)$ records whether the response asserts a detail about the past that the probe did not supply. Locus is scored on the same population as exact agreement on the set of stores,
\begin{equation}
\label{eq:la}
\mathrm{LA} \;=\; \Pr\bigl[\hat\lambda(p) = \lambda^\star(p) \;\bigm|\;
\rho^\star(p) = 1\bigr].
\end{equation}
The population is 404 probes under the recorded reading and 250 under the strict reading of Appendix~\ref{app:onscreen}; both are reported for every system.

% P4 - response
Response is scored on all 1{,}533 scoreable probes by a judge $J$ returning a binary verdict of consistency with the gold,
\begin{equation}
\label{eq:rc}
\mathrm{RC} \;=\; \Pr\bigl[J(\hat{y}(p), g(p)) = 1\bigr].
\end{equation}
The judge is given the probe, the gold response and the candidate, and is not given the reference set, so that it cannot reward a response for citing evidence the gold did not use. Its prompt and what it is and is not shown are given in Appendix~\ref{app:judge}, which also reports its agreement with human adjudication, $\kappa = 0.804$ on the binary verdict and $0.545$ on the graded scale. Response is the only quantity on which a system that never retrieves and a system that always retrieves are directly comparable, which is why it is reported on the full population.

% P5 - both aggregations, and which is used where
Both aggregations are reported, because they answer different questions. A quantity pooled over items estimates the corpus. A quantity macro-averaged over subjects,
\begin{equation}
\label{eq:macro}
\bar\theta \;=\; \frac{1}{U}\sum_{u=1}^{U} \theta_u ,
\qquad U,
\end{equation}
where $U$ is the number of participants with at least one probe in the population: ten for misfire and consistency, seven for attribution and locus under the recorded reading, and five under the strict reading. $\theta_u$ is computed on subject $u$ alone, estimates a relationship. The two differ because two subjects supply 19{,}870 of the 27{,}218 messages and would determine any pooled figure on their own. Neither is right in general. The demand split of Appendix~\ref{app:lowdemand} is reported pooled, because the identity of Proposition~\ref{prop:nonident} holds on pooled means. Every figure in this paper names which aggregation it is.

% P6 - uncertainty
Proportions reported alone carry Wilson intervals over items. Comparisons carry percentile intervals from a paired bootstrap that resamples participants with replacement, because participants are the unit over which the estimate is averaged and probes within a participant are not independent. For a comparison between two systems the same resample is applied to both and the interval is taken on the difference, so the paired structure is preserved. With ten subjects these intervals are wide, and we report them as they come out.

% P7 - strata
The strata of Appendix~\ref{app:probes} carry different work and are never pooled for a rate. The proportional stratum is a random sample of user turns, so it alone supports statements about how often something occurs in the traffic, including every estimate of $\pi$. The enriched stratum is an exhaustive sweep for probes of a given shape, so it supplies statistical power for comparisons conditioned on that shape but has no denominator of its own. A rate quoted without a stratum is not defined, and every rate in this paper names one. The abstention control of 434 cold opens is scored apart and enters no rate.
% ---------- A.3 ----------
\subsection{Measurement under a Low Demand Rate}
\label{app:lowdemand}

% P1 - what follows from pi being small
Three consequences follow from the demand rate $\pi$ of~\eqref{eq:pi} being small, and they bear on what a corpus can establish. The first says which claims the derivation licenses and which it does not. The second says the decision restraint asks for cannot be made from the probe. The third says the contrast the field uses to demonstrate that memory helps does not identify a memory effect at all once $\pi$ is small.

% P2 - Proposition 1
Verification in Appendix~\ref{app:phases} may remove a reference and may never add one. Write $\rho^\circ(p)$ for whether a probe truly needs anything beyond its window and $\pi^\circ = \Pr[\rho^\circ = 1]$ for the true demand rate, which the annotation $\rho^\star$ and the measured rate $\pi$ of~\eqref{eq:pi} estimate. Write $r = \Pr[\rho^\star = 1 \mid \rho^\circ = 1]$ for the recall of the annotation and $\varepsilon = \Pr[\rho^\star = 1 \mid \rho^\circ = 0]$ for its false-positive rate.

\begin{proposition}[One-sided error]
\label{prop:onesided}
$\pi = \pi^\circ r + (1-\pi^\circ)\varepsilon$. Consequently $\pi \le \pi^\circ + (1-\pi^\circ)\varepsilon$, and $\Pr[\rho^\circ = 1 \mid \rho^\star = 1] = 1 - (1-\pi^\circ)\varepsilon / \pi$.
\end{proposition}

\begin{proof}
A probe carries $\rho^\star = 1$ only if stage A proposed a reference and stage B verified it against \texttt{source}, and no later stage introduces one. The law of total probability gives the first identity, the bound follows from $r \le 1$, and the precision identity follows from Bayes' rule.
\end{proof}

% P3 - what it licenses and what it does not
The second half is what the paper uses. When $\varepsilon$ is small the labeled memory-bearing population is almost entirely genuine, so every quantity conditioned on it estimates the corresponding quantity conditioned on $\rho^\star = 1$. That licenses the attribution rate, locus accuracy and the memory effect $\gamma_1$ of
Proposition~\ref{prop:nonident}. The first half runs the other way. When $\varepsilon$ is also negligible, $\pi \approx \pi^\circ r \le \pi^\circ$, so the measured rate is a lower bound on the true one, and the rarity claim rests on evidence about $r$. The census of Appendix~\ref{app:probes} rules out misses from sampling, since stage A ran over every user message. It does not rule out stage A failing to see a genuine referent.

% P4 - detectability
For any feature map $\phi$ of the probe text alone, write $D(\phi) = I\bigl(\phi(p); \rho^\star(p)\bigr)$ for its surface detectability. In a corpus whose probes are drawn, $D$ is an empirical quantity, and this corpus estimates it twice. A lexical scorer's own confidence reaches an area under the ROC curve of 0.546 against 0.500 for chance. A model shown the probe and its three preceding turns and asked directly reaches 0.530 [0.488, 0.572] on the same split at the natural rate, and 0.601 [0.527, 0.676] against its own target, where its threshold catches 6 of 40 and raises 40 false alarms. Independently, the sampler routed every turn into a bucket before any verification ran, and its buckets disagree with the verified tier on 44.7\% of rows, and the bucket it was most confident carried a distant memory demand yielded 10 hard-tier rows in 97. Both measure only the detectors that were tried, so neither shows that $D = 0$ for every $\phi$. The practical reading is that a cheap surface gate does not exist here, which is the assumption every deployed memory system is built on.

% P5 - Proposition 2
Let $\Delta = \mathrm{RC}(C_2) - \mathrm{RC}(C_1)$ be the response-score difference between supplying the recorded evidence turns and withholding them, and let $\gamma_1 = \mathbb{E}[\Delta \mid \rho^\star = 1]$ and $\gamma_0 = \mathbb{E}[\Delta \mid \rho^\star = 0]$.

\begin{proposition}[The memory effect is not identified by a pooled ablation]
\label{prop:nonident}
$\Delta = \pi\gamma_1 + (1-\pi)\gamma_0$. If response scores lie in $[0, G]$, every value $v$ with $|v| + |\Delta - v| \le G$ is attained by some $(\pi, \gamma_1, \gamma_0)$ with $\pi\gamma_1 = v$ that reproduces $\Delta$. The observed $\Delta$ therefore leaves $\pi\gamma_1$ anywhere in $[(\Delta - G)/2,\ (\Delta + G)/2]$, an interval that contains both $0$ and $\Delta$. Identifying $\pi\gamma_1$ requires a per-probe observation of $\rho^\star$ and a probe population with $0 < \pi < 1$.
\end{proposition}

\begin{proof}
The decomposition is the law of total expectation. The differences $\gamma_1$ and $\gamma_0$ lie in $[-G, G]$. Given $v$ with $|v| + |\Delta - v| \le G$, choose $\pi$ with $|v|/G \le \pi \le 1 - |\Delta - v|/G$, and set $\gamma_1 = v/\pi$ and $\gamma_0 = (\Delta - v)/(1-\pi)$. Both lie in $[-G, G]$, and the triple reproduces $\Delta$ exactly. Observing $\rho^\star$ per probe determines $\pi$ and splits the sample, which identifies $\gamma_1$ and $\gamma_0$ separately. At $\pi = 1$ the population defining $\gamma_0$ is empty, and at $\pi = 0$ the one defining $\gamma_1$ is.
\end{proof}

% P6 - Remark
\begin{remark}
\label{rem:authored}
In a corpus whose probes were each written to require planted evidence, $\rho^\star \equiv 1$. Restraint is then not difficult but undefined, $D(\phi) = 0$ for every $\phi$ because the label does not vary, and Proposition~\ref{prop:nonident}'s second condition fails, so no amount of additional data from such a corpus identifies the memory effect. The limitation is structural.
\end{remark}

\begin{table}[h]
\centering\small
\begin{tabular}{llrrrrrr}
\toprule
Stratum & Reading & $\pi$ & $\gamma_1$ & $\gamma_0$ & $\pi\gamma_1$ & $(1-\pi)\gamma_0$ & From no-memory rows \\
\midrule
both & recorded & 0.264 & $+0.233$ & $+0.107$ & $+0.061$ & $+0.079$ & 56\% \\
both & strict   & 0.163 & $+0.256$ & $+0.115$ & $+0.042$ & $+0.096$ & 70\% \\
proportional & recorded & 0.034 & $+0.125$ & $+0.107$ & $+0.004$ & $+0.104$ & 96\% \\
proportional & strict   & 0.013 & $+0.333$ & $+0.105$ & $+0.004$ & $+0.104$ & 96\% \\
\bottomrule
\end{tabular}
\caption{Decomposition of the retrieval contrast by ground-truth demand.}
\label{tab:decomp}
\end{table}

% P7 - the instance
Table~\ref{tab:decomp} gives the decomposition on this corpus. On both strata $\pi$ is the memory-bearing share of the released set, fixed by the enrichment, so those rows describe the design; only the proportional rows estimate a demand rate. The second term is usually assumed away on the grounds that context a system did not need should not help it, and it does help: on the proportional stratum $\gamma_0 = +0.107$, because the retrieved context also supplies the recent turns the gold actually used. With $\pi = 0.034$ the first term contributes $+0.004$ of a pooled $+0.108$, so 96\% of the measured benefit of retrieval comes from probes that require no memory. The per-subject figures are in Appendix~\ref{app:persubject}.

% P8 - the generalization
Proposition~\ref{prop:nonident} is about the field and not about this experiment. Any benchmark that demonstrates the value of memory by supplying retrieved context and measuring the improvement reports $\pi\gamma_1 + (1-\pi)\gamma_0$, and may attribute that improvement to memory only if $\gamma_0$ is zero or $\pi$ is near one. On authored corpora $\pi$ is one by construction, the second term vanishes identically, and the assumption is invisible because it is also true. On real traffic $\pi$ is small and the assumption fails, here by a factor of twenty-four.

% ---------- A.4 ----------
\subsection{Validity Conditions for an Instance}
\label{app:validity}

% P1 - why conditions are needed at all
Appendix~\ref{app:objects} defines the task over any subject record, which is deliberate: \bench is one instance of the task and not the task itself. That generality is only useful if it comes with conditions, because a collection of transcripts and labels can satisfy the definitions formally while measuring nothing. Seven conditions are required. Each is stated with the failure it prevents, because a condition whose violation costs nothing is not a condition. Table~\ref{tab:validity} records where each is established for \bench. Two of them, V2 and V6, are conditions our own first annotation violated, which is why they are stated as conditions.

\begin{table}[h]
\centering
\caption{Validity conditions and where each is established.}
\label{tab:validity}
\begin{tabular}{llll}
\toprule
& Condition & Status in \bench & Established in \\
\midrule
V1 & Sufficiency verified & holds & \ref{app:phases} \\
V2 & References off screen & holds under the strict reading & \ref{app:onscreen} \\
V3 & One unconditioned stratum & holds & \ref{app:probes} \\
V4 & Stores fixed in advance & holds & \ref{app:labels} \\
V5 & An underived layer, cited & holds & \ref{app:refs} \\
V6 & Refutation only demotes & holds after repair & \ref{app:withdraw} \\
V7 & Non-subject evidence bounded & 7.1\% outside the companion locus & this section \\
\bottomrule
\end{tabular}
\end{table}

% P2 - V1 sufficiency
\emph{(V1) Sufficiency is verified, not asserted.} The instance must supply a procedure that checks~\eqref{eq:suff} on every probe and must report the probes on which it fails. Without the check, a gold response may quietly depend on a record the instance never names, and a system is then graded on evidence it was not given. The grounding must be tested against $W_{k(p)}(p) \cup R(p) \cup \{p\}$, as~\eqref{eq:suff} states. Testing against $R(p)$ alone fails nearly every one of the 1{,}129 probes with $R(p) = \emptyset$.

% P3 - V2 disjointness
\emph{(V2) The reference set is disjoint from the recent window.} A record that is already on screen is not something a system must retrieve, and counting it as one inflates the demand rate $\pi$ and credits a system for retrieval it never needed to perform. The condition is easy to state and hard to satisfy, because whether a referent is on screen is a judgment about the text. An instance that cannot guarantee V2 by construction must do the next best thing: record the judgment per probe as released evidence, and report every rate under both the reading that assumes V2 and the reading that enforces it. \bench takes this route. The field and its criterion are in Appendix~\ref{app:onscreen}, and the principal rates are given under both readings. Against the recorded window the condition holds exactly. Across the 1{,}533 scoreable probes, none of the 1{,}989 profile evidence messages and none of the 381 episode reference messages falls inside it, and a release gate enforces this. Against a fixed window of the ten preceding messages, 11 rows carry a reference inside it. All 11 are among the 154 rows the strict reading removes, which also removes rows whose referent is named on screen while its evidence lies further back.

% P4 - V3 unconditioned sampling
\emph{(V3) At least one stratum is sampled without reference to memory demand.} Any estimate of $\pi$, and any claim of the form ``this happens at rate $r$ in real traffic,'' requires a set of probes selected without knowing the answer. A stratum assembled by searching for probes of an interesting shape has no denominator: it can establish that something occurs and can supply power for comparisons conditioned on that shape, but it cannot say how often the shape occurs. Instances may and should contain both kinds of stratum, and must label which is which. Pooling them yields a rate that estimates no quantity at all.

% P5 - V4 fixed stores
\emph{(V4) The store family is fixed before annotation.} The locus decision is falsifiable only if the set of places a record may live is declared in advance. An instance that introduces a store whenever the annotation finds something that fits nowhere else makes locus unfalsifiable, since every prediction can be accommodated after the fact. \bench fixes two reference stores beside the recent window: profile claims and episode turns. The locus takes five values, of which three are distant, and the companion locus is an episode reference restricted to messages the companion sent. The formal locus $\lambda^\star$ therefore takes a value in the two stores, and the released field refines it: \texttt{thread} and \texttt{none} carry $\rho^\star = 0$, and \texttt{companion} refines \texttt{episode}. No item populates both reference lists, so $|\lambda^\star(p)| = 1$ on every memory-bearing probe.

% P6 - V5 refutability
\emph{(V5) One released file is underived, and every derived claim cites it.} Reasoning traces are worth nothing if a reader cannot check them. The instance must therefore release at least one layer that is not derived from any other, every claim in every derived layer must cite records in that layer, and every citation must resolve by stored identifier to a released record. An instance that ships only derived layers, or that ships citations which resolve nowhere, offers annotation that can be read but not contradicted. This is the condition that separates a trace from an explanation.

% P7 - V6 monotone derivation
\emph{(V6) Refutation only demotes, and demotion resets everything derived from the refuted evidence.} A derivation pipeline that both proposes and verifies must be arranged so that verification can lower a probe's demand label but never raise it. The first half of this gives the error a direction: a pipeline that misses a genuine referent makes memory demand look rarer than it is, so a finding of rarity cannot be an artifact of over-labeling. The second half is where implementations fail. When verification refutes the proposed evidence, every label computed from that evidence must be reset along with it, not only the reference set. A pipeline that empties the reference set but leaves the locus of the refuted evidence in place emits probes whose labels contradict each other, and any locus figure computed from them is wrong. The audit of Appendix~\ref{app:withdraw} found and repaired exactly this failure in our own derivation.

% P8 - V7 evidence the subject did not produce
\emph{(V7) Evidence the subject did not produce is bounded and reported.} An instance drawn from an interaction with a system that itself has memory faces a confound the annotation cannot remove: a referent may have been supplied to the subject. No such corpus can guarantee that this never happens, so the condition is that its rate is measured and published. An instance that does not report it leaves a reader unable to tell recall from repetition.

% P9 - the bound for \bench
The companion that produced these histories had its own memory. A gold that rests on a turn the companion sent may therefore record something the participant was told. This is measurable. On the chat track, 22.0\% [18.3, 26.3] of the 404 memory-bearing rows rest on memory evidence sent entirely by the companion. Sixty-five of those are rows whose verified locus is the companion, where recalling what the companion said is the task. That leaves 7.1\%, 24 of 339: episode 17.2\% [11.4, 25.1] of 116, profile 1.8\% [0.7, 4.5] of 223. On the question track the rate is 7.5\% [6.6, 8.5] of the 3{,}165 items with a needed turn, or 5.5\% [4.8, 6.4] read on the cited turns.

% P10 - the grounding reading
A second reading counts the phase D grounding citations other than the probe. It gives 70.3\% [67.5, 73.0] of the 1{,}045 rows that carry one (basic 72.6, intermediate 94.1, hard 50.6; thread 71.9, none 88.6, profile 94.1, episode 25.2, companion 95.2). We print it because any reader with the release will compute it, and we do not use it as the bound. On a thread row it counts a gold continuing the companion's own previous turn, which is the companion knowing what it just said. The confound is a gold asserting a fact about the participant from a turn the participant never sent, and that lives in the memory evidence.
\section{Related Works}
\label{app:related}

Section~\ref{sec:bg} states the position, and Table~\ref{tab:comparison} gives the comparison. Personalization has no observable ground truth, so the field makes three substitutions. A description stands in for the person, an authored question stands in for the moment that needs memory, and a model's judgment stands in for the target. Each substitution has a cost that can be measured. This appendix sets out the three (\ref{app:sub}), what they cost (\ref{app:cost}), the attempts to remove them (\ref{app:attempts}), adjacent work on memory systems and rationales (\ref{app:adjacent}), what releasing the alternative requires (\ref{app:release}), and the criteria behind the table (\ref{app:comparison}).

\subsection{Three Substitutions}
\label{app:sub}

\paragraph{A description stands in for the person.} The word persona entered dialogue research through PersonaChat \citep{zhang2018personachat}, where crowd-workers were given short sets of persona sentences and asked to converse while conditioned on them. It solved a real problem, since a model with no representation of its user produces generic replies. A persona in this tradition is an input the experimenter controls, and no reply can contradict it, because nothing outside the description exists to check it against. Multi-session chat \citep{xu2022goldfish} extended the format past a single conversation, collecting dialogues that resume across three to five sessions and carrying summaries forward, which made long-range consistency measurable. The personas are still supplied in advance, and the time between sessions is simulated, with workers told to write as if one to seven hours or one to seven days had passed. What accumulates is a record of a task performed on a schedule.

The tradition now operates at industrial scale. Nemotron-Personas-USA
\citep{meyer2025nemotron} releases a million synthetic personas aligned to real-world demographic distributions, each a set of prose facets covering profession, sports, arts, travel and cuisine, alongside demographic fields. Two works measure what the assumption behind such personas costs. \citet{venkit2026need} administer a 141-item sociopsychological protocol to 124 people and report that demographics explain roughly 1.5\% of the variance in how similarly two people respond, so a persona assembled from demographic attributes is very nearly unrelated to the person it describes. \citet{neplenbroek2026topics} append high-stakes advice questions to 8{,}011 PRISM conversations \citep{kirk2024prism} and 26{,}521 from the Community Alignment Dataset \citep{zhang2025community}, and find that conversation topic predicts a model's advice better than inferred sociodemographics, which models infer poorly in any case. What a system adapts to, when it adapts to a conversation, is what the conversation is about.

\paragraph{An authored question stands in for the moment that needs memory.} Work on conversational memory has converged on a common shape. A long interaction history is constructed, a question is posed at the end, and the system is scored on whether it recovers the fact needed to answer. LoCoMo \citep{maharana2024locomo} builds very long machine-generated conversations and queries them in the third person. LongMemEval \citep{wu2024longmemeval} scales the context past a million tokens and targets factual recall across sessions. CloneMem \citep{hu2026clonemem} argues that conversation captures only fragments of a life and moves to diaries and messages generated top-down from Big Five profiles into multi-year arcs. It links each question to evidence units, defines eight reasoning categories including unanswerable items, and reports that flat retrieval beats consolidation-based memory because summarization severs the link to the original traces. Transplanted into conversation, the shape brings an assumption with it. A question states what should be retrieved, so the benchmark decides for the system whether retrieval was warranted, and it scores finding the right material and using it well as a single outcome.

The personalization line adopts the same shape while varying what the question is about. LaMP \citep{salemi2024lamp} formulates personalization as seven classification and generation tasks over public user histories, within a single message. PersonaMem \citep{jiang2025personamem} builds 180 simulated interaction histories from PersonaHub seeds \citep{ge2024personahub} and asks a model to identify the most suitable response given the current state of a user profile. PersonaMem-v2 \citep{jiang2025personamemv2} scales to 1{,}000 personas and shifts to implicit preferences that surface as side effects of ordinary requests, where the models it evaluates reach 37\% to 55\%. CUPID \citep{kim2025cupid} curates 756 interaction histories in which a preference holds only in a specific context, and reports that no model it evaluates exceeds 50\% precision and 65\% recall at identifying which prior context a new request depends on.

A recurring finding across this line is that models fail to update. HorizonBench \citep{li2026horizonbench} constructs 4{,}245 items from 360 simulated users whose preferences evolve through a mental state graph, and reports that all 25 models it evaluates disproportionately select the pre-evolution value. Its authors state that whether the failure holds at the same severity in naturalistic interaction is unresolved, and position the benchmark as preparation ``in advance of longitudinal human data.'' In the companion setting, ANCHOR \citep{venkit2026best} evaluates persona collapse and behavioral drift over 2{,}008 conversations spanning 27 personas, nine interaction schedules and three memory configurations. Trajectory accuracy averages 44.4\%, recall of the user's state sits near four-option chance, and no memory configuration resolves either. ANCHOR asks whether the companion stays itself. The 65 probes in this corpus whose locus is the companion ask the same question of an interaction that happened.

\paragraph{A model's judgment stands in for the target.} Scoring open-ended responses with a language model was introduced at scale by MT-Bench and Chat-bot Arena \citep{zheng2023judging}, which also documented the failure modes: position bias, verbosity bias, and a preference for the judge's own outputs. Those biases are tolerable when the quantity being judged is otherwise observable. In personalization the target is unobserved, so the judge substitutes for it. A recent survey of personalized agents names this among the open problems, alongside the reliance on synthetic users \citep{xu2026survey}. Our design keeps this substitution and bounds it three ways. The generator is held constant across all five conditions, so any preference for its style applies uniformly and cannot order them. The five candidates for a probe are shuffled and scored in one call, which removes position effects and drift in the judge's use of the scale. Evidence coverage, which uses no judge, is reported alongside content match as a check that shares none of its failure modes.

\subsection{What the Substitutions Cost}
\label{app:cost}

The three substitutions are not independent. The second, authoring the probe, is load-bearing. It removes three decisions from the task at once, and each removal is visible as a gap in the literature.

\paragraph{Restraint.} A question announces that something must be looked up, so a benchmark built from questions cannot score a system for deciding whether to look. The methods that learn this decision learn it from the question. Self-RAG \citep{asai2024self} trains reflection tokens that let a model call for passages on demand, Adaptive-RAG \citep{jeong2024adaptive} trains a classifier over query complexity and routes between no retrieval, single-step retrieval and multi-step retrieval, and \citet{mallen2023not} retrieve only for questions about less popular entities. All three read a signal the question supplies. SQuAD 2.0 \citep{rajpurkar2018know} made not answering a first-class label by adding over 50{,}000 questions with no answer in the passage, and the idea carried into conversational memory as the unanswerable categories of CloneMem and of our question track. The conversational analogue is a message that calls for no memory at all, which describes most messages. That case cannot be constructed by authoring, because an author writing a probe that requires nothing has written nothing worth scoring.

\paragraph{Locus.} Before a conversational query can be retrieved against, its references have to be resolved, and conversational search has long studied that operation. QReCC \citep{anantha2021open} pairs 14{,}000 conversations with rewrites that message a context-dependent question into a standalone one, and \citet{mo2025survey} survey the line. That work presupposes a question to rewrite. Stage A (Section~\ref{sec:chatgt}) performs the same operation on an ordinary message and records the outcome as \texttt{locus}, which states where the referent lives. The label exists here because the probes are ordinary messages.

\paragraph{Reach.} An author who knows the answer places its evidence wherever the question needs it, so the distance between probe and evidence is a design choice, and \citet{liu2024lost} show that where evidence sits in a context changes whether a model uses it. Real conversation produces a bimodal distribution instead. Probes that point into the current thread reach back a median of 2 messages and at most 15. Probes that point outside it reach a median of 450 messages to the nearest cited message and 2{,}157 to the furthest, and as far as 12{,}542.

\subsection{Attempts to Remove Them}
\label{app:attempts}

Recent work removes one substitution at a time. Each effort keeps the others, and which one it keeps is the informative part. AlpsBench~\citep{xiao2026alpsbench} curates 2{,}500 long-term interaction sequences from WildChat, a corpus of real exchanges between people and assistants, and evaluates seven memory systems on extraction, update, retrieval, and utilization. It is the closest prior work to ours on data, and it removes the first substitution. It keeps the second, since its evaluation queries are synthesized by a language model from the dialogues. Its data also records a different activity. WildChat users bring tasks to a general assistant, so even its longest sequences are a series of requests.

CompanionBench \citep{liu2026companionbench} draws on de-identified real companion conversations and anchors its rubric in psychological theory. It removes the first substitution at collection and restores it at release, since it releases rewritten persona pairs and evaluates single 20-message sessions against a simulator. Its authors state that such sessions cannot reach the depth of disclosure that develops over a longer relationship. PRISM \citep{kirk2024prism} collects feedback from 1{,}500 people across 75 countries and is the largest body of real human preference data in this area. It is the one effort that removes the third substitution as well as the first, since its judgments come from people. Its unit is a rated single-session conversation, so it establishes what people prefer and leaves open what a system should remember.

Two efforts outside conversation turn to real behavior. OmniBehavior~\citep{omnibehavior2026} assembles a user-simulation benchmark from real behavioral traces and describes itself as the first built wholly from real-world data, and LUNAR~\citep{lunar2026} builds personalization tasks from app behavior logs through a synthesis pipeline grounded in real behavioral patterns. Neither is conversational, so neither poses the decision this paper is about: a trace records what a person did, and a simulator is scored against the next action. What OmniBehavior reports is nonetheless the strongest external statement of why records of this kind are needed. Evaluated on real traces, models converge toward an average and agreeable person, showing hyperactivity and homogenized personas, and losing individual differences and long-tail behavior. That is a failure a written corpus cannot exhibit, because a written corpus has no long tail: its users were composed, and nothing is in them that an author did not put there.

\subsection{Adjacent Work}
\label{app:adjacent}
A parallel line of engineering builds the systems these benchmarks measure, and its design choices constrain evaluation from the other direction. MemGPT~\citep{packer2023memgpt} treats the context window as a level in a memory hierarchy and pages material in and out of external storage under model control. Generative Agents \citep{park2023generative} maintains a memory stream retrieved by recency, importance, and relevance, and reflects over it periodically to produce higher-level statements. mem0~\citep{chhikara2025mem0} and A-Mem \citep{xu2025amem} extract and consolidate, writing distilled memory statements. AI Persona~\citep{wang2024aipersona} carries this furthest for user modeling, maintaining a lifelong profile that an optimizer updates every few sessions without retraining; it is the closest prior work to the profile file we release, and differs in that its claims are not bound to the messages that established them. The last property decides what a benchmark can measure. A system that stores messages can be scored by identifier match against a gold set. A system that stores distilled statements cannot, because there is no identifier to match, which is why Section~\ref{sec:setup} defines a judged recall alongside the exact one and reports the two separately. We evaluate three of these systems as deployed products, since the write policy is part of what is being measured and a re-implementation would replace it with ours.

Attaching rationales to labels has a long history, with e-SNLI \citep{camburu2018esnli} as the canonical example: annotators assign a label and then write why. Those rationales are written after the label by the annotator who assigned it, so a rationale cannot disagree with its label. CloneMem's evidence units and CUPID's preference checklists are partial instances of the same idea, since both expose what a gold answer rests on. Neither records how the gold answer was reached. Each item here carries a reasoning trace instead, recording the procedure that produced the label, the intermediate results that were checked, and the points where a check failed and the label was revised. A stage can contradict the stage before it, and on 685 of the 1{,}533 scoreable items, 44.7\%, the final tier differs from the pre-verification signal. The practical difference is what a reader can do with a disagreement. A rationale can be judged unconvincing, which leaves the label standing. A reasoning trace can be rejected at a named stage, which identifies every other item that stage decided the same way.

\subsection{What Releasing the Alternative Requires}
\label{app:release}

The three substitutions were forced by access, so removing them means releasing material that was withheld for a reason. Two works of literature state what that costs and what it is worth.

Because memory accumulates personal information, a parallel line asks what a system should decline to use. ConfAIde~\citep{mireshghallah2024confaide} and CIMemories~\citep{mireshghallah2026cimemories} apply contextual integrity to model outputs, showing that appropriateness depends on the recipient and the purpose. CIMemories shapes our design for a methodological reason. It argues that leakage and utility must be scored on the same items. The same logic governs retrieval. A memory system that surfaces nothing never intrudes and never helps, which is why restraint and recall are scored on the same probes and why the abstention controls in Section~\ref{sec:chatgt} exist. The release itself is held to the same standard: Appendix~\ref{app:datasheet} reports what an adversary recovers from the files alongside what the files still support.

Outside machine learning, researchers study these relationships directly. Interview studies of companion application users describe how the relationship develops over months and what people come to expect from it \citep{skjuve2021companion}, and follow-on work examines how users construe the bond, including where they treat it as a friendship and where they do not \citep{brandtzaeg2022aifriend}. Thematic analyses document the social support people report receiving \citep{ta2020socialsupport}, and grounded-theory work on emotional dependence documents the harms, including distress when the system's behavior changes \citep{laestadius2024toohuman}.

Every one of these studies works from what people say about the interaction, through interviews, surveys, app reviews and forum posts. That is the constraint that produced the three substitutions, met by a different field and answered a different way. Where machine learning generated a substitute for the interaction, this literature asked people to recall it. This release supplies the interaction itself, which is why every profile claim cites the messages behind it and every chat item carries its reasoning trace. What the corpus cannot supply is what these studies supply. It contains no participant's account, and ten people who chose one product are not a sample of any population.

A conversation between a particular person and an agent over months exists once, belongs to that person, and cannot be regenerated if the release is wrong. That asymmetry is why Appendix~\ref{app:datasheet} measures what an adversary recovers, and why the Ethics Statement states the basis on which the material may be used. Removing the substitutions is a question of consent before it is a question of scale.

\subsection{The Comparison Table}
\label{app:comparison}

% P1 - the criteria
Table~\ref{tab:comparison} marks five properties and groups its rows by a sixth, stated here as criteria a reader can apply to a benchmark the table does not list.

\begin{table}[h]
\centering
\small
\begin{tabular}{lp{10.2cm}}
\toprule
Column & Criterion \\
\midrule
Row group & How the underlying record was produced. \texttt{crowdworker} when people
were recruited to converse, \texttt{generated} when a model produced it, \texttt{public
histories} when it was assembled from public activity, \texttt{real} when people
produced it while using a deployed system for their own purposes. \\
Longitudinal & Yes when the record spans more than one session. It does not require the
elapsed time to be real; MSC instructs workers to write as though days had passed. \\
Probes & How the item put to the system was produced: \texttt{authored} by a person,
\texttt{synthesized} by a model, \texttt{verbatim} when taken unchanged from the record. \\
Evidence & Yes when each item names the specific records its gold answer rests on. \\
Abstention & Yes when the benchmark scores declining to answer or declining to retrieve. \\
Derivations & Yes when each item records how its label was produced, including checks
that failed and labels that were revised. A rationale written after a label is assigned
does not qualify. \\
\bottomrule
\end{tabular}
\caption{Criteria behind each column.}
\label{tab:axes}
\end{table}

% P2 - the n/a convention and what is in the table at all
A cell is half-filled when the benchmark poses no task for which the column is defined, and \texttt{no} when it poses the task and lacks the property. CIMemories, PRISM, and CompanionBench score appropriateness or preference, so the evidence column does not apply to them, while PersonaChat is scored against a record and names no evidence, so it is open. The table lists a benchmark when it scores a response conditioned on a record of a particular person. That criterion excludes work discussed above for other reasons: conversational query rewriting and reading comprehension pose no person; persona corpora supply a record but no evaluation; and memory architectures are systems. Every cell was read from the benchmark's own description of itself, and each is recorded with a page reference in the reproduction map released with the corpus.
\section{The Corpus}
\label{app:corpus}

Section~\ref{sec:dataset} describes the corpus; this appendix specifies it. It covers the application the conversations came from (\ref{app:setting}), the schema of the five released files (\ref{app:schemas}), the features of the medium a reader would otherwise misread (\ref{app:medium}), what timestamps do and do not mean (\ref{app:time}), the per-subject census behind every macro-average (\ref{app:census}), how the profile and persona were built and audited (\ref{app:construction}), and the invariants a reader can check without trusting us (\ref{app:refs}).

% ---------- C.1 ----------
\subsection{The Setting}
\label{app:setting}

% P1 - the application
The corpus comes from an internal AI-companion application in which a user keeps a persistent friendship-style companion. The companion has a name, a mood, an outfit, and the application layers a daily streak, a habit tracker, a snack economy and several mini-games over an open-ended chat. Two properties of this setting matter for the benchmark. The traffic is ordinary contact, made of greetings, accounts of a day, running jokes and long stretches in which nothing is asked. The companion also opens the exchange as often as the participant does, speaking first on 217 of the 430 active days against 207 for a participant and 6 for the application, so a probe often sits inside a conversation the participant did not start. Both properties are costly to script, and a corpus written for evaluation rarely has either.

% P2 - the record
The release contains 27{,}218 messages across ten subjects and 430 subject-days. How that divides between the subjects, and what the derived layers add to it, is in Appendix~\ref{app:census}; Table~\ref{tab:peruser} gives the per-subject totals. Measured with the tokenizer released alongside the corpus, the transcripts total 1{,}643{,}158 tokens, or 60.37 per message.

% ---------- C.2 ----------
\subsection{Released Files and Schemas}
\label{app:schemas}

% P1 - the inventory
Each subject is five JSON files named \texttt{U\#\#\_\{source,profile,persona,chat,qa\}}. Table~\ref{tab:files} gives their contents and corpus totals. Only \texttt{source} is underived; the other four cite it, and every citation is a message identifier of the form defined in Appendix~\ref{app:medium}. Across the corpus, the four derived files place 18{,}983 citations into \texttt{source}, and Appendix~\ref{app:refs} reports that none of them fails to resolve.

\begin{table}[h]
\centering\small
\begin{tabular}{llr}
\toprule
File & Contents & Corpus total \\
\midrule
\texttt{source}  & every message in order, with sender, timestamp and media marker & 27{,}218 messages \\
\texttt{profile} & evidenced claims about the participant, in eleven top-level fields & 3{,}580 claims \\
\texttt{persona} & an interpretive portrait of the participant, thirty-nine slots & 390 slots \\
\texttt{chat}    & the chat evaluation track, with a separate abstention control & 2{,}034 rows \\
\texttt{qa}      & the question evaluation track & 3{,}312 rows \\
\bottomrule
\end{tabular}
\caption{The five released files.}
\label{tab:files}
\end{table}

% P2 - source
A \texttt{source} file carries a header of five counters and a list of messages. Each message has \texttt{id}, \texttt{sender} in \{\texttt{user}, \texttt{companion}, \texttt{app}\}, an ISO \texttt{timestamp}, \texttt{content}, and a \texttt{media} marker that is \texttt{gif} or \texttt{photo} where media accompanied the turn and absent otherwise. The referenced media are not released. The sender labels are field values and not a claim about the relationship.

% P3 - the claim record
A \texttt{profile} file has eleven top-level fields. Seven of them, \texttt{identity} (13 fixed slots), \texttt{user\_facts} (16 kinds), \texttt{entities} (named people, pets and objects, each with its own claims), \texttt{psychology} (stressors, coping, behaviors), \texttt{speaking\_style} (6 slots), \texttt{personality\_dimensions} (21 slots) and \texttt{background} (4 slots), are built from one repeated object, the claim record, whose fields are in Table~\ref{tab:claim}. Every claim carries the messages that establish it, so a reader can retrieve them from \texttt{source} and decide that the claim is wrong. \texttt{behavior} differs in kind. It holds statistics measured from the messages and needs no evidence list to be recomputed. The remaining three, \texttt{name}, \texttt{narrative} and \texttt{big\_five}, are plain values and are the only unevidenced content in the file. The \texttt{big\_five} of \texttt{persona} is a different object and does carry evidence. Keys prefixed \texttt{l5\_} record repairs made after the final pseudonymization pass.

\begin{table}[h]
\centering\small
\begin{tabular}{llr}
\toprule
Field & Meaning & Present on \\
\midrule
\texttt{text} & the claim & 3{,}580 \\
\texttt{evidence} & message identifiers establishing it, non-empty & 3{,}579 \\
\texttt{evidence\_count} & length of that list & 3{,}580 \\
\texttt{confidence}, \texttt{confidence\_label} & strength, numeric and banded & 3{,}579 \\
\texttt{first\_seen}, \texttt{last\_seen} & observation window & 3{,}109 \\
\texttt{temporality} & stable, event or transient & 3{,}122 \\
\texttt{grounded}, \texttt{kind} & fact or characterization & 114 \\
\texttt{value} & normalized form, where one differs from \texttt{text} & 34 \\
\texttt{l5\_*} & audit markers, eighteen distinct keys & varies \\
\bottomrule
\end{tabular}
\caption{The profile claim record.}
\label{tab:claim}
\end{table}

% P4 - persona
A \texttt{persona} file is flat: thirty-nine content slots beside three identifier fields. Sixteen are demographic slots, eleven are prose portraits of a life domain, two are lists derived from those portraits, and ten are other descriptive fields. Among the ten, \texttt{psychotherapy\_note} holds named working hypotheses, each with its grounds and how firmly it is held, \texttt{inferred\_traits} gives each trait its reasoning and a strength, and \texttt{big\_five} and \texttt{quirks} carry evidence lists of their own. An \texttt{\_evidence} object maps the prose fields to the messages behind them. Unlike \texttt{profile}, \texttt{persona} holds interpretations. Its traits and hypotheses carry a stated strength, its prose fields a confidence through \texttt{\_evidence}, and its demographic values, narrative and derived lists carry no marker of either kind. Of the 390 content slots across the ten subjects, 207 are empty, and the emptiness is concentrated: Appendix~\ref{app:construction} gives the breakdown.

% P5 - the chat row
Table~\ref{tab:chatrow} gives the chat row. The file holds 2{,}034 rows in three strata. The proportional and enriched strata form the 1{,}600 sampled rows, 1{,}533 of them scoreable, and the remaining 434 are cold opens held as an abstention control and excluded from every headline count. Three fields carry what a system is scored against, \texttt{probe}, \texttt{required\_context} and \texttt{gold\_response}, and the rest exist so that those can be contested. \texttt{trace} records how each label was produced stage by stage, \texttt{signals} records what the pre-verification heuristic proposed before any check ran, and \texttt{l5\_category\_recomputed}  records the 72 scoreable rows whose category changed when stage C was re-run. A row with \texttt{withdrawn} set is excluded from every count in this paper, and 13 of the 67 withdrawn rows ship without probe text.

\begin{table}[h]
\centering\small
\begin{tabular}{ll}
\toprule
Field & Meaning \\
\midrule
\texttt{probe}, \texttt{probe\_dayid} & the participant's message, verbatim, and its identifier \\
\texttt{gold\_response} & the reference reply, written under the recorded context \\
\texttt{tier}, \texttt{category} & context shape, recomputable from the lists, and category, assigned by rule \\
\texttt{stratum} & \texttt{proportional}, \texttt{enriched} or \texttt{abstention} \\
\texttt{required\_context} & \texttt{recent\_context}, \texttt{profile\_refs}, \texttt{episode\_refs}, \texttt{recent\_needed} \\
\texttt{trace} & phases \texttt{A} to \texttt{E}, each with its result and rationale \\
\texttt{signals} & the pre-verification hint and its two scores \\
\texttt{referent\_on\_screen} & verdict, its source, distance and criterion \\
\texttt{gold\_depends\_on\_refs} & whether the gold uses the references, on 404 rows \\
\texttt{l5\_category\_recomputed} & prior cohort, current cohort and why \\
\texttt{withdrawn} & the row is not scoreable \\
\bottomrule
\end{tabular}
\caption{The chat row.}
\label{tab:chatrow}
\end{table}

% P6 - the qa row
A \texttt{qa} row carries \texttt{question}, \texttt{answer}, \texttt{category}, \texttt{answerable}, \texttt{evidence}, \texttt{qtype} and \texttt{stratum}, with optional markers recording repairs. Evidence is present on 3{,}249 of the 3{,}312 rows and absent on exactly the 63 marked unanswerable. The file also carries a \texttt{difficulty} field, normalized at assembly to the four tiers with the delivered value preserved as \texttt{difficulty\_delivered}.

% P7 - two things a reader should know before joining the files
Two properties of the release will otherwise be discovered the hard way. First, a \texttt{profile\_refs} entry in a chat row resolves to the transcript through its \texttt{evidence\_dayids}, not through its \texttt{claim\_text}, which is a readable copy and matches a profile claim verbatim on 631 of 749 entries; a join on text will silently lose the rest. Second, the profile ships 455 claims whose audit verdict is \texttt{unsupported}, retained so that the audit can itself be checked. The chat track cites none of them. Of the 631 references that resolve to a claim, 486 resolve to a claim the audit supported, 107 to one it judged repairable, 16 to one needing review, 22 to one it did not judge, and none to one it marked unsupported.

\begin{figure}[p]
\centering
\begin{minipage}{\linewidth}
\scriptsize
\linespread{0.9}\selectfont
\begin{lstlisting}[style=rcitem]
{
    "user_id": "U10", "conversation_id": "U10_chat",
  "chat_gt": [
    {
      "probe": "Honestly, julian is my exact standard, haha.",
      "probe_dayid": "day10-21", "tier": "intermediate", "category": "biographical",
      "required_context": {
        "recent_context": ["day10-20"],
        "profile_refs": [
          { "facet": "relationship_status", "claim_text": "Crush on Julian",
            "evidence_dayids": ["day1-34"], "weight": "primary" },
          { "facet": "entities[Julian]", "claim_text": "Julian hugs and smile",
            "evidence_dayids": ["day3-665"], "weight": "supporting" }
        ],
        "episode_refs": [], "recent_needed": 1
      },
      "gold_response": "That makes total sense! With his amazing smile and those great
        hugs, it is pretty easy to see why he sets the bar for you. Haha.",
      "trace": {
        "A": { "referent_text": "julian", "locus": "profile",
          "rationaleA": "Julian refers to a person in the user's life / romantic ideal,
            which is a stable fact about the user not mentioned in the recent
            conversation turns." },
        "B": { "locus": "profile", "stage_b_reverified": true,
          "rationaleB": "profile: 8 matched claim(s); 2 verified against their evidence
            turns in source, 2 dropped, 4 with no evidence before the probe; kept 2" },
        "C": { "tier": "intermediate", "category": "biographical",
          "ruleC": "relies on a stable profile fact (facet set later)" },
        "D": { "grounding_dayids": ["day10-20"],
          "rationaleD": "gold grounded only in recent + verified refs" },
        "E": { "ok": true, "category_ok": true,
          "rationaleE": "The gold response accurately references Julian's smile and hugs,
            which are directly supported by the cited profile reference from day3-665.",
        "note": "phase E with the recent thread and the references' turn text visible
            (PHASE_E_REC)" }
      },
      "signals": { "hint": "intermediate", "epi_score": 0.011, "prof_score": 0.167 },
      "stratum": "enriched", "gold_depends_on_refs": true, "l5_stage_b_reverified": true,
            "referent_on_screen": { "verdict": "no", "source": "S2", "distance": null,
        "criterion": "the referent's name occurs in the ten turns before the probe but
          the fact the probe presupposes lives only in the distant reference; hand
          reading of 2026-09-17" },
     "l5_row_built": { "phases": "A(llm) B(rule+llm, turn-level) C(rule) D(llm) E(llm)",
        "from": "the L5 source and L5 profile of this subject",
        "why": "exhaustive sweep of every unused user turn (2026-09-14); the turn is
          the probe, verbatim" }
    },
    {
      "probe": "Going with a bare-faced aesthetic! As I mentioned yesterday, I really
        dislike cosmetics, so I never apply any and do not even possess cosmetics for
        that reason.",
      "probe_dayid": "day8-151", "tier": "hard", "category": "event_recall",
      "required_context": {
        "recent_context": ["day8-150"],
        "profile_refs": [],
        "episode_refs": [
          { "dayids": ["day7-50"],
            "gist": "I prefer to leave my face bare. I really dislike cosmetics, so I
              never apply any and do not even possess products.",
            "weight": "primary" }
        ],
        "recent_needed": 1
      },
      "gold_response": "A bare face suits you wonderfully! I remember you mentioning
        yesterday that you dislike cosmetics and don't own any products at all, so
        going natural is perfectly true to your style.",
      "trace": {
        "A": { "referent_text": "I really dislike cosmetics", "locus": "episode",
          "rationaleA": "The user explicitly references a statement made 'yesterday'
            about disliking cosmetics, which refers to an event/conversation from an
            earlier day not shown in the recent thread." },
        "B": { "locus": "episode", "stage_b_reverified": true,
          "rationaleB": "retrieved 5 distant candidate(s); 1 verified against source by
            re-reading each turn" },
        "C": { "tier": "hard", "category": "event_recall",
          "ruleC": "distant USER episode, verified -> hard (event_recall)" },
        "D": { "grounding_dayids": ["day7-50"],
          "rationaleD": "gold grounded only in recent + verified refs" },
        "E": { "ok": true, "category_ok": true,
          "rationaleE": "The GOLD correctly recalls the cited statement from yesterday
            (day7-50) where the user mentioned disliking cosmetics and not owning any
            products. The episode reference and category are accurate.",
        "note": "phase E with the recent thread and the references' turn text visible
            (PHASE_E_REC)" }
      },
      "signals": { "hint": "basic", "epi_score": 0.017, "prof_score": 0.062 },
        "stratum": "enriched", "gold_depends_on_refs": true,
              "referent_on_screen": { "verdict": "no", "source": "S3B", "distance": null},
      "l5_stage_b_reverified": true,
      "l5_row_built": { "phases": "A(llm) B(rule+llm, turn-level) C(rule) D(llm) E(llm)",
        "from": "the L5 source and L5 profile of this subject",
        "why": "exhaustive sweep of every unused user turn (2026-09-14); the turn is
          the probe, verbatim" }
    }]}
\end{lstlisting}
\end{minipage}
\caption{Two released items: one \texttt{intermediate}, one \texttt{hard}.}
\label{fig:item}
\end{figure}

% ---------- C.3 ----------
\subsection{Artifacts of the Medium and Removals}
\label{app:medium}

% P1 - the identifier
Every message carries a stored identifier of the form \texttt{day}$n$\texttt{-}$j$, where $n$ is the rank of the message's date among the dates on which that subject was active, so $n$ runs from 1 to the subject's active-day count, and $j$ is the position of the message within its day in the stream as it stood before the removals described below. The stream is ordered by timestamp with ties broken by identifier; exactly one tie occurs corpus-wide. Identifiers are stored, not computed. Every reference in the release resolves by lookup on the identifier, and reconstructing $j$ by counting released messages is incorrect on 355 of the 1{,}600 sampled chat rows.

% P2 - artifacts of the medium
Four features of the transcript are properties of the application, and reading them as speech would be an error. A companion turn is delivered as a sequence of chat bubbles separated by \texttt{///}, and 3{,}009 turns carry such a separator and no other marker. A companion turn may also offer tappable suggested replies, appended after \texttt{////}; 8{,}121 of the 14{,}329 companion turns do so, and the text after that marker was never uttered by anyone. Thirty-three of the 12{,}795 user turns reproduce a suggested reply character for character, 54 after case and punctuation are normalized, against 32{,}688 chips offered in 8{,}121 companion messages. The affordance the product invests in is almost entirely unused. Turns with sender \texttt{app} are instructions the client posted into the feed to steer the companion's next reply, typically on receipt of a gift in the snack economy; they are visible in the feed, but they are not the subject's words. Media markers appear on 902 companion turns and 426 user turns.

% ---------- C.4 ----------
\subsection{Timestamps and Observation Windows}
\label{app:time}

% P1 - the offset
Timestamps were shifted by a per-subject offset, uniform across every date and timestamp belonging to that subject, deterministic in the subject identifier, with magnitude between 60 and 359 days and a sign that varies by subject. Time of day is carried through unchanged. Every interval is therefore exact, so active spans, gaps between sessions, session rhythm and the hour at which a subject writes are all preserved and can be measured from the release. Day of the week survives only for a participant whose offset happens to be a multiple of seven, and a reader cannot tell which participants those are, so the release supports no claim about weekday or weekend behavior.

% P2 - observation windows
An observation window runs from a subject's first message to their last authored turn. It closes on a turn the subject wrote, because the companion may send further messages after someone has stopped replying, and counting those would credit the relationship with days in which the subject was absent. Both bounds are recoverable from the released timestamps, so the window is re-computable. Its length is the number of calendar days it touches, counting both endpoints. Windows span 36 to 120 days, of which 9 to 111 carry activity. Because the offset above is uniform within a subject, the length of a window is exact even though its position in the calendar is not.

% ---------- C.5 ----------
\subsection{Per-Subject Census}
\label{app:census}

% P1 - what each column of the body table counts
Table~\ref{tab:peruser} is re-computable from the release, and this is what each column counts. \emph{Span} is the length of the observation window defined in Appendix~\ref{app:time}, and \emph{Active days} is the number of calendar days on which any message was sent. \emph{Messages} counts every message regardless of sender, 12{,}795 of them the participants' own. \emph{Chat items} counts scoreablerows in the proportional and enriched strata, and \emph{Question items} counts question rows not withdrawn, of which 77 are.

% P2 - the subjects are not interchangeable
The ten subjects differ by two orders of magnitude. The smallest contributes 115 turns over 9 active days and the largest 12{,}627 over 95, and density varies independently of volume: one subject is active on 111 of 111 consecutive days, another on 9 days spread across 36. The derived layers scale with the transcript, from 22 profile claims to 1{,}363, but one of them conspicuously does not. Persona emptiness is almost flat across the corpus: the subject with 115 turns leaves 26 of 39 slots empty and the subject with 12{,}627 leaves 19. A hundredfold increase in what someone said fills seven further slots. The sparsity reported in Section~\ref{sec:dataset} is therefore not an artifact of thin data, and Appendix~\ref{app:construction} gives the breakdown by slot type.

% P3 - what the concentration does to aggregation
The concentration decides how results are aggregated. Two subjects supply 73\% of the turns, so a figure pooled over items estimates how those two behave, which is why every rate in this paper is reported pooled over items with its macro beside it, as Appendix~\ref{app:scoring} specifies. One consequence must be stated. Under the strict reading of Appendix~\ref{app:onscreen} five subjects contribute no memory-bearing probes at all, so quantities conditioned on memory demand are macro-averaged over the five that do and not over ten, and on the proportional stratum those five contribute between one and seven probes each. Table~\ref{tab:census} gives the counts. Intervals on those quantities are correspondingly wide and are reported as they come out.

\begin{table}[h]
\centering\small
\begin{tabular}{lrrrrrr}
\toprule
& & & \multicolumn{4}{c}{Memory-bearing probes} \\
\cmidrule(lr){4-7}
& & & \multicolumn{2}{c}{All scoreable} & \multicolumn{2}{c}{Proportional stratum} \\
\cmidrule(lr){4-5}\cmidrule(lr){6-7}
Participant & Profile claims & Persona slots empty & Recorded & Strict & Recorded & Strict \\
\midrule
U01 &    39 & 26 &   0 &   0 &  0 & 0 \\
U02 &    22 & 23 &   1 &   0 &  0 & 0 \\
U03 &    36 & 26 &   0 &   0 &  0 & 0 \\
U04 &    40 & 24 &   0 &   0 &  0 & 0 \\
U05 &    56 & 22 &   4 &   0 &  2 & 0 \\
U06 &   140 & 18 &   4 &   2 &  3 & 1 \\
U07 &   421 & 19 &  20 &  14 &  4 & 2 \\
U08 &   328 & 16 &  39 &  22 &  3 & 2 \\
U09 & 1{,}135 & 23 & 123 &  88 & 11 & 7 \\
U10 & 1{,}363 & 19 & 213 & 124 & 17 & 3 \\
\midrule
Total & 3{,}580 & 216 & 404 & 250 & 40 & 15 \\
\bottomrule
\end{tabular}
\caption{Derived layers and memory-bearing counts per subject.}
\label{tab:census}
\end{table}

% ---------- C.6 ----------
\subsection{Profile and Persona Construction}
\label{app:construction}

% P1 - the funnel
Extraction is adversarial toward its own output, and most of what it proposes it throws away. Of 7{,}176 candidate claims, 2{,}654 were dropped for cause: 1{,}508 cited no evidence at all, 693 were ungrounded in the evidence they did cite, and 453 were unreliable given it, a rejection rate of 37.0\%. A further 1{,}212 were merged into existing claims as duplicates, and a later enrichment pass over the same conversations restored material the first pass had missed, leaving the 3{,}580 claims that ship, 3{,}579 of them carrying evidence. Every drop reason is released. The rejected material is informative, because attributes that were never stated and aspirations recorded as facts are exactly the errors that would otherwise become gold answers.

% P2 - the audit, and why failures are retained
A later audit re-examined the surviving claims against \texttt{source}, judged 3{,}378 of the 3{,}580, and recorded each verdict on the claim itself; Table~\ref{tab:verdicts} gives the distribution. Claims the audit could not support are retained, so that the audit is itself auditable: a reader who disagrees with a verdict has the claim, its evidence and the verdict in front of them. Eighteen further markers record what was done to individual claims, among them 538 rejudged after an error in \texttt{source} was corrected, 500 re-filed against a normalized entity name, 290 with a speaker attribution fixed, 135 with journal evidence pruned and 69 with a date re-anchored. Appendix~\ref{app:schemas} reports that no chat item cites a claim the audit marked unsupported.

\begin{table}[h]
\centering\small
\begin{tabular}{lr}
\toprule
Audit verdict & Claims \\
\midrule
\texttt{supported} & 1{,}471 \\
\texttt{repairable} & 1{,}025 \\
\texttt{needs\_review} & 390 \\
\texttt{unsupported} & 455 \\
\texttt{supported\_after\_pruning} & 37 \\
not judged & 202 \\
\midrule
Total & 3{,}580 \\
\bottomrule
\end{tabular}
\caption{Audit verdicts on profile claims.}
\label{tab:verdicts}
\end{table}

% P3 - the persona and the shape of its emptiness
The persona is assembled by a different procedure and to a different standard, since its entries are interpretations of the participant. Its emptiness is not uniform, and the gradient is the finding. Of the 160 demographic slots across the ten participants, 147 are empty. The ten of them that Nemotron-Personas-USA also defines, and fills for each of its
million personas, are empty on 96 of 100. Of the 110 domain portraits, which ask what the participant is like in one area of life, 45 are empty, and so are 3 of the 20 lists derived from them. Of the 100 remaining descriptive fields, 12 are empty. A slot that can be read from anything a person says, such as how they communicate or what their personality is, is almost always filled. A slot that needs them to have raised one particular topic is filled less often, and a slot that needs them to state a category outright is almost never filled. People tell a companion about their family, their health and what is worrying them, and almost never state their census category. Volume barely moves this, since the participant with 115 messages leaves 26 of 39 slots empty and the participant with 12{,}627 leaves 17.

% ---------- C.7 ----------
\subsection{Reference Resolution and the Verbatim Invariant}
\label{app:refs}

% P1 - every reference resolves
Across the released files the derived layers place 18{,}983 references into \texttt{source}: 5{,}113 from \texttt{chat}, 4{,}964 from \texttt{qa}, 6{,}784 from \texttt{profile} and 2{,}122 from \texttt{persona}. Every one of them resolves to a message in the released \texttt{source}. Counting withdrawn chat rows as well brings the total to 19{,}042, and the phase D grounding lists add a further 1{,}740 identifiers over scoreable rows, all of which also resolve. No reference anywhere in the release points at a message that is not there.

% P2 - the verbatim invariant
The claim that the annotation can be checked against the transcript is a claim about one file pair, and it holds exactly. Each of the 1{,}533 scoreable chat rows names its probe by identifier, and all 1{,}533 resolve to a message in \texttt{source} whose text is identical to the probe recorded in the ground truth, character for character. The same holds for all 434 rows of the abstention control. 13 identifiers in the chat file do not resolve, and all thirteen belong to withdrawn rows whose text was removed with them. No scoreable row depends on a message that is absent, and no scoreable row quotes a message inexactly.

% P3 - what resolution does not establish
Resolution is a weaker property than correctness and should not be read as more than it is. It establishes that every cited turn exists and that every probe is quoted exactly, which is what makes a disagreement with the annotation demonstrable. It does not establish that a cited turn supports the claim that cites it. That is what the audits assess: Appendix~\ref{app:construction} for the profile, Appendix~\ref{app:audit} for the chat track and Appendix~\ref{app:qagt} for the question track.
\section{Ground Truth Derivation}
\label{app:gt}

Section~\ref{sec:chatgt} describes what a chat item records and why its labels can be trusted. This appendix gives the full derivation and what it rejected. It covers the five stages (\ref{app:phases}), probe selection (\ref{app:probes}), the three labels and their recomputability (\ref{app:labels}), the judgment of whether a referent was already on screen (\ref{app:onscreen}), the rows withdrawn, demoted or repaired (\ref{app:withdraw}), the independent audit (\ref{app:audit}), the release gates (\ref{app:gates}), and the question track (\ref{app:qagt}).

% ---------- D.1 ----------
\subsection{The Five Phases}
\label{app:phases}

% P1 - the shape of the chain
Derivation alternates between proposing and checking. Stages A and D put something forward, stages B and E check it against \texttt{source}, and stage C applies a fixed rule to what survives. The arrangement is what gives the error a direction, because no phase downstream of a check can restore what the check removed, so a failure of the pipeline makes an item look less demanding than it is. Table~\ref{tab:phases} states what each phase writes and what it is permitted to revise. Every phase records its own reasoning in the item, so a reader can locate a disagreement at the phase that caused it.

\begin{table}[h]
\centering\small
\begin{tabular}{llp{5.6cm}}
\toprule
Phase & Writes & May revise \\
\midrule
A & \texttt{referent\_text}, \texttt{locus} & nothing; it is the first statement \\
B & \texttt{rationaleB} & the reference set, by removing entries whose evidence does not verify; it may never add one \\
C & \texttt{tier}, \texttt{category}, \texttt{ruleC} & nothing; both are functions of the context shape B left behind \\
D & \texttt{gold\_response}, \texttt{grounding\_dayids} & nothing; it conditions on the context C left behind \\
E & \texttt{ok}, \texttt{category\_ok} & whether the item ships: a failure is returned for repair or withdrawn \\
\bottomrule
\end{tabular}
\caption{What each phase writes.}
\label{tab:phases}
\end{table}

% P2 - A and B
Phase A resolves what the probe is pointing at. It records the referring expression verbatim in \texttt{referent\_text} and names the store the referent would live in, and it is permitted to be wrong, because nothing has been checked yet. Phase B is the check. For each reference A proposed it retrieves the candidate records and tests them against \texttt{source}, and where the evidence does not verify it removes the reference. A reference that survives B has been matched to turns that exist and precede the probe. A tier therefore falls only when stage B has removed evidence (Appendix~\ref{app:withdraw}).

% P3 - C
Phase C assigns the tier and the category by rule, from the context shape that survived B. Nineteen rules cover the corpus, and the item records which one fired in \texttt{ruleC}. Because C reads only the shape and never the prose, the tier is a function of the three context lists and not a judgment, which is what makes it re-computable by a reader and what produces the exact correspondence reported in Appendix~\ref{app:labels}. It is also why B cannot demote an item directly: B removes evidence, and C then reads a different shape and assigns a different tier.

% P4 - D
Phase D writes the reference reply and records which turns it rests on. The test that matters is containment: the grounding must lie inside the recent window, the surviving references, and the probe itself, since a gold response is entitled to use the turn it is answering and the turns already on screen. Grounding is recorded on 1{,}342 of the 1{,}533 scoreable rows, and on none of them does it name a turn outside that set, which is the verification of the sufficiency condition of Appendix~\ref{app:validity}. On 297 rows the grounding is the probe alone, which is the ordinary case of a reply that answers what was just said.

% P5 - E, and how to read it
Phase E decides whether the item ships. It re-reads the probe, the gold, and the assigned category together and either passes the item or sends it back, and items that cannot be repaired are withdrawn. Every shipped row therefore carries \texttt{ok} and \texttt{category\_ok} set true, on all 1{,}533, and this should not be read as evidence that the annotation is correct. It is a property of what ships: the field records that an item survived its own pipeline, not that an independent reader agrees with it. What speaks to correctness is the audit of Appendix~\ref{app:audit}, which was run afterwards and against different questions.

% ---------- D.2 ----------
\subsection{Probe Selection and the Three Strata}
\label{app:probes}

% P1 - the sampling frame
Probes are user turns, and the sampling frame is every user turn in the corpus. Turns were drawn from buckets by a routing heuristic scored on the turn and its neighborhood, and the bucket each row came from is recorded in \texttt{signals.hint}. That field is deliberately never recomputed. It states which bucket the sampler drew from, before any phase had checked anything, so the released item preserves what the sampler believed alongside what the derivation found. A reader can therefore audit the sample as well as the labels.

% P2 - the three strata
The chat file carries three strata, and they answer different questions. The \emph{proportional} stratum is 1{,}227 rows, 1{,}169 of them scoreable, drawn between 32 and 326 per participant without conditioning on whether a message needs memory. It alone has a denominator, so every rate that describes the traffic is computed on it and nowhere else. The \emph{enriched} stratum is 373 rows, 364 scoreable, found by running stage A over all 11{,}581 unused user messages and keeping what stage B verified. All 364 are memory-bearing by construction, so the stratum supplies power and no rate, since comparisons conditioned on memory demand would otherwise rest on the 40 such rows the proportional stratum contains. The \emph{abstention} control is 434 session openers with empty context lists, scored apart and entering no rate. The census also corroborates the proportional estimate: stage B verified a distant referent on 3.7\% of the 11{,}369 messages stage A classified, against 3.4\% in the proportional stratum.

% P3 - the sampler was wrong on nearly half the rows
Because the bucket is preserved, the sample yields a measurement nobody designed it to yield. Table~\ref{tab:buckets} crosses the bucket a row was drawn from against the tier phase C assigned after verification. They agree on 848 of the 1{,}532 rows carrying a comparable bucket, so the heuristic was wrong 44.6\% of the time; one further row carries a non-tier bucket value and is excluded. The failures are not symmetric. Of the 97 turns the sampler was most confident carried a distant memory demand, 10 survived verification as \texttt{hard}, and 72 turned out to need nothing beyond the current thread. Of the 386 it routed as \texttt{intermediate}, 88 were. Of the 173 it routed as needing nothing, 63 turned out to need a specific earlier exchange. This is an independent estimate of the same quantity the retrieval analysis reports: whether a turn needs memory is not predictable from the turn and what surrounds it. It arrived as a byproduct of building the sample.

\begin{table}[h]
\centering\small
\begin{tabular}{lrrrr}
\toprule
& \multicolumn{4}{c}{Tier assigned after verification} \\
\cmidrule(lr){2-5}
Sampled as & basic & intermediate & hard & edge \\
\midrule
basic (876)        & 696 &  99 & 79 &  2 \\
intermediate (386) & 269 &  88 & 29 &  0 \\
hard (97)          &  72 &  15 & 10 &  0 \\
edge (173)         &  36 &  20 & 63 & 54 \\
\bottomrule
\end{tabular}
\caption{Sampling bucket against verified tier.}
\label{tab:buckets}
\end{table}

% P4 - why the strata are never pooled
The two strata are never pooled for a rate, and the reason is visible in the paragraph above. A pooled denominator mixes turns drawn without regard to memory demand with turns selected for it, so the resulting figure estimates neither the traffic nor the swept population. Every rate in this paper names its stratum, and a rate quoted without one is not defined. Comparisons conditioned on memory demand are reported on both strata separately, because the proportional estimate is the honest one and the enriched estimate is the one with enough rows to have an interval worth reading. Comparisons are reported on the proportional stratum and on the full scoreable set. The first has no selection, and the second has enough memory-bearing rows for an interval worth reading.

% ---------- D.3 ----------
\subsection{Tier, Category and Locus}
\label{app:labels}

% P1 - three families, two of which are functions
An item carries three labels, and they are not the same kind of object. The tier and the category are functions of the context the item was left with, computed by rule in phase C and therefore re-computable by anyone holding the file. The locus is a statement about where a referent lives, made by phase A and tested by phase B, and it is the only one of the three that rests on a judgment.

% P2 - the tier is the context shape
The tier is the shape of the context lists, and the correspondence is exact. Of the 1{,}533 scoreable rows, 1{,}073 carry a recent window and no references and are \texttt{basic}; 223 carry at least one profile reference and are \texttt{intermediate}; 181 carry at least one episode reference and are \texttt{hard}; and 56 carry nothing at all and are \texttt{edge}. No row populates both reference lists, and no row of one tier has the shape of another. The consequence is that a reader never has to accept a tier: recomputing it from the three lists either reproduces the shipped label on every row, which it does, or identifies the rows where it does not.

% P3 - the category
The category subdivides the tier and is assigned by the same stage C rule from the same inputs. Every chat category but one lies inside a single tier, and \texttt{change\_over\_time} spans two (97 intermediate, 8 hard), so Table~\ref{tab:cohorts} refines the tier partition with one exception. Categories are released at full granularity even where a category holds few items, because a downstream user can pool labels and cannot recover a distinction that was merged away. A category is reported in this paper only under the rule the manifest records and gate O enforces: macro-averaged where it holds at least five scoreable items in at least five participants, pooled otherwise, and released without a claim below twenty items.

\begin{table}[h]
\centering\small
\begin{tabular}{llr}
\toprule
Tier & Category & Items \\
\midrule
basic        & \texttt{continuation}       & 944 \\
             & \texttt{clarification}      &  95 \\
             & \texttt{social}             &  34 \\
\midrule
intermediate & \texttt{change\_over\_time} &  97 \\
             & \texttt{biographical}       &  68 \\
             & \texttt{habitual}           &  30 \\
             & \texttt{preferential}       &  16 \\
             & \texttt{relational}         &  12 \\
\midrule
hard         & \texttt{event\_recall}      &  93 \\
             & \texttt{implicit\_callback} &  71 \\
             & \texttt{emotional\_thread}  &   9 \\
             & \texttt{change\_over\_time} &   8 \\
\midrule
edge         & \texttt{cold\_open}         &  56 \\
\midrule
Total        &                             & 1{,}533 \\
\bottomrule
\end{tabular}
\caption{Chat categories by tier, scoreable items.}
\label{tab:cohorts}
\end{table}
% P4 - the locus
The locus names which store holds the record a probe needs, and it takes five values. Phase A states it and phase B tests it, and both are recorded, so an item carries what was claimed alongside what survived verification. Table~\ref{tab:locus} gives the verified distribution by stratum. The locus and the tier cross in only six of their twenty possible cells. \texttt{thread} is the 1{,}009 \texttt{basic} items that need the current exchange, \texttt{none} splits into 64 \texttt{basic} items and the 56 \texttt{edge} items, \texttt{profile} is exactly the intermediate tier, and \texttt{episode} and \texttt{companion} are exactly the hard tier, separated by whether the record is something the participant said or something the companion did. Appendix~\ref{app:withdraw} gives the rows on which the claimed and verified values differ.

\begin{table}[h]
\centering
\small
\setlength{\tabcolsep}{6pt}
\begin{tabular}{@{}lrrr@{}}
\toprule
Where the referent lives & \multicolumn{2}{c}{Proportional stratum} & Scoreable set \\
\cmidrule(lr){2-3}\cmidrule(lr){4-4}
 & Items & Share & Items \\
\midrule
The current thread                    & 1009 & 86.3\% & 1009 \\
A claim in the profile                &   13 &  1.1\% &  223 \\
A specific earlier exchange           &   10 &  0.9\% &  116 \\
The companion's own past behavior     &   17 &  1.5\% &   65 \\
Nothing prior                         &  120 & 10.3\% &  120 \\
\midrule
Total                                 & 1169 &        & 1533 \\
\bottomrule
\end{tabular}
\caption{Verified locus of every scoreable chat item.}
\label{tab:locus}
\end{table}

% ---------- D.4 ----------
\subsection{Whether the Referent Was on Screen}
\label{app:onscreen}

% P1 - why the field exists
Condition V2 of Appendix~\ref{app:validity} requires that a reference name something the system must go and find, not something already in front of it. The derivation cannot guarantee this. Phase A resolves what a probe points at and phase B verifies that the turns it points at exist, but neither asks whether those turns are still on screen when the probe arrives, and a referent established twenty turns ago and restated three turns ago is not a memory demand however far back its first mention lies. Every memory-bearing row carries a judgment of it, and every rate in this paper is reported under more than one way of using that judgment.

% P2 - the criterion
All 404 memory-bearing rows were examined against one rule applied to the ten turns preceding the probe. A row is \texttt{clear} when those ten turns already state what the probe presupposes, so a system reading only its recent window would have what it needs. It is \texttt{borderline} when the referent appears there under a name, alias or description but not in a form that states the fact the probe rests on. It is \texttt{no} when the referent is absent from the window whatever else it contains. The rule is recorded per row in a \texttt{criterion} field, so a reader can see which of the six phrasings was applied, and a \texttt{distance} field records how far back the nearest on-screen mention lies on the 199 rows where one exists.

% P3 - how each verdict was reached
Table~\ref{tab:onscreen} gives the verdicts by how each was reached. Source \texttt{S1} is mechanical: a cited reference turn falls inside the ten-turn window, so the row is \texttt{clear} by arithmetic and no reading was required. The remaining three sources are readings of the text by one person, in three passes; \texttt{S3B} is the pass that closed the last 229 rows, which until then had carried no verdict and had been treated as \texttt{no} by default. That default is gone: every one of the 404 rows now carries a source. When the rule was unified for the final pass, it was reapplied to the first fifty rows of the earliest pass, and those verdicts held, so the two halves of the population were judged under the same rule and not merely under the same name.

\begin{table}[h]
\centering\small
\begin{tabular}{lrrrr}
\toprule
How the verdict was reached & Clear & Borderline & Not on screen & Rows \\
\midrule
Cited reference inside the window & 11 & 0 & 0 & 11 \\
First reading pass & 68 & 5 & 12 & 85 \\
Second reading pass & 57 & 10 & 12 & 79 \\
Final reading pass & 18 & 30 & 181 & 229 \\
\midrule
Total & 154 & 45 & 205 & 404 \\
\bottomrule
\end{tabular}
\caption{On-screen verdicts by provenance.}
\label{tab:onscreen}
\end{table}

% P4 - the three readings
The field supports three populations and the paper reports against all of them. The \emph{recorded} reading takes every row with a non-empty reference set, which is what the derivation produced and what a reader reproduces from the context lists alone. The \emph{strict} reading removes the rows judged \texttt{clear}, leaving the probes whose recent window does not already state what they presuppose. The third removes the borderline rows as well, which is the most conservative population the field supports. Table~\ref{tab:readings} gives the counts. The proportional stratum contains no borderline rows at all, so its second and third readings coincide, and the demand rate it supports falls from 3.4\% to 1.3\% [0.8, 2.1] once the clear rows are removed.

\begin{table}[h]
\centering\small
\begin{tabular}{lrrr}
\toprule
Stratum & Recorded & Clear removed & Clear and borderline removed \\
\midrule
Proportional (1{,}169) & 40 & 15 & 15 \\
Enriched (364) & 364 & 235 & 190 \\
\midrule
Pooled (1{,}533) & 404 & 250 & 205 \\
\bottomrule
\end{tabular}
\caption{Memory-bearing rows under each reading.}
\label{tab:readings}
\end{table}

% P5 - a convergence worth noting
The strict proportional count arrives at the same place as a decision taken by a different route. Five demotions proposed during the audit and declined would have removed between fourteen and nineteen rows from the proportional memory-bearing population; removing the rows this field marks \texttt{clear} removes fifteen. Of the twenty-five proportional rows so marked, eight are cases where the companion had just demonstrated the behaviour the probe refers to, which is the pattern those demotions were about. Two procedures that did not consult each other, one editing labels and one reading windows, identify nearly the same rows, and the field reaches the result without changing a single label.

% P6 - what the field does not establish
Three limits. The readings were made by one person, so the field records a judgment and not an agreement rate, and a second reader would be the obvious improvement. The rows were examined because they carry references, so the field says nothing about probes the derivation never marked as memory-bearing. And the error has a direction worth stating, which runs opposite to the derivation's. Reading a row as \texttt{no} when it should be \texttt{clear} keeps it in the memory-bearing population, so mistakes of that kind make memory demand look commoner than it is. Verification errs the other way, since it can only remove a reference (Proposition~\ref{prop:onesided}). The two bound the rate from opposite sides and neither is quantified, so the strict reading is reported beside the recorded one.

% ---------- D.5 ----------
\subsection{Withdrawal, Demotion and Repair}
\label{app:withdraw}

% P1 - three outcomes, all recorded
An item can fail in three ways, and the release treats them differently. It can be unusable, in which case it is withdrawn but still shipped as a row with its reason. It can survive with a weaker label than it was proposed with, in which case the demotion is visible as a disagreement between two recorded phases. Or a defect can be found after the labels were fixed, in which case it is marked. Nothing is silently deleted, because a corpus that hides its failures cannot be audited on them.

% P2 - withdrawal
Sixty-seven of the 1{,}600 chat rows are withdrawn and excluded from every count in this paper. Each carries a code naming what went wrong, and Table~\ref{tab:withdrawn} groups them by what failed. Withdrawals for a probe that could not be released are a cost of full anonymization, and so are the 17 \texttt{context\_lost} rows, every one of whose cited messages the full anonymization removal manifest deleted. Withdrawals for an invalid reference include nine rows whose references turned out to be on screen, which were removed before the field of Appendix~\ref{app:onscreen} existed; had that field come first these nine would have been marked \texttt{clear} and retained, and the two mechanisms now do the same
work by different means.

\begin{table}[h]
\centering\small
\begin{tabular}{llr}
\toprule
What failed & Code & Rows \\
\midrule
The context could not be reconstructed & \texttt{context\_lost} & 17 \\
 & \texttt{gold\_outside\_empty\_context} & 3 \\
 & \texttt{stage\_b\_empty\_context} & 2 \\
\midrule
The gold could not be supported & \texttt{gold\_unsupported\_rebuild\_failed} & 12 \\
 & \texttt{stored\_verdict\_failed} & 3 \\
 & \texttt{stage\_b\_gold\_unsupported} & 1 \\
\midrule
A reference was invalid & \texttt{distant\_references\_on\_screen} & 9 \\
 & \texttt{citation\_after\_probe} & 4 \\
 & \texttt{renarrated\_citation} & 3 \\
\midrule
The probe could not be released & \texttt{probe\_turn\_withheld} & 13 \\
\midrule
Total & & 67 \\
\bottomrule
\end{tabular}
\caption{Withdrawn rows by reason.}
\label{tab:withdrawn}
\end{table}

% P3 - demotion
Demotion is not an operation any stage performs. Stage B removes references whose evidence does not verify, stage C then reads a context shape with fewer references in it, and the tier that results is lower than the one the item would have received had the evidence held. No stage after verification can add a reference, so no stage after verification can raise a tier, and Proposition~\ref{prop:onesided} states what this licenses.

% P4 - the locus reset
One consequence of that design was mishandled and has been repaired. Because phase B removes references without rewriting the store name phase A assigned, 51 rows reached the release with an empty reference set and a locus naming a distant store, their own \texttt{rationaleB} recording that no evidence had verified. On every one of them the recent window supplied the gold, so the correct value was the local one. The reset moved 48 rows to \texttt{thread} and 3 to \texttt{none}, drawn from \texttt{profile} (22), \texttt{episode} (21) and \texttt{companion} (8), taking the claimed distribution of 961 / 117 / 245 / 137 / 73 to the verified distribution of Table~\ref{tab:locus}. Both values remain in the file: phase A's under \texttt{trace.A} and the verified one under \texttt{trace.B}. A gate now fails if any row carries an empty reference set together with a locus outside \{\texttt{thread}, \texttt{none}\}.

% P5 - disputed references
A later audit found references that resolve correctly but do not bear on the gold response, and golds addressed to the wrong party. Seventeen rows are affected. They are marked: fifteen carry \texttt{l5\_reference\_disputed} naming the identifiers and the reason, and two carry \texttt{l5\_gold\_addressee\_disputed}. No tier, category, gold or reference was changed. Withdrawing them would have altered three result tables to remove seventeen rows, and correcting them would have replaced a recorded judgment with an unrecorded one. Marking them leaves every number in this paper computed on the corpus as released while telling a reader exactly which rows a second opinion would start from, which is the same choice made for the on-screen judgment and for the unsupported profile claims of Appendix~\ref{app:construction}.

% ---------- D.6 ----------
\subsection{The Independent Audit}
\label{app:audit}

% P1 - why a self-report is not enough
A retriever scored against bad ground truth produces numbers that are corrupted without being obviously wrong, and every figure the derivation reports about itself is a self-report from the process under question. The corpus is therefore graded a second time by a model that built and verified none of it, under an instruction written before any item was read, over all 1{,}533 live rows. Its findings are then read against \texttt{source} by hand, because a flag is a question raised and not a defect established.

% P2 - the instrument, and its own correction
The first instruction was itself defective in two ways: it described the \texttt{cold\_open} gold incorrectly and gave no guidance on companion-authored clusters. Both passes are reported in Table~\ref{tab:audit}. Correcting them removed 305 of the rationale-only majors, 297 of the thread-only ones and 30 on the cold opens, and the two passes agree on the material flag for 1{,}434 of 1{,}530 rows. The auditor is sensitive: it raises a material question on 443 rows and on 290 of the memory-bearing ones, and hand reading confirms a defect on between half and three quarters of a flagged group depending on the tier.

% P3 - the two numbers that establish what the audit is worth
Two measurements decide whether the instrument is worth trusting. The first is whether it misses anything: 119 rows it did not flag were read against \texttt{source} by hand, and none carried a gold or reference defect, which bounds the miss rate at roughly 2.5\%. The second is the corpus-level outcome. Every live row was re-verified by a judge shown the recent thread and the full text of the turns the item cites. It disagrees with the stored reference reply on 7 of 1{,}533 rows, or 0.46\%; two were regenerated and the remaining five were read against \texttt{source} and kept.

% P4 - the prose and the decision fail at different rates
The audit grades the reasoning as well as the labels, and the two fail at very different rates. Read against the ten turns preceding each probe, stage A's rationale fails to justify the recorded locus on 395 of 1{,}533 rows, or 25.8\%, most often by asserting an absence the window contradicts. Hand reading confirms 34 of 40 flagged rows and finds no defect in 50 unflagged ones, so close to a fifth of stage A's written reasons are wrong, against 7 of 1{,}533 for the reference replies those reasons accompany. The rationale is a record of what the builder said and not the warrant for the reference, which stage B establishes against the conversation, so no label was changed on this basis. Eighty of the 395 concern the locus itself, and those enter the field of Appendix~\ref{app:onscreen}. A corpus that ships its reasoning has to ship the rate at which that reasoning is bad.

% P5 - what the audit changed
The audit changed the corpus. 19 golds were regenerated, an unsupported part was removed from a further 129, 108 were re-verified and left unchanged, 30 rows were rebuilt outright, and 67 were withdrawn. Every count in this paper is computed after these repairs, and the gates of Appendix~\ref{app:gates} pass on the released files.

\begin{table}[h]
\centering\small
\begin{tabular}{@{}lrr@{}}
\toprule
& First instruction & Corrected \\
\midrule
Rows audited                               & 1{,}533 & 1{,}530 \\
No finding / minor / major                 & 384 / 197 / 952 & 677 / 220 / 633 \\
Majors concerning rationale wording only   & 396 & 91 \\
Majors on thread-only rows                 & 591 & 294 \\
Majors on cold opens                       & 54 & 24 \\
Rows with a material flag                  & 459 & 443 \\
\midrule
\multicolumn{3}{@{}l}{Agreement between the two passes on the material flag:
 1{,}434 of 1{,}530} \\
\multicolumn{3}{@{}l}{Unflagged rows read by hand: 119, gold or reference defects
 found: 0} \\
\multicolumn{3}{@{}l}{Every live row re-verified against the thread and the cited
 turns: 7 of 1{,}533 fail} \\
\bottomrule
\end{tabular}
\caption{Independent audit of the chat track.}
\label{tab:audit}
\end{table}

% ---------- D.7 ----------
\subsection{Gates and Mechanical Checks}
\label{app:gates}

% P1 - what a gate is, and how it differs from the audit
The audit of Appendix~\ref{app:audit} asks a model whether a label is right. The checks in this subsection ask whether the corpus is internally consistent, and they use no model at all. Each is a predicate over the released files that either holds on every row or names the rows on which it fails, so a check cannot be partly satisfied and cannot be argued with. 26 run as release gates, listed in Table~\ref{tab:gates}; a further five families of check run alongside them and are reported in Table~\ref{tab:families}. The gates are what the release is blocked on, and all 26 pass on the corpus as shipped.

\begin{table}[h]
\centering\small
\begin{tabular}{ll}
\toprule
Gate & What it tests \\
\midrule
A & The probe is a real user turn, reproduced verbatim \\
B & Database identifiers parallel the turn identifiers \\
C & Every cited identifier resolves to a released turn \\
D & Every cited identifier precedes its probe \\
E & A callback reaches outside the recent window \\
F & A gold response is present unless the item is a cold open \\
G & No turn serves as the probe of two rows \\
H & The designed category mix is held \\
I & Each category belongs to its tier \\
J & The phase C record agrees with the row's tier and category \\
K & The phase E verdict passed, or a marker says why not \\
L & An episode gist still matches the turns it cites \\
M & \texttt{recent\_needed} lies within the length of \texttt{recent\_context} \\
N & No vocabulary term is absent from both the corpus and the persona \\
O & Per-category figures follow the reporting rule recorded in the manifest \\
P & Every reference identifier exists in \texttt{source}, persona included \\
Q & The phase B rationale describes what phase B did for the row's locus \\
R & Every question item's difficulty is one of the four tier values \\
S & Every gold citation lies inside the recorded context \\
T & The file's counters agree with its own rows \\
U & The verified locus follows the reference set \\
V & Every row lies in a declared stratum and no control row enters the union \\
W & No database key survives anywhere in a released file \\
X & A boolean marker agrees with the verdict it accompanies \\
Y & No surrogate in a derived file disagrees with the participant's source \\
Z & No retired label is used as a reporting category \\
\bottomrule
\end{tabular}
\caption{Release gates, all passing.}
\label{tab:gates}
\end{table}

% P2 - the gates that record a failure
Two gates exist because a failure shipped and was repaired. Gate T was added after the per-file counters were found to have been written once and never updated, so that one file reported three hundred rows against four hundred and eighty-two. Gate U was added after 51 rows were found carrying an empty reference set beside a locus naming a distant store, the repair of Appendix~\ref{app:withdraw}. Gate V was added with the abstention stratum and fails on any row outside a declared stratum, on any pooling of the sampled strata that does not reconcile, and on any abstention row that is not an empty session opener. Gates B and Z pass vacuously, since the identifiers B checks were removed before release and no label has been retired, and both exist so that a reappearance would be caught.

% P3 - the five families
The five families in Table~\ref{tab:families} run over all 2{,}034 chat rows, the 1{,}600 sampled and the 434 abstention rows together. Their arithmetic closes: the category family accounts for every row as agreeing, relabelled by record, or lacking a session index, and the signals family accounts for every row as matched or as one of thirty-five whose probe text was withheld or is too short to match. All five families pass. The structural family passes on the question track after the repairs of Appendix~\ref{app:qagt}, which rebuilt \texttt{multi\_hop} and \texttt{knowledge\_update} and gave \texttt{long\_range} a span threshold.

\begin{table}[h]
\centering\small
\begin{tabular}{llp{5.4cm}}
\toprule
Family & What it checks & Result \\
\midrule
Referent & The recorded referent is present in the turns the row cites & 1{,}339 of 1{,}341 judged rows grounded; two exceptions recorded and allowed to stand. \\
Signals & The pre-verification scores and the probe text still reproduce & Scores exact on 1{,}999 of 1{,}999 matched rows; probe verbatim on 1{,}999 of 1{,}999 \\
Category rule & The tier, the category and the rule sentence agree & Agree on 2{,}018 rows; 3 relabeled by record; 13 carry no session index \\
Structural & Each question category satisfies its own definition & Passes:  0 of 38 \texttt{multi hop} within one session, 0 of 26 \texttt{knowledge update} citing one message, 0 of 22 \texttt{long range} short of the span threshold\\
Arithmetic & The validity conditions of Appendix~\ref{app:validity} hold row by row & Tier re computable on 1{,}967 of 1{,}967 scoreable rows; no stratum exception; the unanswerable count reconciles \\
\bottomrule
\end{tabular}
\caption{Check families.}
\label{tab:families}
\end{table}

% P4 - what passing establishes
Passing every gate is a weak property and should be read as one. It establishes that the corpus says nothing self-contradictory: no citation dangles, no label disagrees with the rule that produced it, no counter disagrees with the rows it counts. It establishes nothing about whether a label is correct, which is why the audit exists and why the audit's findings are worse than the gates'. The useful thing about a gate is not that it passes but that it would have failed: gates T and U each caught a defect that had already shipped, and each now runs on every build.

% ---------- D.8 ----------
\subsection{The Question Track}
\label{app:qagt}

% P1 - what it is, and why it is audited separately
The question track poses 3{,}312 items, 3{,}235 of them scoreable, over the same ten records. Its probes are authored, so it supports none of the rate claims in this paper; what it supports is the comparison of Section~\ref{sec:results}, where an authored query and a spoken one are put to the same retriever over the same conversations. It was built by a different procedure from the chat track and it is audited separately, by a model that built none of it, under an instruction written before any item was read.

% P2 - what the audit found, stated plainly
The audit found a defect rate materially worse than the chat track's. On the delivered items, 51.3\% [49.6, 53.0] carried a content defect: an unsupported answer, an ill formed question, a wrong answerable flag, or citations none of which is evidence. We report that figure rather than the softer one because it is the one a user of the corpus would meet. The cause was identified and is mechanical. A re-narration of \texttt{source} late in construction changed facts, and items authored before it kept the earlier ones, so a question would ask about a Tuesday that the turn it cites calls a Wednesday.

% P3 - what each repair bought
Repairs were applied in three stages and each was measured separately, so the cost of the defect and the value of each remedy are both visible. Table~\ref{tab:qarepair} gives them. Reclassifying the categories bought nothing on content, by construction, and is reported for completeness. Repairing the wording against the cited turns bought twelve points. Replacing the items that survived both bought ten more.

\begin{table}[h]
\centering\small
\begin{tabular}{lrrrr}
\toprule
& Delivered & After reclassification & After wording repair & After replacement \\
\midrule
Content defect rate & 51.3 & 51.2 & 39.2 & \textbf{28.8} \\
\quad unsupported answer & 31.6 & 32.1 & 26.1 & 18.1 \\
\quad question ill formed & 35.6 & 34.9 & 24.2 & 17.5 \\
\quad answerable flag wrong & 11.5 & 15.7 & 10.7 & 4.9 \\
\quad no cited turn is evidence & 9.7 & 11.2 & 7.7 & 3.1 \\
\bottomrule
\end{tabular}
\caption{What each repair bought, same items, pooled percent.}
\label{tab:qarepair}
\end{table}

% P4 - the repair works, and the evidence is in the residual
The residual is informative. Grouped by what was done to an item, replaced items carry a content defect on 10.2\% [7.4, 14.0], re-anchored items on none of 50, wording-repaired items on 30.3\%, and items nothing was done to on 31.7\% [29.7, 33.7]. The rate is therefore dominated by items that were never treated.

% P5 - the category disagreement is agreement, and should be read as such
An early reading of the audit counted a proposed change of category as a defect, which put 44.1\% of items in fault. That was a mistake of interpretation. Category assignment is a judgment on which two competent annotators may differ, and the right measurement is agreement, not error. Between the classifier and the independent auditor over the final partition, raw agreement is 0.880, and Cohen's $\kappa$ is \textbf{0.866}, which is substantial by any convention. Table~\ref{tab:qakappa} gives it per category and names the pairs the two readers actually confuse. Category disagreement is excluded from the content rate above and reported only here.

\begin{table}[h]
\centering\small
\begin{tabular}{lr@{\qquad}lrr}
\toprule
Category & $\kappa$ & Confused pair & Disagreements & $\kappa$ on the pair \\
\midrule
\texttt{unanswerable}         & 1.000 & \texttt{attribute} / \texttt{basic\_fact}        & 90 & 0.639 \\
\texttt{ambiguous\_ref}       & 1.000 & \texttt{opinion} / \texttt{preference}           & 52 & 0.671 \\
\texttt{temporal}             & 0.989 & \texttt{basic\_fact} / \texttt{person\_reference} & 34 & 0.656 \\
\texttt{long\_range}          & 0.977 & \texttt{basic\_fact} / \texttt{preference}       & 30 & 0.607 \\
\texttt{multi\_hop}           & 0.958 & \texttt{basic\_fact} / \texttt{opinion}          & 20 & 0.523 \\
\texttt{paraphrase}           & 0.928 & \texttt{basic\_fact} / \texttt{routine}          & 17 & 0.403 \\
\texttt{person\_reference}    & 0.917 & & & \\
\texttt{companion\_relational} & 0.883 & & & \\
\texttt{preference}           & 0.872 & & & \\
\texttt{attribute}            & 0.846 & & & \\
\texttt{routine}              & 0.835 & & & \\
\texttt{opinion}              & 0.829 & & & \\
\texttt{distractor}           & 0.823 & & & \\
\texttt{knowledge\_update}    & 0.817 & & & \\
\texttt{basic\_fact}          & 0.795 & & & \\
\texttt{implicit\_entity}     & 0.776 & & & \\
\midrule
Overall                       & 0.866 & & & \\
\bottomrule
\end{tabular}
\caption{Category agreement between the classifier and the auditor.}
\label{tab:qakappa}
\end{table}

% P6 - what was retired, and why that is a finding
Three delivered categories were repaired against their own definitions. \texttt{long\_range} was authored from early turns and encoded no distance property: 38 of its 39 delivered items cite a single turn, and in distance to the asking point its median is shorter than that of the other categories. Every question in this track is asked at the end of the history, so distance from evidence to query is a property of the track, and the delivered label named a recipe. The category now carries a span threshold, which all 22 of its released items meet. \texttt{multi\_hop} and \texttt{knowledge\_update} were rebuilt from the profile's own structure, so that no \texttt{multi\_hop} item stays within one session and no \texttt{knowledge\_update} item cites a single message (Table~\ref{tab:families}).

% P7 - what the repairs did to the numbers, and what remains
One reported figure moves materially with the repair and it should be read as a correction. Question-track \textsc{bm25} recall@5 rises from 0.356 to 0.480 [0.421, 0.536], because the repaired wording restored the cited turns' vocabulary to the questions; the delivered figure was measuring a defect of the delivery, and \textsc{recency}, \textsc{random} and \textsc{oracle} move by at most 0.005 on the same items. The track ships 3{,}312 items of which 3{,}235 are scoreable, with every repair recorded per item, and the per-item audit verdicts are released as an artifact so that the rate in this appendix can be recomputed.
\section{Evaluation Protocol}
\label{app:setup}

Section~\ref{sec:setup} states what is run; this appendix gives the operational definitions, taken from the code that produced the numbers. It covers the context ablation (\ref{app:conditions}), the three retrieval controls (\ref{app:controls}), the retrieval baselines (\ref{app:baselines}), the response judge (\ref{app:judge}), how the quantities of Appendix~\ref{app:scoring} are computed (\ref{app:metrics}), and the models behind each role (\ref{app:models}).

% ---------- E.1 ----------
\subsection{The Context Ablation}
\label{app:conditions}

% P1 - the shared prompt
All five conditions use one generator at temperature zero and one prompt shape: a context block, then the line \texttt{THE USER JUST SAID:} followed by the probe. Messages inside the block are rendered as \texttt{[sender] text} and truncated at 220 characters. The generator is never shown the profile, the persona, the gold response, any label, or which participant it is answering for, so nothing beyond the block distinguishes one condition from another. Table~\ref{tab:conditions} gives what each block contains.

\begin{table}[h]
\centering\small
\begin{tabular}{llp{6.2cm}}
\toprule
Condition & Block heading & Contents \\
\midrule
C0 & none & The probe alone \\
C1 & \textsc{recent conversation} & The 3 messages immediately preceding the probe, either sender \\
C2 & \textsc{relevant memories} & The item's recorded messages: recent context, then episode references, then profile evidence, de-duplicated and restricted to messages preceding the probe. Every other turn is withheld \\
C3 & \textsc{retrieved memories} & The top 10 messages returned by BM25 over the participant's preceding messages, queried with the probe \\
C4 & \textsc{recent conversation} & The 10 messages immediately preceding the probe, either sender \\
C3m & \textsc{earlier messages from this chat} & The same 10 messages as C3, in chronological order \\
C4m & \textsc{earlier messages from this chat} & The same 10 messages as C4 \\
\bottomrule
\end{tabular}
\caption{The five context conditions and the matched pair.}
\label{tab:conditions}
\end{table}

% P2 - which comparisons the design supports
The pairs are chosen so that each isolates one thing. C2, the oracle condition, against C1 varies what is supplied while holding the generator fixed, and Section~\ref{sec:results} shows why that difference is not by itself a memory effect. C2 and C1 also differ in heading, \textsc{relevant memories} against \textsc{recent conversation}, and the matched pair shows that a heading alone moves memory use by 10 to 14 points, so the C2 against C1 difference carries a heading component this design does not separate. A matched pair, C3m and C4m, repeats both under the single heading \textsc{earlier messages from this chat} in chronological order, and Section~\ref{sec:results} reports the matched difference. C0 is the floor and C4 the ceiling reachable without retrieving anything. None of these comparisons exists without a per-item record of which messages the reply required, which is the sense in which the ablation is a consequence of the annotation.

% ---------- E.2 ----------
\subsection{Retrieval Controls}
\label{app:controls}

% P1 - the pool, and what a control does
A method's pool is every turn strictly preceding the probe. A control substitutes a different gold set for the item's own and re-scores every method against it, so that what changes is the target. Each control destroys one property of the gold set while holding the others, which is what makes the drop attributable. Table~\ref{tab:controls} states each. The third is an assertion: it hands every method messages guaranteed to lie outside the gold set, and a method that scores anything above zero on it has a bug.

\begin{table}[h]
\centering\small
\begin{tabular}{lp{4.2cm}p{4.2cm}}
\toprule
Control & Held fixed & Destroyed \\
\midrule
Offset permuted & The pool, and the distribution of distances from probe to gold & The relation between the gold messages and what the probe is about \\
Position randomized & The pool, and the size of the gold set & Both the distances and the content relation \\
Broken oracle & Nothing; 20 pool messages known to lie outside the gold & Everything; the expected score is zero \\
\bottomrule
\end{tabular}
\caption{Retrieval controls.}
\label{tab:controls}
\end{table}

% P2 - how the permutation is built
The offset permutation is constructed per item. A second probe is drawn at random from the same participant; its gold messages are expressed as offsets from its own probe, and the messages at those offsets from the present probe become the substituted gold; offsets reaching past the start of the pool are dropped. The substituted set therefore sits at the same distances as a real gold set drawn from the same participant, and differs only in having no relation to what was said. A method that reads distance alone scores the same as before; a method that reads content does not.

% ---------- E.3 ----------
\subsection{Retrieval Baselines}
\label{app:baselines}

% P1 - the five methods
Each method returns 20 messages from the pool and is scored at ranks 1, 5, and 20 and by reciprocal rank. \textsc{Random} draws 20 messages from the pool under a fixed seed. \textsc{Recency} returns the 20 messages immediately preceding the probe, either sender. \textsc{Recency-user} returns the 20 preceding user messages, which separates adjacency from who was speaking. \textsc{Oracle} returns the item's own gold set, ordered by weight, with recent context and primary references above supporting ones, truncated at 20; it falls short of perfect recall on items whose gold set exceeds the cut. \textsc{BM25} is the only method that reads the probe's content. On the question track, the pool is the participant's whole history, since a question is posed over the record, and evidence messages carry the primary weight.

% P2 - BM25 in full
\textsc{BM25} uses $k_1 = 1.5$, $b = 0.75$ and $\mathrm{idf} = \log\bigl(1 + (N - \mathrm{df} + 0.5)/(\mathrm{df} + 0.5)\bigr)$. The index unit is one turn, both senders, and each participant's history is indexed once. Text is lower-cased and tokenized on \texttt{[a-z0-9']+} with tokens shorter than three characters and a fixed 62-word stop list discarded. The query is the probe text unmodified. The top 40 are retrieved, filtered to positions preceding the probe, and the first 20 kept. All values are macro-averaged over participants.

% ---------- E.4 ----------
\subsection{The Response Judge}
\label{app:judge}

% P1 - the protocol
One judge at temperature zero scores all five candidate replies for a probe in a single call, so the conditions are ranked against one another, and no drift in the judge's use of the scale can order them. It is shown the probe, the full gold response, the text of the messages the item records as required at 200 characters each with an explicit placeholder where there are none, and the five candidates at 400 characters each, labeled by condition. It is not shown the recent thread, the profile, the persona, the reasoning trace, the participant, or what the condition labels mean. It returns three fields: \texttt{content\_match} on a three-point scale against the gold, \texttt{used\_memory} recording whether the reply asserts a specific detail about the past that the probe did not supply, correct or not, and \texttt{restraint} recording whether the reply volunteers no such detail at all.

% P2 - what the judge is worth
The judge's agreement with human adjudication was measured on 100 probes, each carrying five replies, read blind. On \texttt{used\_memory} and \texttt{restraint} it agrees on 90.6\% of replies, $\kappa = 0.804$. On the three-point \texttt{content\_match} it agrees exactly on 77.4\% and is never wrong by two points, $\kappa = 0.545$ and 0.606 under linear weighting; where it differs it is harsher than the reader, lower on 81 replies and higher on 32. The contrasts this paper relies on reproduce. On the same sample the reader and the judge put C2 above C1 by 0.13 and 0.14, C3 below C1 by 0.15 and 0.19, C4 below C1 by 0.05 and 0.09, and C1 above C0 by 0.22 and 0.24. A systematic severity that applies to all five conditions cannot reorder them.

% ---------- E.5 ----------
\subsection{How Each Quantity Is Computed}
\label{app:metrics}

% P1 - the mapping
Appendix~\ref{app:scoring} defines what is scored; this subsection says which field computes it. Table~\ref{tab:impl} gives the mapping. Two of the correspondences are approximations and are stated as such. A system in the ablation never emits a retrieval decision, so its restraint decision is read off its reply: a reply that volunteers a detail about the past is treated as having chosen to retrieve. Likewise the appropriate application rate is read from whether the reply asserts a past detail the probe did not supply, correct or not, because the ablation supplies context and no system returns a retrieved set; whether the detail is right is left to content match.

\begin{table}[h]
\centering\small
\begin{tabular}{llp{5.4cm}}
\toprule
Quantity & Field & Computed as \\
\midrule
Misfire rate & \texttt{restraint} & One minus the restraint rate on probes with an empty reference set \\
Appropriate application & \texttt{used\_memory} & The rate at which a reply on a memory-bearing probe asserts a past detail the probe did not supply, correct or not \\
Content match & \texttt{content\_match} & The three-point verdict against the reference reply: 2 conveys its substance, 1 partly, 0 not at all \\
Retrieval & rank of the gold messages & Hit at 1, 5 and 20 (a required message in the top $k$), recall over the whole gold set at the same cutoffs, and reciprocal rank of the first gold message \\
Detectability & BM25 top-1 score; detector probability & Area under the ROC curve against the recorded label, on the proportional stratum, for the lexical scorer's own confidence and for a model shown the probe and three preceding messages \\
\bottomrule
\end{tabular}
\caption{How each quantity is computed.}
\label{tab:impl}
\end{table}

% P2 - why evidence coverage is reported
Evidence coverage exists because every other response-side number depends on the judge of Appendix~\ref{app:judge}, whose agreement on the graded scale is the weaker of its two measurements. It uses no model, shares none of the judge's failure modes, and moves in the same direction on the same comparisons, which is the closest thing to a validation the response side currently has. It is a weak measure taken alone, since a reply can repeat the words of an evidence turn without using it, and is reported only alongside content match.

% ---------- E.6 ----------
\subsection{Models and Decoding}
\label{app:models}

Every call in the chat and question tracks is made at temperature zero, so a run is reproducible up to the provider's own nondeterminism; the reconstruction agents run under their own settings (Appendix~\ref{app:recon}). The generator, the judge and the auditor are three different models, and the auditor built none of the ground truth it grades.

\begin{table}[h]
\centering\small
\begin{tabular}{lll}
\toprule
Role & Model & Temperature \\
\midrule
Ablation generator & \texttt{gemini-3-flash-preview} & 0 \\
Response judge & \texttt{gemini-3.6-flash} & 0 \\
Independent auditor & \texttt{gemini-3.1-pro-preview} & 0 \\
\bottomrule
\end{tabular}
\caption{Models by role.}
\label{tab:models}
\end{table}
\section{Additional Results}
\label{app:extra}

Section~\ref{sec:results} reports one figure per claim. This appendix reports the rest: every method on every subset, every condition under every reading, and the per-subject values behind each macro-average. It covers how the demand rate was estimated (\ref{app:extra-demand}), retrieval in full (\ref{app:extra-retrieval}), context-window coverage (\ref{app:extra-coverage}), the ablation in full (\ref{app:extra-ablation}), selectivity and detection (\ref{app:extra-selectivity}), the per-subject tables (\ref{app:persubject}), and the deployed systems (\ref{app:deployed}).

% ---------- F.1 ----------
\subsection{Estimating the Demand Rate}
\label{app:extra-demand}

% P1 - why three procedures
The demand rate $\pi$ of~\eqref{eq:pi} is the quantity every other rate in this paper is conditioned on, and Proposition~\ref{prop:onesided} makes a single measurement of it a lower bound. Three procedures were therefore run, with different denominators and different failure modes, and this subsection gives each.

% P2 - the census
The first is a census. Phase~A was run over every user turn in the corpus, not over a sample, so a memory-bearing turn can be missed only if the resolver fails on it, never because it was not looked at. That is what makes the enriched stratum a population, and it is the only one of the three that bounds the recall term $r$ directly.

% P3 - the residual sweep
The second is a residual sweep. Every turn the census did not mark was re-examined under a widened criterion: 921 further candidates were proposed and two survived verification. A procedure that had been missing a substantial share of memory-bearing turns would not return two.

% P4 - the adjudication
The third is an independent referent adjudication over all 404 memory-bearing probes, which asks whether the thing the probe points at was already visible on screen. It removes 154, giving the strict reading of 250. This procedure cannot add a probe to the memory-bearing set, only remove one, so it tightens the estimate in the conservative direction.

% P5 - the three estimates side by side
Table~\ref{tab:demand} gives the three. The stratum is the paper's demand rate, because it alone is a sample drawn without regard to memory demand; the other two are checks on it over the turns the sample did not take, and together the two denominators partition 12{,}538 of the 12{,}795 user turns. Every estimate falls between one and four percent, and the reading moves the figure further than the procedure does.

\begin{table}[t]
\centering\small
\begin{tabular}{llrrr}
\toprule
Procedure & Denominator & recorded & strict & strict 2 \\
\midrule
Natural-rate stratum & 1{,}169 sampled turns   & 3.4 [2.5, 4.6] & 1.3 [0.8, 2.1] & 1.3 \\
Exhaustive sweep     & 11{,}369 unsampled turns & 3.2 [2.9, 3.5] & 2.1 [1.8, 2.3] & 1.7 [1.5, 1.9] \\
Sweep re-read        & 11{,}369 unsampled turns & 1.73 [1.45, 2.06] & 1.12 [0.94, 1.33] & 0.90 [0.76, 1.07] \\
\bottomrule
\end{tabular}
\caption{Demand rate as a percentage of user turns.}
\label{tab:demand}
\end{table}

% P6 - what the third row is and is not
The third row's numerator is 196 live rows [165, 234], obtained by re-reading the sweep's output under the builder's own batched instruction, which calls 44.5\% of the sweep's distant turns distant and 0.8\% of its thread turns. It says what the original procedure would have found on the unsampled turns, which is why it sits below the sweep's own figure, and it is reported because the difference between a liberal and a conservative resolver is itself information about the width of the estimate.

% P7 - two figures that are not the demand rate
Two quantities in the released tables are not demand rates and should not be quoted as such. The detector pools, 7 of 536, are a different instrument measured on a different population. The phase~A distant rates before verification, 7.8\% and 7.1\%, are claim rates: they say how often the resolver proposed a distant referent, not how often one survived.

% ---------- F.2 ----------
\subsection{Retrieval in Full}
\label{app:extra-retrieval}

% P1 - what the controls isolate
Each method is rescored against three substituted gold sets, which isolate what it reads. \emph{Position randomized} places the gold turns at random positions, destroying both the distance relation and the content relation. \emph{Offset permuted} preserves each item's distances from probe to gold and permutes which item they belong to, destroying the content relation alone. \emph{Broken oracle} supplies turns known to lie outside the gold set and must score zero. Table~\ref{tab:controls} gives all four readings over the 1{,}477 rows that carry a gold set, macro-averaged over subjects. The body reports the same contrasts pooled over items, which is the form in which the identity of Proposition~\ref{prop:nonident} holds; the two differ by a few points, and each is labeled where it appears.

\begin{table}[h]
\centering\small
\begin{tabular}{lrrrrr}
\toprule
Gold set & Random & Recency over user turns & Recency & BM25 & Oracle \\
\midrule
\multicolumn{6}{@{}l}{\emph{Required turn in the top five}} \\
Actual              & 0.073 & 0.372 & 0.979 & 0.282 & 1.000 \\
Position randomized & 0.092 & 0.072 & 0.104 & 0.056 & 1.000 \\
Offset permuted     & 0.089 & 0.377 & 0.982 & 0.216 & 1.000 \\
Broken oracle       & 0.000 & 0.000 & 0.000 & 0.000 & 0.000 \\
\midrule
\multicolumn{6}{@{}l}{\emph{Recall over the gold set at five}} \\
Actual              & 0.055 & 0.183 & 0.892 & 0.197 & 0.969 \\
Position randomized & 0.069 & 0.053 & 0.081 & 0.040 & 0.969 \\
Offset permuted     & 0.079 & 0.194 & 0.905 & 0.168 & 0.976 \\
Broken oracle       & 0.000 & 0.000 & 0.000 & 0.000 & 0.000 \\
\bottomrule
\end{tabular}
\caption{Retrieval under substituted gold sets.}
\label{tab:controlresults}
\end{table}

% P2 - the dissociation
The two scorers come apart exactly as the controls are designed to show. Randomizing positions takes recency from 0.979 to 0.104, which is chance, while permuting the offsets leaves it at 0.982: recency reads position and nothing else. Lexical matching falls under both, from 0.282 to 0.056 and to 0.216, because it reads content and the content relation is destroyed either way. The broken oracle returns zero for every method under every reading, so nothing in the measure rewards a turn outside the gold set. The proportional stratum gives the same pattern on 1{,}113 rows, with recency at 0.998 actual against 0.111 randomized and lexical matching at 0.265 against 0.044.

% ---------- F.3 ----------
\subsection{Context-Window Coverage}
\label{app:extra-coverage}

% P1 - what is computed and why it needs no model
For each memory-bearing probe and each window size, we ask whether every turn the reply requires lies within the window preceding the probe. The computation uses no model and no retriever: it is the position of the required turns in the transcript, so a reader can reproduce it from the release alone. Table~\ref{tab:coverage} gives it on a turn axis and on a token axis, under the recorded reading and both strict ones.

% P2 - what it says
The token axis, plotted in Figure~\ref{fig:window}, is the one that answers the question a reader brings. A frontier context window of 128k tokens holds everything 45.3\% of these probes require, and 200k holds 56.7\%. Full coverage arrives only near a million tokens, which is also the scale of the largest single history in the corpus at 841{,}994 tokens. On the turn axis, a system holding the last thousand turns is complete on 30.9\% of memory-bearing probes and needs ten thousand to reach 95\%, which is longer than nine of the ten histories here. Long context and retrieval are not alternatives at these lengths; a window large enough to make retrieval unnecessary is a window larger than most of these relationships.

\begin{table}[h]
\centering\small
\begin{tabular}{lrrr@{\qquad}lrrr}
\toprule
\multicolumn{4}{l}{\emph{Turns preceding the probe}} & \multicolumn{4}{l}{\emph{Tokens preceding the probe}} \\
Window & recorded & strict & strict 2 & Window & recorded & strict & strict 2 \\
\midrule
10      & 0.010 & 0.000 & 0.000 & 8k    & 0.109 & 0.088 & 0.083 \\
25      & 0.030 & 0.024 & 0.024 & 32k   & 0.248 & 0.212 & 0.205 \\
50      & 0.050 & 0.032 & 0.034 & 128k  & 0.453 & 0.440 & 0.424 \\
100     & 0.079 & 0.060 & 0.059 & 200k  & 0.567 & 0.548 & 0.532 \\
250     & 0.158 & 0.136 & 0.137 & 400k  & 0.782 & 0.772 & 0.756 \\
500     & 0.230 & 0.196 & 0.195 & 1M    & 1.000 & 1.000 & 1.000 \\
1{,}000  & 0.309 & 0.296 & 0.288 &       &       &       &       \\
2{,}000  & 0.473 & 0.460 & 0.449 &       &       &       &       \\
3{,}000  & 0.599 & 0.592 & 0.576 &       &       &       &       \\
5{,}000  & 0.745 & 0.744 & 0.722 &       &       &       &       \\
7{,}500  & 0.884 & 0.876 & 0.873 &       &       &       &       \\
10{,}000 & 0.950 & 0.944 & 0.946 &       &       &       &       \\
\bottomrule
\end{tabular}
\caption{Share of memory-bearing probes whose every required turn lies in the window.
Populations are 404, 250 and 205. Tokens are counted with \texttt{cl100k\_base}.}
\label{tab:coverage}
\end{table}

\begin{figure}[h]
\centering
\begin{tikzpicture}[x=1cm,y=1cm]
  % axes
  \draw[cGrey,thin] (0,0) -- (10.4,0);
  \draw[cGrey,thin] (0,0) -- (0,5.4);
  % y gridlines and labels
  \foreach \y/\lab in {0/0.0, 1.25/0.25, 2.5/0.5, 3.75/0.75, 5/1.0}{
    \draw[cWash,very thin] (0,\y) -- (10.2,\y);
    \node[left,font=\scriptsize,text=cGrey] at (-0.05,\y) {\lab};
  }
  % x ticks
  \foreach \x/\lab in {0/8k, 2.871/32k, 5.742/128k, 6.667/200k, 8.103/400k, 10/1M}{
    \draw[cGrey,thin] (\x,0) -- (\x,-0.08);
    \node[below,font=\scriptsize,text=cGrey] at (\x,-0.10) {\lab};
  }
  % coverage curve
  \draw[cData,very thick]
    (0,0.545) -- (2.871,1.240) -- (5.742,2.265)
    -- (6.667,2.835) -- (8.103,3.910) -- (10,5.000);
  \foreach \x/\y in {0/0.545, 2.871/1.240, 5.742/2.265, 6.667/2.835, 8.103/3.910, 10/5.000}{
    \filldraw[cData] (\x,\y) circle (1.6pt);
  }
  % value labels
  \node[above left,font=\scriptsize,text=cData] at (0,0.545) {0.109};
  \node[above left,font=\scriptsize,text=cData] at (5.742,2.265) {0.453};
  \node[above left,font=\scriptsize,text=cData] at (10,5.000) {1.000};
  % axis titles
  \node[font=\scriptsize,text=cGrey,rotate=90] at (-0.95,2.5)
    {coverage of required turn};
  \node[font=\scriptsize,text=cGrey] at (5.2,-0.70) {context budget (tokens, log scale)};
\end{tikzpicture}
\caption{Evidence coverage by context budget.}
\label{fig:window}
\end{figure}
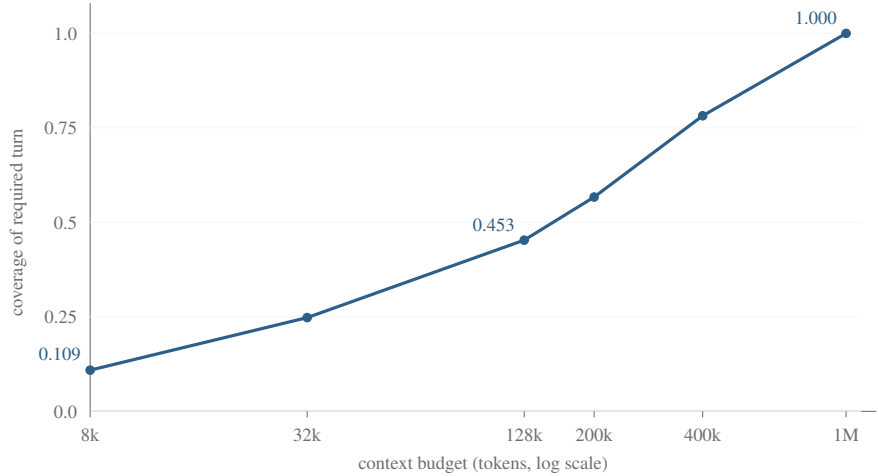

% ---------- F.4 ----------
\subsection{The Ablation in Full}
\label{app:extra-ablation}

% P1 - levels
Table~\ref{tab:levels} gives the response score reached by each condition, macro-averaged over subjects, on each population. Two things are visible there that the body's single table cannot show. On the memory-bearing rows the ordering of the conditions is not the ordering on the corpus as a whole: a ten-turn window, which ties a three-turn window overall, falls below it where memory is required, because the additional seven turns are the wrong seven. And exact context is the only condition whose advantage survives on every population.

\begin{table}[h]
\centering\small
\begin{tabular}{lrrrrr}
\toprule
Population & C0 & C1 & C2 & C3 & C4 \\
\midrule
All rows, both strata ($n{=}1{,}533$) & 1.293 & 1.590 & \textbf{1.739} & 1.358 & 1.587 \\
All rows, proportional ($n{=}1{,}169$) & 1.288 & 1.612 & \textbf{1.747} & 1.350 & 1.610 \\
Memory-bearing, recorded ($n{=}404$) & 1.410 & 1.452 & \textbf{1.780} & 1.501 & 1.377 \\
Memory-bearing, strict ($n{=}250$) & 1.555 & 1.519 & \textbf{1.822} & 1.697 & 1.513 \\
\bottomrule
\end{tabular}
\caption{Response score by condition.}
\label{tab:levels}
\end{table}

% P2 - contrasts, and the pattern across readings
Table~\ref{tab:contrasts} gives the paired contrasts. The pattern across readings is the subsection's finding: removing the probes whose referent was already on screen raises almost every conditional effect. On the hard tier the gain from exact context goes from 0.437 to 0.541, on the hard and dependent rows from 0.549 to 0.650, and on the intermediate tier from an interval containing zero to one that does not. The rows removed by the strict reading were diluting the effects, not carrying them, which is the argument for having recorded the judgment at all.

\begin{table}[h]
\centering\small
\begin{tabular}{lrrr}
\toprule
Contrast and population & recorded & strict & strict 2 \\
\midrule
\multicolumn{4}{@{}l}{\emph{Whole corpus, both strata}} \\
C1 minus C0 & $+0.297$ [$+0.197$, $+0.403$] & --- & --- \\
C2 minus C1 & $+0.149$ [$+0.107$, $+0.201$] & --- & --- \\
C3 minus C1 & $-0.232$ [$-0.328$, $-0.148$] & --- & --- \\
C4 minus C1 & $-0.003$ [$-0.052$, $+0.044$] & --- & --- \\
C2 minus C3 & $+0.381$ [$+0.304$, $+0.463$] & --- & --- \\
\midrule
\multicolumn{4}{@{}l}{\emph{Whole corpus, proportional stratum}} \\
C1 minus C0 & $+0.325$ [$+0.225$, $+0.423$] & --- & --- \\
C2 minus C1 & $+0.135$ [$+0.090$, $+0.186$] & --- & --- \\
C3 minus C1 & $-0.263$ [$-0.360$, $-0.172$] & --- & --- \\
C4 minus C1 & $-0.002$ [$-0.053$, $+0.046$] & --- & --- \\
\midrule
\multicolumn{4}{@{}l}{\emph{C2 minus C1 by subset, both strata}} \\
Hard & $+0.437$ [$+0.224$, $+0.673$] & $+0.541$ [$+0.334$, $+0.807$] & $+0.536$ \\
Hard and dependent & $+0.549$ [$+0.294$, $+0.802$] & $+0.650$ [$+0.410$, $+0.900$] & $+0.748$ \\
Intermediate & $+0.039$ [$-0.040$, $+0.108$] & $+0.106$ [$+0.017$, $+0.196$] & $+0.124$ \\
Intermediate and dependent & $+0.366$ [$+0.125$, $+0.500$] & $+0.419$ [$+0.135$, $+0.683$] & $+0.480$ \\
Memory-bearing & $+0.328$ [$+0.128$, $+0.586$] & $+0.304$ [$+0.203$, $+0.418$] & $+0.301$ \\
Memory-bearing and dependent & $+0.532$ [$+0.278$, $+0.789$] & $+0.598$ [$+0.386$, $+0.850$] & $+0.648$ \\
\bottomrule
\end{tabular}
\caption{Paired contrasts, bootstrapped over subjects. Whole-corpus rows do not vary with
the reading, which affects only the memory-bearing populations.}
\label{tab:contrasts}
\end{table}

% ---------- F.5 ----------
\subsection{Selectivity}
\label{app:extra-selectivity}

% P1 - The three measures
The five conditions of Appendix~\ref{app:extra-ablation} support three measures in the form of \citet{yoon2026benchpres}. Misapplication is the rate at which a reply uses supplied context on a probe whose verified locus is \emph{none}. Appropriate application is the rate at which a reply uses it on a memory-bearing probe. Restraint is the rate at which a reply declines to use it at all. Table~\ref{tab:selectivity} reports all three pooled over items with Wilson intervals; the per-subject means are in the released tables.

\begin{table}[h]
\centering
\small
\begin{tabular}{lccc}
\toprule
Condition & Misapplication & Appropriate application & Restraint \\
 & \scriptsize $n = 120$ & \scriptsize $n = 404$ & \scriptsize $n = 1{,}533$ \\
\midrule
C0 no context               & 0.0 [0.0, 3.1]    & 0.0               & 100.0 \\
C1 three preceding turns    & 46.7 [38.0, 55.6] & 34.7 [30.2, 39.5] & 63.5 [61.1, 65.9] \\
C2 recorded evidence turns  & 29.2 [21.8, 37.8] & 69.8 [65.2, 74.1] & 41.7 [39.2, 44.2] \\
C3 ten turns from BM25      & 60.8 [51.9, 69.1] & 62.4 [57.6, 67.0] & 36.9 [34.5, 39.4] \\
C4 ten preceding turns      & 54.2 [45.3, 62.8] & 47.5 [42.7, 52.4] & 47.7 [45.3, 50.3] \\
\midrule
C3m matched, neutral label  & 50.8 [42.0, 59.6] & 48.5 [43.7, 53.4] & 51.2 [46.4, 56.1] \\
C4m matched, neutral label  & 59.2 [50.2, 67.5] & 52.7 [47.9, 57.5] & 46.8 [42.0, 51.7] \\
\bottomrule
\end{tabular}
\caption{Selectivity by supplied context.}
\label{tab:selectivity}
\end{table}

% P2 - What the pattern is, and what the label does
Misapplication is exactly zero without context and non-zero under every condition that supplies any. The shipped conditions appear to show retrieval misapplying most, 60.8\% against 54.2\% for a recency window, but that ordering is the prompt label: under a matched pair holding label and ordering equal, retrieval misapplies less (50.8\% [42.0, 59.6]) than the recency window (59.2\% [50.2, 67.5]). The label effect is one-sided and consistent. Announcing the turns as retrieved memories raises memory use from 50.8 to 60.8\% on probes requiring none, from 48.5 to 62.6\% on memory-bearing probes and from 48.6 to 59.9\% on cold opens, while the same relabeling of a recency window moves nothing beyond its intervals. On the memory-bearing probes, the two conditions sit fourteen points apart under the shipped labels and cross under the neutral one, so how context is announced moves the decision to consult memory further than what the context contains.

% P3 - Separability by probe origin and by detector
Two detectors were run on the same populations. The first is the lexical scorer's own confidence, the top-one BM25 score of the probe against the turns before it, used as a signal that retrieval should occur. It reaches 0.546 [0.490, 0.605] on probes drawn from real conversation and 0.504 [0.423, 0.579] on cold opens, both containing chance, and 0.696 [0.629, 0.758] on authored questions over the same histories, which does not. A written question sits closer to its evidence in word overlap than a spoken turn does. The second is a model shown the probe and its three preceding turns and asked whether memory is needed. On the proportional stratum it is at chance on both of the first two splits, 0.530 [0.488, 0.572] and 0.548 [0.488, 0.609], and above chance against its own target at 0.601 [0.527, 0.676], rising to 0.715 [0.619, 0.815] under the strict reading. Its operating point is what matters: it flags 6 of the 40 probes that need memory and raises 40 false alarms doing so. On the question track it reaches 0.802 [0.742, 0.857] when shown three retrieved turns, and that threshold calls 40\% of answerable questions unanswerable. Every instrument ranks the cases somewhat and decides them badly.

% ---------- F.6 ----------
\subsection{Per-Subject Values}
\label{app:persubject}

% P1 - why the per-subject values matter more here than usual
The macro-average over subjects is the aggregation this paper defends, and this subsection is where its cost is visible. Under the recorded reading only four subjects hold five or more memory-bearing rows: U07 with 20, U08 with 39, U09 with 123 and U10 with 213. U02 holds one, U05 and U06 hold four each, and all three enter the seven-subject macro at full weight. Table~\ref{tab:persubject-hit} shows what that does. Lexical matching reaches a macro of 0.424 while its pooled value is 0.248 and no populated subject exceeds 0.400 except the one holding a single row. A macro over subjects of wildly unequal size produces a number no subject exhibits, which is why the body reports pooled figures and prints the macro beside them.

\begin{table}[h]
\centering\scriptsize\setlength{\tabcolsep}{3pt}
\begin{tabular}{lrrrrrrrr}
\toprule
& U02 & U05 & U06 & U07 & U08 & U09 & U10 & \\
Method & \scriptsize 1 & \scriptsize 4 & \scriptsize 4 & \scriptsize 20 & \scriptsize 39 & \scriptsize 123 & \scriptsize 213 & macro \ / \ pooled \\
\midrule
Random                  & 0.000 & 0.000 & 0.000 & 0.000 & 0.026 & 0.000 & 0.005 & 0.004 \ / \ 0.005 \\
Recency over user turns & 0.000 & 0.000 & 0.000 & 0.050 & 0.026 & 0.016 & 0.009 & 0.014 \ / \ 0.015 \\
Recency                 & 0.000 & 0.000 & 0.000 & 0.100 & 0.077 & 0.016 & 0.009 & 0.029 \ / \ 0.022 \\
BM25                    & 1.000 & 0.500 & 0.250 & 0.400 & 0.359 & 0.268 & 0.192 & 0.424 \ / \ 0.248 \\
Oracle                  & 1.000 & 1.000 & 1.000 & 1.000 & 1.000 & 1.000 & 1.000 & 1.000 \ / \ 1.000 \\
\bottomrule
\end{tabular}
\caption{Required turn in the top five, by subject, recorded reading.}
\label{tab:persubject-hit}
\end{table}

% P2 - the strict reading narrows it further
The strict reading removes U02 and U05 entirely and leaves five subjects. Lexical matching then reads 0.397 macro against 0.300 pooled, and recency is exactly zero on every subject. The dependent subset behaves the same way: BM25 reaches 0.658 macro against 0.347 pooled on the 167 recorded dependent rows, because U02, U05 and U06 contribute one or two rows each and score 1.000 on them.

\begin{table}[h]
\centering\scriptsize\setlength{\tabcolsep}{3pt}
\begin{tabular}{lrrrrrrrrrrr}
\toprule
Condition & U01 & U02 & U03 & U04 & U05 & U06 & U07 & U08 & U09 & U10 & macro \ / \ pooled \\
\midrule
\multicolumn{12}{@{}l}{\emph{Misapplication, verified locus none}} \\
C1 & 0.000 & 0.333 & 0.000 & 0.500 & 0.625 & 0.222 & 0.333 & 0.562 & 0.487 & 0.545 & 0.361 \ / \ 0.467 \\
C2 & 1.000 & 0.000 & 0.500 & 0.375 & 0.250 & 0.111 & 0.417 & 0.188 & 0.282 & 0.364 & 0.349 \ / \ 0.292 \\
C3 & 0.000 & 1.000 & 0.000 & 0.625 & 0.750 & 0.444 & 0.750 & 0.500 & 0.615 & 0.636 & 0.532 \ / \ 0.608 \\
C4 & 1.000 & 0.667 & 0.000 & 0.500 & 0.750 & 0.333 & 0.500 & 0.625 & 0.487 & 0.636 & 0.550 \ / \ 0.542 \\
\midrule
\multicolumn{12}{@{}l}{\emph{Appropriate application, memory-bearing, recorded}} \\
C1 & --- & 0.000 & --- & --- & 0.250 & 0.000 & 0.150 & 0.308 & 0.325 & 0.394 & 0.204 \ / \ 0.347 \\
C2 & --- & 1.000 & --- & --- & 0.500 & 0.500 & 0.600 & 0.744 & 0.642 & 0.737 & 0.675 \ / \ 0.698 \\
C3 & --- & 1.000 & --- & --- & 0.500 & 0.500 & 0.500 & 0.718 & 0.626 & 0.624 & 0.638 \ / \ 0.626 \\
C4 & --- & 0.000 & --- & --- & 0.750 & 0.250 & 0.500 & 0.538 & 0.341 & 0.554 & 0.419 \ / \ 0.483 \\
\bottomrule
\end{tabular}
\caption{Use of supplied context, by subject. C0 is zero in every cell.}
\label{tab:persubject-use}
\end{table}

% P3 - the use decisions
Table~\ref{tab:persubject-use} gives misapplication on the 120 verified-none rows and appropriate application on the memory-bearing rows. Misapplication is defined on all ten subjects, so its macro is better behaved than the retrieval macro; even so, U01 contributes one row and U03 two, and both take extreme values. C0 is exactly zero in every cell of every subject, which is the zero-fabrication result at the level it was measured.

\begin{table}[h]
\centering\scriptsize\setlength{\tabcolsep}{3pt}
\begin{tabular}{lrrrrrrrrrrr}
\toprule
Condition & U01 & U02 & U03 & U04 & U05 & U06 & U07 & U08 & U09 & U10 & macro \ / \ pooled \\
\midrule
C0 & 0.929 & 1.286 & 1.115 & 1.333 & 1.243 & 1.025 & 1.513 & 1.394 & 1.643 & 1.378 & 1.286 \ / \ 1.431 \\
C1 & 1.571 & 1.429 & 1.462 & 1.692 & 1.568 & 1.568 & 1.680 & 1.647 & 1.696 & 1.560 & 1.587 \ / \ 1.622 \\
C2 & 1.643 & 1.714 & 1.538 & 1.769 & 1.784 & 1.790 & 1.793 & 1.805 & 1.827 & 1.691 & 1.736 \ / \ 1.763 \\
C3 & 0.964 & 1.314 & 1.154 & 1.564 & 1.432 & 1.296 & 1.507 & 1.339 & 1.622 & 1.324 & 1.352 \ / \ 1.426 \\
C4 & 1.643 & 1.543 & 1.423 & 1.538 & 1.514 & 1.580 & 1.707 & 1.606 & 1.737 & 1.546 & 1.584 \ / \ 1.625 \\
\midrule
C2 $-$ C1 & +0.071 & +0.286 & +0.077 & +0.077 & +0.216 & +0.222 & +0.113 & +0.158 & +0.131 & +0.131 & 0.148 \ / \ 0.140 \\
\bottomrule
\end{tabular}
\caption{Content match by condition and subject, all judged rows.}
\label{tab:persubject-content}
\end{table}

% P4 - the response scores
Table~\ref{tab:persubject-content} gives content match per condition on all judged rows. The contrast the paper turns on, C2 minus C1, is positive for every one of the ten subjects, ranging from $+0.071$ to $+0.286$, so the effect is not carried by one relationship. Its macro is 0.148 [0.109, 0.194] and its pooled value 0.140 [0.126, 0.167]. On the memory-bearing rows alone the same contrast is 0.328 macro and 0.233 pooled, and the second of those is the $\gamma_1$ of record, since the identity in Proposition~\ref{prop:nonident} holds on pooled means. The proportional stratum, the strict readings and the dependent track are in the released tables.

% ---------- F.7 ----------
\subsection{Deployed Memory Systems}
\label{app:deployed}

% P1 - what was run, on what release, and what was not
Three memory systems were run and their coverage is not equal. Table~\ref{tab:deployed} marks each row with the release it was scored against. The Codex port covers all ten subjects and both tracks on the release reported here: 3{,}235 questions and 1{,}533 chat rows, every scoreable item, each judged against the gold that ships. Supermemory covers two subjects on both tracks, scored against the release of 2026-09-15 and not re-run; its stores were deleted the following day and the question sets of that date predate the enriched items later added to them. mem0 reached the first 850 of one subject's 986 questions before its search quota was exhausted, ran no chat track, and ran nothing on the other nine subjects. No row should be read against a row scored on a different release.

\begin{table}[h]
\centering\small
\begin{tabular}{lllrrrr}
\toprule
System & Subjects & Run & Track & $n$ & Answer accuracy & Required turn in top five \\
\midrule
Codex port  & ten & current & questions & 3{,}235      & 0.445 & 0.331 \\
Codex port  & ten & current & chat      & 1{,}533      & 0.581 & 0.137 \\
Supermemory & U09 & carried & questions & 986          & 0.664 & 0.537 \\
Supermemory & U09 & carried & chat      & 434          & 0.634 & 0.199 \\
Supermemory & U10 & carried & questions & 1{,}255      & 0.596 & 0.465 \\
Supermemory & U10 & carried & chat      & 482          & 0.633 & 0.117 \\
mem0        & U09 & carried & questions & 850 of 986   & 0.593 & 0.548 \\
\bottomrule
\end{tabular}
\caption{Deployed memory systems by coverage.}
\label{tab:deployed}
\end{table}

% P2 - what the Codex port shows
On the current release the port answers 44.5\% of answerable questions correctly, macro over ten subjects, and 58.1\% of chat probes. Its judge recall at five is 0.331 on questions and 0.137 on chat, and recall at fifteen equals recall at five almost everywhere, because the port returns a handful of notes per query and recall does not grow with $k$. On the memory-bearing chat rows it answers 55.4\% correctly with recall 0.250, and under the strict reading 66.6\% with recall 0.149 over five subjects. Every question row of the port is judged against the gold that ships, including the one item whose gold was repaired after the first scoring and which was re-run against the repaired version.

% P3 - what the carried rows show and what they cannot
The carried rows are stronger on both tracks, which is what a deployed product with a tuned ingestion should be, and they cannot be compared with the current rows: a different release, two subjects out of ten, and a chat track that has since been repaired. They are printed because withholding them would leave the appendix silent about the only deployed products in it, and marked so that nobody averages across the marker. Abstention on the unanswerable questions is 0.804 for the Codex port macro-averaged over ten subjects, and 1.000 and 0.632 for Supermemory on its two. It is not computable for mem0: none of that subject's unanswerable items was reached before the quota stopped the run, so its 0.044 abstention rate is on answerable questions alone.
\section{Persona and Profile Reconstruction}
\label{app:recon}

The reconstruction track asks a system to read a subject's full history in \texttt{source} and build the persona and profile layers from it (schemas in Appendix~\ref{app:schemas}). The persona holds Big Five levels, demographics, behavioral quirks, fifteen prose sections and personality dimensions; the profile holds identity fields, speaking style, Big Five levels, personality dimensions, categorized facts and named entities. The release pipeline builds both from the history in pieces: the persona chain distills it one day at a time into a store of claims, and the profile generators read it in windows. The agentic method instead reads the whole history in one pass, with no day slicing and no claim store, and every extracted item must cite the subject's own turns. Three systems each ran three times on each of the ten subjects.

\subsection{Task and Scope}
\label{app:recontask}

\paragraph{Subjects.} Ten subjects, U01--U10, with 115 to 12,627 turns each (Table~\ref{tab:peruser}). Each turn has an id of the form \texttt{dayN-M} and a sender: user (the subject), companion or app. A profile claim may rest on something the companion said if the subject ratified it, by affirming it, repeating it or building on it as settled; silence does not count, and the evidence cited is the subject's ratifying turn. Token counts here are the builder's estimate, characters $\div$ 4 over message text, about 20\% above the \texttt{cl100k\_base} counts of Section~\ref{sec:dataset}; they run from 2,743 to 1,023,911 per subject (Table~\ref{tab:recon-peruser}), and the transcript rendered for the prompts is 6--26\% larger. GT fills 239 persona and 476 profile fields across the ten subjects. The two longest subjects hold 83\% of all transcript tokens but only 18\% of the scored persona fields and 39\% of the scored profile fields, so pooled scores weight subjects by fields, and \S\ref{app:reconlength} reports each subject.

\paragraph{Scope.} All ten subjects are scored, by all three systems in both layers, and every pooled score covers all ten. A history too long for one context is read in two halves (\S\ref{app:reconsystems}).

\paragraph{Ground truth.} GT is the released persona and profile layers, which the release pipeline generated and the authors then reviewed by hand. The persona chain's late stages read fixed-size inputs, which shrinks the persona GT on long histories (\S\ref{app:reconlength}). The profile GT for U02--U10 was built on the pre-anonymization history and repaired against the release, and unlike the runs it also cites companion turns (1,034 of its 7,315 citations).

\subsection{Systems and Settings}
\label{app:reconsystems}

\paragraph{Systems.} Each system ran three independent runs per subject and layer, sharing no state. Claude Opus 5.5 ran as agent workflows with file tools and a 1M-token context, with the history inline (U01--U07) or in files each agent reads in full (U08--U10); contexts near the window were auto-compacted, and on U09, 23 of the 27 persona readers and 4 of the 6 profile agents answered from a 17--38k-token summary, as did one U10 agent. Codex GPT-5.6-sol ran through the Responses API, one request per call with the history inline, sent by a separate runner, not audited here, that applies the same prompts and edits; a single request cannot be compacted, and the largest was about 597k tokens (U09). Antigravity ran Gemini 3.8 Flash at the high thinking level in an agentic IDE, each run's agents working from the same run scripts as Claude's; every call had a 256k-token budget and none exceeded about 186k, so U09 and U10 were read in pieces. Only Codex held the whole history, or on U10 the whole half, in every call.

\paragraph{Prompts.}\label{app:reconprompts} The persona builder uses the chain's own prompts, lifted from its source and checked byte for byte, with fifteen exact-match edits that point them at the transcript instead of a claim store. One edit changes the method: the chain's prose prompt wrote a sentence for each dimension it was given, and the edited prompt may write any of the 21 dimensions the history evidences. The profile builder uses the profile generator's prompts verbatim, with a citation schema of bare turn ids. Every prompt ends with one line asking for the return value, each run's manifest lists all edits, and the full prompt texts ship with the code.

A persona run makes five Big Five calls, one per trait with its definition, each returning a level, a confidence, a rationale and citations, with medium when the evidence is thin. One call each extracts the quirks, the demographics, and the prose sections with one sentence per evidenced dimension. Each extracted demographic field is then checked by a call that reads only its cited turns and is dropped unless confirmed, and a psychotherapy-style note, written last from the transcript and the run's own Big Five levels and quirks, supplies one scored prose section. A profile run makes two calls: atomic claims, each with one of 16 fact categories, an optional named entity and bare turn-id citations, and then a summary holding the narrative, Big Five, identity, speaking style, personality dimensions, psychology and background.

\paragraph{Citations.} Citations of companion or app turns, and ids that do not exist, are removed, and an item left with no citation to the subject's turns is dropped: quirks, demographic fields and personality dimensions in the persona, and facts, identity, style, dimension, psychology and background entries in the profile. Big Five levels are kept without a citation. No personality dimension was ever dropped as uncited.

\paragraph{Histories too long for one context.} U10 is split at the day boundary nearest its token midpoint with at least three hours of silence across it, and each half is read single-pass by the same agents. The halves are merged by rule. In the persona, Big Five follows the chain's rule (the half with higher confidence wins; ties go to the half with more evidence, then to medium), quirks are pooled and de-duplicated with evidence unioned, and demographics go by confidence. In the profile, facts are concatenated, a summary field is present if either half has it, with evidence unioned, and Big Five takes the shared level, or Moderate when the halves disagree. In both, one text-only LLM call merges the halves' wording without seeing the history or the citation lists. A run with a failed half is excluded.

\subsection{Scoring}
\label{app:reconscoring}

Each scored field in each run is classed against GT: \textbf{TP} when the run matches a non-empty GT value or both mark an item present; \textbf{FP} when the run has something GT lacks; \textbf{FN} when GT has something the run lacks; and a \textbf{wrong value} when both have a value and they differ, counted once as FP and once as FN. Fields both leave empty are excluded everywhere, since counting them as agreement rewards writing nothing.

\paragraph{Categories.} The persona is scored on its five Big Five levels (low / medium / high), 13 demographic fields compared by exact value with no normalisation (GT stores age as a number), the personality dimensions GT or a run writes out of 21 names, 15 prose sections by presence, and GT's distinct quirks by recall: a run quirk matches when it shares at least two cited turns and at least half of the smaller citation set (near-duplicates merged, 25 listed and 15 scored). The profile is scored on its five Big Five levels (Low / Moderate / High), and by presence on 7 identity fields, 6 speaking-style fields, 21 personality dimensions and 16 fact categories, a category counting as present with at least one fact; GT's entities are scored by recall, on the exact name after lower-casing and trimming. A different Big Five level is a wrong value. Big Five is a fixed set, and quirks and entities are scored for recall only, so a run can add false positives in every other category.

\paragraph{Metrics.} F1 $= 2\,\mathrm{TP}/(2\,\mathrm{TP}+\mathrm{FP}+\mathrm{FN})$, with precision and recall, computed per run over the pooled cells of all ten subjects and reported as the mean (min--max) over the three runs. F1 is the Dice overlap between a run and GT over non-empty cells; \emph{runs agree} is the same Dice overlap between two runs, averaged over the three pairs, so the two can be compared directly.

\paragraph{Votes.} For each field GT fills, we count how many of the three runs match it: 3/3, a split (2/3 or 1/3), or 0/3. A 0/3 field is a systematic difference from GT; a split is run-to-run noise.

\paragraph{Length analysis.} Length is tested across subjects, by Spearman $\rho$ between history size and per-subject F1 (mean of runs) with an exact two-sided permutation p over all 10! orderings. Sensitivity variants drop the categories whose GT depends on history length, or score dimensions against the chain's pre-cut set (a diagnostic, not an alternative GT).

\subsection{Overall Agreement}
\label{app:reconoverall}

\begin{table}[htbp]
\centering\small\setlength{\tabcolsep}{4pt}
\caption{Agreement with GT over the ten subjects, mean (min--max) over the three runs. Models as in \S\ref{app:reconsystems}.}
\label{tab:recon-agreement}
\begin{tabular}{@{}llcccc@{}}
\toprule
System & Layer & F1 & Precision & Recall & Runs agree \\
\midrule
Claude & persona & 0.69 (0.68--0.69) & 0.56 (0.55--0.56) & 0.89 (0.88--0.90) & 0.93 \\
 & profile & 0.80 (0.79--0.81) & 0.80 (0.79--0.80) & 0.80 (0.79--0.82) & 0.95 \\
Codex & persona & 0.69 (0.69--0.69) & 0.57 (0.57--0.57) & 0.86 (0.86--0.86) & 0.93 \\
 & profile & 0.77 (0.77--0.77) & 0.75 (0.75--0.75) & 0.79 (0.79--0.80) & 0.95 \\
Antigravity & persona & 0.70 (0.68--0.72) & 0.59 (0.57--0.61) & 0.86 (0.84--0.88) & 0.95 \\
 & profile & 0.78 (0.77--0.78) & 0.80 (0.80--0.81) & 0.75 (0.74--0.76) & 0.91 \\
\bottomrule
\end{tabular}
\end{table}

\begin{itemize}
\item \textbf{Persona.} All three systems reach F1 0.69--0.70. Recall is 0.86--0.89 and precision 0.56--0.59: the runs rarely miss what GT holds, and most disagreement is a run writing something GT leaves empty.
\item \textbf{Runs agree with each other more than with GT:} 0.93--0.95 between runs against 0.69--0.70 with GT, so the difference from GT is systematic, not run-to-run noise.
\item \textbf{The systems agree with each other more than with GT:} the same Dice overlap between two systems' runs, averaged over the nine pairs, is 0.86--0.88 on the persona and 0.85--0.89 on the profile, although the systems span a frontier and a flash-tier model. They converge on one reading of the history, and it departs from GT in the same way.
\item \textbf{Profile} F1 is higher (Claude 0.80, Codex 0.77, Antigravity 0.78) because the profile GT fills more of the fixed lists both scorers check for presence, which leaves fewer extras to count as false positives (\S\ref{app:reconlength}).
\end{itemize}

\subsection{Results by Category}
\label{app:reconcats}

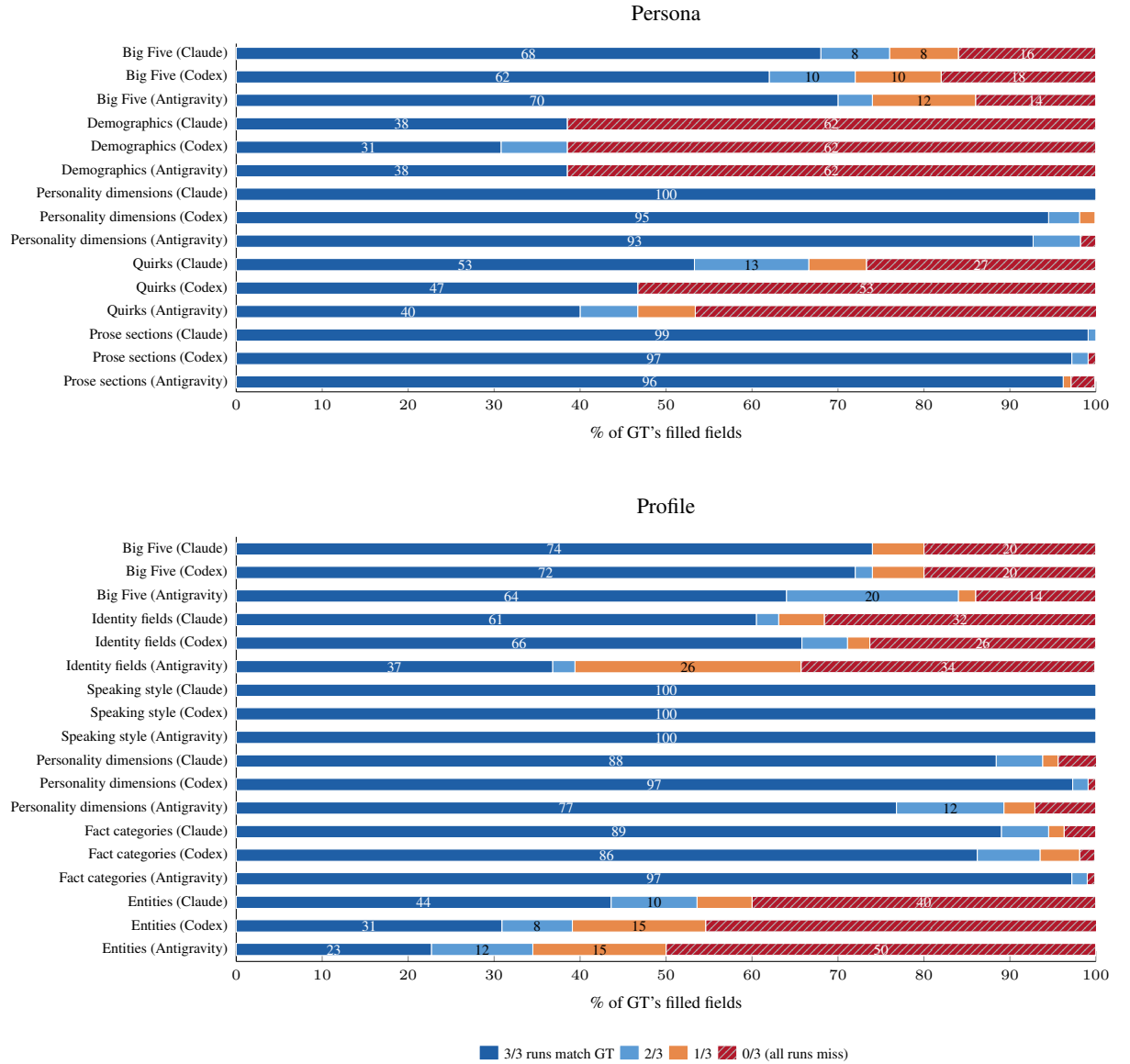
\begin{figure}[htbp]
\centering
% Generated by findings/paper/scripts/fig_votes.py. Needs pgfplots and the patterns library.
\begin{tikzpicture}
\begin{axis}[xbar stacked, bar width=5pt, width=\linewidth, height=6.6cm, xmin=0, xmax=100,
  title={\small Persona}, xlabel={\scriptsize \% of GT's filled fields}, symbolic y coords={{Prose sections (Antigravity)},{Prose sections (Codex)},{Prose sections (Claude)},{Quirks (Antigravity)},{Quirks (Codex)},{Quirks (Claude)},{Personality dimensions (Antigravity)},{Personality dimensions (Codex)},{Personality dimensions (Claude)},{Demographics (Antigravity)},{Demographics (Codex)},{Demographics (Claude)},{Big Five (Antigravity)},{Big Five (Codex)},{Big Five (Claude)}},
  ytick=data, y tick label style={font=\tiny}, x tick label style={font=\tiny}, enlarge y limits=0.03,
  legend style={at={(0.5,-0.18)},anchor=north,legend columns=4,font=\tiny,draw=none},
  point meta=explicit symbolic, nodes near coords, nodes near coords style={font=\tiny,color=white},
  axis x line*=bottom, axis y line*=left]
\addplot[fill={rgb,255:red,31;green,95;blue,166}, draw=white] coordinates {(96.2,{Prose sections (Antigravity)}) [96] (97.2,{Prose sections (Codex)}) [97] (99.1,{Prose sections (Claude)}) [99] (40.0,{Quirks (Antigravity)}) [40] (46.7,{Quirks (Codex)}) [47] (53.3,{Quirks (Claude)}) [53] (92.7,{Personality dimensions (Antigravity)}) [93] (94.5,{Personality dimensions (Codex)}) [95] (100.0,{Personality dimensions (Claude)}) [100] (38.5,{Demographics (Antigravity)}) [38] (30.8,{Demographics (Codex)}) [31] (38.5,{Demographics (Claude)}) [38] (70.0,{Big Five (Antigravity)}) [70] (62.0,{Big Five (Codex)}) [62] (68.0,{Big Five (Claude)}) [68]};
\addplot[fill={rgb,255:red,91;green,155;blue,213}, draw=white, every node near coord/.append style={color=black}] coordinates {(0.0,{Prose sections (Antigravity)}) [] (1.9,{Prose sections (Codex)}) [] (0.9,{Prose sections (Claude)}) [] (6.7,{Quirks (Antigravity)}) [] (0.0,{Quirks (Codex)}) [] (13.3,{Quirks (Claude)}) [13] (5.5,{Personality dimensions (Antigravity)}) [] (3.6,{Personality dimensions (Codex)}) [] (0.0,{Personality dimensions (Claude)}) [] (0.0,{Demographics (Antigravity)}) [] (7.7,{Demographics (Codex)}) [] (0.0,{Demographics (Claude)}) [] (4.0,{Big Five (Antigravity)}) [] (10.0,{Big Five (Codex)}) [10] (8.0,{Big Five (Claude)}) [8]};
\addplot[fill={rgb,255:red,232;green,137;blue,74}, draw=white, every node near coord/.append style={color=black}] coordinates {(0.9,{Prose sections (Antigravity)}) [] (0.0,{Prose sections (Codex)}) [] (0.0,{Prose sections (Claude)}) [] (6.7,{Quirks (Antigravity)}) [] (0.0,{Quirks (Codex)}) [] (6.7,{Quirks (Claude)}) [] (0.0,{Personality dimensions (Antigravity)}) [] (1.8,{Personality dimensions (Codex)}) [] (0.0,{Personality dimensions (Claude)}) [] (0.0,{Demographics (Antigravity)}) [] (0.0,{Demographics (Codex)}) [] (0.0,{Demographics (Claude)}) [] (12.0,{Big Five (Antigravity)}) [12] (10.0,{Big Five (Codex)}) [10] (8.0,{Big Five (Claude)}) [8]};
\addplot[fill={rgb,255:red,178;green,24;blue,43}, draw=white, postaction={pattern=north east lines, pattern color=white!70!black}] coordinates {(2.8,{Prose sections (Antigravity)}) [] (0.9,{Prose sections (Codex)}) [] (0.0,{Prose sections (Claude)}) [] (46.7,{Quirks (Antigravity)}) [47] (53.3,{Quirks (Codex)}) [53] (26.7,{Quirks (Claude)}) [27] (1.8,{Personality dimensions (Antigravity)}) [] (0.0,{Personality dimensions (Codex)}) [] (0.0,{Personality dimensions (Claude)}) [] (61.5,{Demographics (Antigravity)}) [62] (61.5,{Demographics (Codex)}) [62] (61.5,{Demographics (Claude)}) [62] (14.0,{Big Five (Antigravity)}) [14] (18.0,{Big Five (Codex)}) [18] (16.0,{Big Five (Claude)}) [16]};
\end{axis}
\end{tikzpicture}
\par\medskip
% Generated by findings/paper/scripts/fig_votes.py. Needs pgfplots and the patterns library.
\begin{tikzpicture}
\begin{axis}[xbar stacked, bar width=5pt, width=\linewidth, height=7.7cm, xmin=0, xmax=100,
  title={\small Profile}, xlabel={\scriptsize \% of GT's filled fields}, symbolic y coords={{Entities (Antigravity)},{Entities (Codex)},{Entities (Claude)},{Fact categories (Antigravity)},{Fact categories (Codex)},{Fact categories (Claude)},{Personality dimensions (Antigravity)},{Personality dimensions (Codex)},{Personality dimensions (Claude)},{Speaking style (Antigravity)},{Speaking style (Codex)},{Speaking style (Claude)},{Identity fields (Antigravity)},{Identity fields (Codex)},{Identity fields (Claude)},{Big Five (Antigravity)},{Big Five (Codex)},{Big Five (Claude)}},
  ytick=data, y tick label style={font=\tiny}, x tick label style={font=\tiny}, enlarge y limits=0.03,
  legend style={at={(0.5,-0.18)},anchor=north,legend columns=4,font=\tiny,draw=none},
  point meta=explicit symbolic, nodes near coords, nodes near coords style={font=\tiny,color=white},
  axis x line*=bottom, axis y line*=left]
\addplot[fill={rgb,255:red,31;green,95;blue,166}, draw=white] coordinates {(22.7,{Entities (Antigravity)}) [23] (30.9,{Entities (Codex)}) [31] (43.6,{Entities (Claude)}) [44] (97.2,{Fact categories (Antigravity)}) [97] (86.2,{Fact categories (Codex)}) [86] (89.0,{Fact categories (Claude)}) [89] (76.8,{Personality dimensions (Antigravity)}) [77] (97.3,{Personality dimensions (Codex)}) [97] (88.4,{Personality dimensions (Claude)}) [88] (100.0,{Speaking style (Antigravity)}) [100] (100.0,{Speaking style (Codex)}) [100] (100.0,{Speaking style (Claude)}) [100] (36.8,{Identity fields (Antigravity)}) [37] (65.8,{Identity fields (Codex)}) [66] (60.5,{Identity fields (Claude)}) [61] (64.0,{Big Five (Antigravity)}) [64] (72.0,{Big Five (Codex)}) [72] (74.0,{Big Five (Claude)}) [74]};
\addplot[fill={rgb,255:red,91;green,155;blue,213}, draw=white, every node near coord/.append style={color=black}] coordinates {(11.8,{Entities (Antigravity)}) [12] (8.2,{Entities (Codex)}) [8] (10.0,{Entities (Claude)}) [10] (1.8,{Fact categories (Antigravity)}) [] (7.3,{Fact categories (Codex)}) [] (5.5,{Fact categories (Claude)}) [] (12.5,{Personality dimensions (Antigravity)}) [12] (1.8,{Personality dimensions (Codex)}) [] (5.4,{Personality dimensions (Claude)}) [] (0.0,{Speaking style (Antigravity)}) [] (0.0,{Speaking style (Codex)}) [] (0.0,{Speaking style (Claude)}) [] (2.6,{Identity fields (Antigravity)}) [] (5.3,{Identity fields (Codex)}) [] (2.6,{Identity fields (Claude)}) [] (20.0,{Big Five (Antigravity)}) [20] (2.0,{Big Five (Codex)}) [] (0.0,{Big Five (Claude)}) []};
\addplot[fill={rgb,255:red,232;green,137;blue,74}, draw=white, every node near coord/.append style={color=black}] coordinates {(15.5,{Entities (Antigravity)}) [15] (15.5,{Entities (Codex)}) [15] (6.4,{Entities (Claude)}) [] (0.0,{Fact categories (Antigravity)}) [] (4.6,{Fact categories (Codex)}) [] (1.8,{Fact categories (Claude)}) [] (3.6,{Personality dimensions (Antigravity)}) [] (0.0,{Personality dimensions (Codex)}) [] (1.8,{Personality dimensions (Claude)}) [] (0.0,{Speaking style (Antigravity)}) [] (0.0,{Speaking style (Codex)}) [] (0.0,{Speaking style (Claude)}) [] (26.3,{Identity fields (Antigravity)}) [26] (2.6,{Identity fields (Codex)}) [] (5.3,{Identity fields (Claude)}) [] (2.0,{Big Five (Antigravity)}) [] (6.0,{Big Five (Codex)}) [] (6.0,{Big Five (Claude)}) []};
\addplot[fill={rgb,255:red,178;green,24;blue,43}, draw=white, postaction={pattern=north east lines, pattern color=white!70!black}] coordinates {(50.0,{Entities (Antigravity)}) [50] (45.5,{Entities (Codex)}) [45] (40.0,{Entities (Claude)}) [40] (0.9,{Fact categories (Antigravity)}) [] (1.8,{Fact categories (Codex)}) [] (3.7,{Fact categories (Claude)}) [] (7.1,{Personality dimensions (Antigravity)}) [] (0.9,{Personality dimensions (Codex)}) [] (4.5,{Personality dimensions (Claude)}) [] (0.0,{Speaking style (Antigravity)}) [] (0.0,{Speaking style (Codex)}) [] (0.0,{Speaking style (Claude)}) [] (34.2,{Identity fields (Antigravity)}) [34] (26.3,{Identity fields (Codex)}) [26] (31.6,{Identity fields (Claude)}) [32] (14.0,{Big Five (Antigravity)}) [14] (20.0,{Big Five (Codex)}) [20] (20.0,{Big Five (Claude)}) [20]};
\legend{3/3 runs match GT,2/3,1/3,0/3 (all runs miss)}
\end{axis}
\end{tikzpicture}
\caption{Vote distribution by category and system, pooled over the ten subjects: the share of GT's filled fields that 3, 2, 1 or 0 of a system's three runs match (the 0/3 segment is hatched).}
\label{fig:recon-votes}
\end{figure}

\begin{itemize}
\item \textbf{Shared across systems} (Figure~\ref{fig:recon-votes}): every system writes nearly every personality dimension and prose section GT has, then about ten to eleven more dimensions per subject and run in the persona (4--8 in the profile) and about two more prose sections, and every run matches every speaking-style field. 62\% of GT's filled demographic fields are missed by every run of each system, mostly wording that exact matching cannot credit (a correct paraphrase of an occupation scores zero).
\item \textbf{Big Five:} all three runs match 62--70\% of persona traits and 64--74\% of profile traits, and every miss, on all systems and both layers, is exactly one level off.
\item \textbf{Where the systems differ:} each is strongest somewhere else. Claude recovers the most entities (44\% named by all three runs, against 23--31\%) and quirks (8 of GT's 15, against 6--7); Codex the most identity fields (66\%) and profile dimensions (97\%); Antigravity the most fact categories (97\%, against 86--89\%). Antigravity's profile runs are the weakest on identity fields (37\%, against 61--66\%) and on entities, 50\% of which every run misses (40--45\% for the others).
\item \textbf{Entities:} all three runs name 23--44\% of GT's entities, and 26--48\% without U09, whose 23 entities include four composite role labels (e.g.\ ``user's sister'') that cannot match a plain name.
\end{itemize}

\subsection{Results by Subject and History Length}
\label{app:reconlength}

\begin{table}[htbp]
\centering\small\setlength{\tabcolsep}{5pt}
\caption{F1 per subject (mean of the three runs), and Spearman $\rho$ between transcript size and F1 across the ten subjects with its exact p. \emph{Adjusted} drops the categories whose GT depends on history length: extra dimensions and demographics in the persona, and every presence category in the profile, leaving Big Five and entities.}
\label{tab:recon-peruser}
\begin{tabular}{@{}lrcccccc@{}}
\toprule
& & \multicolumn{3}{c}{Persona F1} & \multicolumn{3}{c}{Profile F1} \\
\cmidrule(lr){3-5}\cmidrule(lr){6-8}
subject & tokens & Claude & Codex & Antigravity & Claude & Codex & Antigravity \\
\midrule
U01 & 2,743 & 0.72 & 0.73 & 0.70 & 0.74 & 0.67 & 0.82 \\
U02 & 11,662 & 0.79 & 0.81 & 0.78 & 0.65 & 0.55 & 0.67 \\
U03 & 4,493 & 0.75 & 0.72 & 0.72 & 0.64 & 0.62 & 0.68 \\
U04 & 7,326 & 0.76 & 0.75 & 0.86 & 0.76 & 0.77 & 0.77 \\
U05 & 12,086 & 0.74 & 0.68 & 0.78 & 0.76 & 0.75 & 0.78 \\
U06 & 83,935 & 0.84 & 0.76 & 0.81 & 0.77 & 0.72 & 0.77 \\
U07 & 95,760 & 0.65 & 0.68 & 0.68 & 0.86 & 0.89 & 0.84 \\
U08 & 123,032 & 0.73 & 0.76 & 0.71 & 0.90 & 0.89 & 0.92 \\
U09 & 621,374 & 0.47 & 0.55 & 0.52 & 0.88 & 0.86 & 0.77 \\
U10 & 1,023,911 & 0.47 & 0.48 & 0.52 & 0.81 & 0.75 & 0.71 \\
\midrule
\multicolumn{2}{@{}l}{$\rho$, as scored} & $-$0.55 & $-$0.44 & $-$0.55 & $+$0.83 & $+$0.65 & $+$0.18 \\
\multicolumn{2}{@{}l}{\quad exact p} & 0.105 & 0.204 & 0.105 & 0.005 & 0.049 & 0.632 \\
\multicolumn{2}{@{}l}{$\rho$, adjusted} & $+$0.02 & $+$0.01 & $-$0.01 & $-$0.36 & $-$0.06 & $-$0.66 \\
\multicolumn{2}{@{}l}{\quad exact p} & 0.973 & 1.000 & 1.000 & 0.307 & 0.872 & 0.044 \\
\bottomrule
\end{tabular}
\end{table}

\paragraph{Per subject.} U09 and U10 are the weakest persona subjects for every system (F1 0.47--0.55, against 0.65--0.86 on U01--U08), and U10 has the largest share of fields every run misses (22--26\%). On the profile, U09 is among the best subjects for Claude and Codex, and U10 is mid-range.

\paragraph{Across subjects.} As scored (Table~\ref{tab:recon-peruser}), persona F1 falls with history length ($\rho$ $-$0.44 to $-$0.55, not significant) and profile F1 rises ($\rho$ +0.83 for Claude, p = 0.005; +0.65 for Codex, p = 0.049; +0.18 for Antigravity, p = 0.632). Both trends come mainly from how GT was built, not from how the runs read:
\begin{itemize}
\item \textbf{The runs behave the same way in both layers.} On U09 and U10 every system writes 20--21 of the 21 personality dimensions in the persona, and Claude and Codex do in the profile (Antigravity 18).
\item \textbf{The persona GT shrinks relative to the history.} The chain's prose stage sees only the first 14,000 characters of its grounded content, so on U07--U10 the persona GT keeps 4, 6, 3 and 2 of the 20, 18, 18 and 19 dimensions the chain had grounded, and its Big Five, quirk and demographic stages see only the first 11,000 characters of the claim store and 6,000 of the day summaries. The runs' other dimensions count as false positives, although 83--94\% of them on U07--U10 are dimensions the chain itself had grounded. Removing extra dimensions and demographics removes the persona trend ($\rho$ $-$0.01 to +0.02); scoring dimensions against the pre-cut set reverses it ($\rho$ +0.53 to +0.66), partly by construction.
\item \textbf{The profile GT grows with the history.} Its generators read the history in windows and merge the results, so it grows from 5--26 facts on U01--U05 to 227--797 on U07--U10 and fills its fixed lists on long histories (all 21 dimensions on U08 and U09, 15--16 of 16 fact categories on U07--U10). Runs that fill the same lists score true positives with no room for false positives, so profile precision rises with length: the presence categories alone give $\rho$ +0.85 to +0.88, while the content categories (Big Five and entities) lean negative ($\rho$ $-$0.66 to $-$0.06), significantly only for Antigravity (p = 0.044).
\item \textbf{Big Five}, the one item both layers score as content, shows no significant trend in either ($\rho$ $-$0.39 to +0.37, every p $\geq$ 0.26).
\end{itemize}
With the confounded categories removed, U09 and U10 still score 0.78--0.84 against a U01--U08 mean of 0.88--0.90, and U10's persona Big Five (0.47--0.60) is at or near each system's lowest; with two subjects, and U10 read in halves, this residual is suggestive only.

\subsection{Cost}
\label{app:reconcost}

\begin{table}[htbp]
\centering\small
\caption{Cost per run over the ten subjects (U10 read in halves): model calls, tokens processed and output tokens.}
\label{tab:recon-cost}
\begin{tabular}{@{}lrrrrrr@{}}
\toprule
& \multicolumn{3}{c}{Persona} & \multicolumn{3}{c}{Profile} \\
\cmidrule(lr){2-4}\cmidrule(lr){5-7}
System & calls & processed & output & calls & processed & output \\
\midrule
Claude & 1,716 & 545.5M & 3.0M & 296 & 117.8M & 904k \\
Codex & 138 & 17.7M & 153k & 28 & 3.7M & 182k \\
Antigravity & 498 & 35.8M & 677k & 509 & 34.3M & 902k \\
\bottomrule
\end{tabular}
\end{table}

Tokens come from each system's own logs, over all attempts including retries, divided by three for a per-run figure; \emph{processed} counts all input tokens, cached ones included, and output includes reasoning tokens. At the same persona F1, Codex processes about 31$\times$ fewer tokens per run than Claude, and Antigravity about 15$\times$ fewer. Call counts are not comparable across systems, since an agent step re-sends its context while a Codex call is one request. Three counts are low: Claude's compaction calls are not logged (about +4.5\% on U09 persona), 15 failed Codex profile calls carry no usage, and the sessions that launched the runs are excluded.

\section{Datasheet}
\label{app:datasheet}

This datasheet follows the structure of \citet{gebru2021datasheets}.

% ---------- G.1 ----------
\subsection{Motivation}
\label{app:ds-motivation}

The corpus was assembled to measure whether a system knows when to remember and what to remember, on conversations nobody wrote for that purpose. No existing corpus pairs real longitudinal conversation with a per-item record of whether memory was required, and Appendix~\ref{app:sub} sets out why: the records are private, so the field substituted written users, authored questions and model judgment for them. The benchmark is the narrowest use of the result; Appendix~\ref{app:discussion} describes the others. Assembly, annotation and release were carried out by the authors.

% ---------- G.2 ----------
\subsection{Composition}
\label{app:ds-composition}

\emph{What an instance is.} The release has two kinds. A turn is one message in a conversation, with its sender, timestamp, text and a marker where media accompanied it. An item is one evaluation instance: a probe turn, the context a correct reply may draw on, a reference reply, and the record of how those were decided.

\emph{How many.} Ten subjects contribute 27{,}218 turns across 430 subject-days, of which 12{,}795 are the subject's, 14{,}329 the companion's and 94 the application's, totaling 1.64 million tokens. The chat track ships 2{,}034 rows, 1{,}600 sampled and 434 abstention, of which 1{,}967 are scoreable. The question track ships 3{,}312 items, of which 3{,}235 are scoreable. The profile layer holds 3{,}580 claims and the persona layer 390 content slots per the schemas of Appendix~\ref{app:schemas}.

\emph{Sample or population.} For each of the ten subjects the corpus contains every turn from the first message to the last, so it is complete for those subjects and is not a sample of any population. The probe strata are described in Appendix~\ref{app:probes}; only the proportional stratum supports a rate.

\emph{Labels.} Each item carries whether retrieval beyond the recent window is required, which store holds the needed record, a tier and category recomputable from the context lists, a reference reply, and a five-phase derivation recording how each was reached.

\emph{What is missing, and why.} Media content is not released, only a marker that media were present. Model replies generated for comparison and never shown to a subject, and messages emitted inside in-application games, were removed before delivery, leaving 3{,}273 vacant slots in the identifier space. Sixty-seven chat rows are withdrawn and ship without probe text, each with its reason. Timestamps carry a per-subject offset, so intervals are exact and calendar position is not.

\emph{Relationships between instances.} Every claim in every derived layer cites turn identifiers in the underived layer: 18{,}983 citations, none of which fails to resolve. Appendix~\ref{app:refs} reports the check.

\emph{Splits.} There is no canonical train and test split, and none should be inferred from the strata, which partition by how probes were selected.

\emph{Errors and noise.} The chat track's independent audit is in Appendix~\ref{app:audit}; the question track carries a content defect rate of 3.8\% [3.2, 4.6] after four repair passes, given in Appendix~\ref{app:qagt}. The profile ships 455 claims whose audit verdict is unsupported, retained so that the audit is itself auditable; no chat item cites one.

\emph{Self-containment.} The release is self-contained. Nothing is linked and nothing requires an external resource to interpret.

\emph{Sensitive and distressing content.} The corpus contains disclosures about health, family, work and emotional distress, including material falling within the special categories of Article~9 of the GDPR. It contains that material necessarily: a reply ignoring what someone has said about their health is not the correct reply. Readers and downstream users should expect content that is personal and at times distressing.

\emph{Subpopulations.} None can be identified. The persona layer's demographic slots are empty on 147 of 160, so the corpus supports no analysis by demographic group and should not be used to attempt one.

% ---------- G.3 ----------
\subsection{Collection}
\label{app:ds-collection}

\emph{Source.} The conversations come from the operational store of a single mobile companion application, described in Appendix~\ref{app:setting}. They were produced by people using a deployed product for their own purposes, not by participants recruited to a protocol, and every released turn is one the subject actually saw.

\emph{Mechanism and timeframe.} Records were exported from the production database. Observation windows run from each subject's first message to their last authored turn and span 36 to 120 days.

% ---------- G.4 ----------
\subsection{Preprocessing and Labeling}
\label{app:ds-preprocessing}

Three passes were applied. Structural identifiers were removed and every turn was rewritten so that what was said survives and the words used to say it do not; dates were shifted by a per-subject offset described in Appendix~\ref{app:time}. The rewriting cost 30 chat rows, withdrawn because the probe could not be released or its cited context was removed (Appendix~\ref{app:withdraw}); its effect on retrieval and reply scores is not measured, since that would need the original wording, which is not released. Shadow replies and in-application game turns were removed. Labels were then derived by the five phases of Appendix~\ref{app:phases}, and the result was graded by an independent audit and by the twenty-six mechanical gates released with the corpus. The original text is not released and is not recoverable from the release. The labeling and scoring software is distributed with the corpus.

\emph{What the persona layer is not.} The \texttt{persona} layer contains trait estimates and a clinical-style summary, produced by a language model from the conversation. Neither is a validated psychological instrument. Neither was administered to the subject or reviewed by them, and the release terms permit no use of either as an assessment of any person. They are released because they are part of what the derivation produced, and because withholding a derived layer while releasing the layer it came from would make the derivation uncheckable.
% ---------- G.5 ----------
\subsection{Uses}
\label{app:ds-uses}

The corpus has been used for the benchmark reported here and for nothing else. Appendix~\ref{app:discussion} sets out uses outside machine learning that the record could support.

\emph{What composition implies for future use.} Ten subjects, one application, English text after surrogate substitution, and shifted dates. Rates may be computed only on the proportional stratum. Quantities conditioned on memory demand rest on five subjects under the strict reading and should not be reported as though they rested on ten.

\emph{Uses the corpus must not be put to.} Profiling, identifying or targeting any individual, whether a subject of the corpus or a third party mentioned in it. Any commercial application. Any claim about whether these relationships benefited or harmed anyone: the corpus records that a relationship continued, not what it did for the person in it.

% ---------- G.6 ----------
\subsection{Maintenance}
\label{app:ds-maintenance}

The corpus is maintained by the authors and reached at the contact address distributed with the release; both are withheld for anonymous review. Each release carries a manifest naming every shipped file with its hash, so a version is identified by content, and every repair applied since the first audit is recorded in a statement distributed with it. Corrections will be issued as new versions under the same scheme, with prior manifests retained so that a published result can be tied to the exact files it was computed on. The scoring package accepts submissions from others, and the procedure is documented with it.
\section{Discussion}
\label{app:discussion}

This paper uses the corpus for one thing: measuring whether a system knows when to remember and what to remember. That is the narrowest use it has. What the release contains is ten people talking to the same companion for months, with what they said, what could be established about them, which turn established each claim, and how every label was reached. Several lines of work have wanted records of that kind and have not had them. This appendix sets out what each substituted instead, what recent work in each says the substitute costs, and which of their questions the record could answer: evaluation and conversational modeling (\ref{app:disc-nlp}), affect and mental state (\ref{app:disc-affect}), the psychology of these relationships (\ref{app:disc-psych}), and consumer research
(\ref{app:disc-market}). Appendix~\ref{app:disc-limits} states what the corpus cannot do. The pattern repeats in each case, and it is the pattern of Appendix~\ref{app:sub}: a field needed real people observed over time, could not obtain them, substituted something, and is now measuring what the substitution cost.

% ---------- H.1 ----------
\subsection{Modeling People in Conversation}
\label{app:disc-nlp}

% P1 - what a person needs from a reply, measured
For this community, the corpus carries a result about how people behave in conversation as well as a dataset. Benchmarks decide in advance what a person will need from a reply. An authored question stands in for the moment a person reaches back, and a model's judgment stands in for what the person meant (Appendix~\ref{app:sub}). Proposition~\ref{prop:nonident} prices the first substitution. A benchmark that supplies context and measures the improvement reports the sum of two effects, one on messages that need the context and one on messages that do not, and it may credit the sum to the first only where the second is absent. On an authored corpus that holds by construction, because every question was written to need something. Real people rarely talk that way. Here the probes that need nothing earlier contribute twenty-four times as much of the measured gain as those that do. The result applies to any benchmark whose probes were written, and it rests on a fact about people: most of what they say to a companion concerns the present.

% P2 - what the release supports that a leaderboard does not
The release also supports work a scored benchmark does not. The restraint decision has per-probe supervision, 1{,}129 negatives against 404 positives, which is what training a retrieval gate requires and what no other corpus supplies; the architectures surveyed in Appendix~\ref{app:attempts} retrieve on every turn in part because nothing has ever told them when not to. The derivations are supervision of a second kind, since each item records the referent a resolver proposed, what verification kept of it, the rule that fixed the label and the turns the reply rested on, so a model can be trained or evaluated on the intermediate decisions and not only on the outcome. And the coverage curve of Section~\ref{sec:results} answers by measurement a question the field answers by assertion: a system holding the last thousand turns has everything it needs on 31\% of memory-bearing probes, and reaching 95\% takes ten thousand turns, which is longer than 9 of the 10 histories in this corpus. At these lengths, long context and retrieval are not alternatives.

% ---------- H.2 ----------
\subsection{Affect, Mental State and Prediction}
\label{app:disc-affect}

% P1 - what the line is doing now
A third line models the interaction itself, and it has moved in one direction over the past two years: from classifying a turn to maintaining a state that persists across turns. \citet{zou2026affective} model the trajectory of affect through a conversation, and \citet{zhao2026stateless} argue for replacing a stateless responder with one that carries a psychological model of the person it is speaking to. Benchmarks have followed, with \citet{iyer2026heart} scoring people and models on the same support dialogues. A parallel effort attributes beliefs and intentions: \citet{li2026rectom} ask whether a conversational recommender infers what a user wants, \citet{wang2026mind2dialogue} train a responder on simulated user mental states directly, and \citet{ni2026survey} survey the simulators on which most of this rests.

% P2 - the state itself as the object of evaluation
A closer line makes the internal state the object of evaluation. \citet{yadav2026dialtom} ask a model to infer beliefs, desires, intentions, emotions, knowledge and trust from a dialogue, then to forecast how that dialogue continues from those inferences alone, and report that models identify the states and cannot use them to predict what follows. \citet{bawatneh2026omnitom} score belief structures directly, over 895 narratives and some 22{,}000 labeled belief propositions, and find actor-specific belief tracking is where models fail. \citet{zhao2026affectverse} learn a predictive representation of how affect unfolds and inject it into a language model, evaluated across nine clip-level corpora. \citet{cheng2026using} run the argument the other way, fitting a cognitive model of human decision-making to improve a language model's simulation of people in persuasion games, and \citet{lin2506large} sets out what that practice can and cannot support. The dissociation \citet{yadav2026dialtom} report is the one this paper finds on a different axis: there a model names a mental state and cannot act on it, here a model ranks probes needing memory above chance and cannot decide which to retrieve for.

% P2 - what they substitute, and what the record supplies
Each of these needs something it cannot observe. A persistent psychological state is worth maintaining only if the person persists, and what these lines evaluate on does not persist: authored narratives, clip-length recordings, games that end, and support corpora that are largely single sessions between strangers, assembled for the purpose. A mental-state attribution is worth scoring only against what someone actually wanted, and a simulated user wants what it was told to want. The recurring move is to construct the state because it cannot be watched: the situated responder builds a psychological world, the mental-state trainer simulates one. This corpus watches one instead. It records the same person across months, what they disclosed and when, what the companion said back, and whether they came back; its profile layer carries stressors, coping patterns and behaviors with the turns that established each and the window over which each held, so a claim about someone's state on day forty can be checked against what they said on day nine. It records that a relationship continued.

% ---------- H.3 ----------
\subsection{The Psychology of Human--AI Relationships}
\label{app:disc-psych}

% P1 - the field and the constraint it works under
Outside machine learning, these relationships are studied directly, and Appendix~\ref{app:release} describes what has been built: interview studies of how the bond develops over months, thematic analyses of the support people report receiving, and grounded-theory work on dependence and on distress when a system's behavior changes. \citet{pi2026only} survey what longitudinal work exists, and \citet{zhang2025rise} report on well-being at scale. Every one of these studies is built from what people say about the interaction, through interviews, surveys, app reviews and forum posts. The interactions themselves are private, and this is the constraint that produced the three substitutions of Appendix~\ref{app:sub}, met by a different field and answered a different way.

% P2 - what a transcript makes answerable
What a transcript adds is a check. A field whose evidence is self-report about an interaction has no way to compare the report against the interaction, and for ten people this corpus closes that gap. Some of what it shows is not what a survey would return. The companion opens the exchange on 217 of the 430 participant-days, so the AI initiates these relationships about half the time. More than half of the companion's messages offer tappable replies, and 33 of the 12{,}795 messages participants sent reproduce one, so the affordance the product invests in is almost entirely unused. Self-description fills where it asks for an interpretation and empties where it asks for a category, 91.9\% empty across demographic slots against 26.1\% elsewhere, and it does so regardless of how much a person said. And most of what people bring to a companion concerns the present. Only 3.4\% of their messages depend on anything said earlier, a fact about companionship that work organized around memory would not predict. None of these is obtainable from an account of an interaction. All are obtainable from the interaction.

% ---------- H.4 ----------
\subsection{Consumer Research and Marketing}
\label{app:disc-market}

% P1 - the practice, and the year it began to be audited
Consumer research has replaced its respondents. Synthetic panels, in which a model is prompted into a persona and asked what it would prefer or buy, are now routine in commercial practice, and the field has begun to audit them. \citet{lukauskas2026plausible} find such respondents plausible and not valid against psychometric criteria. \citet{tigre2026can} propose diagnostics and corrections for consumer panels, which is itself a statement that uncorrected panels are unreliable. \citet{chen2026synthetic} benchmark simulated survey responses across domains and report where they break. \citet{maier2025llms} recover purchase intent under a particular elicitation, which sharpens the question, since it shows the answer depends on how the model is asked. Beneath all of it sits the result of \citet{venkit2026need}: demographic attributes explain roughly 1.5\% of the variance in how similarly two people respond, so a persona assembled from them describes very nearly an unrelated person. Real self-description points the same way. In this corpus the demographic fields a synthetic panel starts from are empty 156 times in 160, while the remaining persona fields are filled three times in four.

% P2 - what is being asked for, and what the release permits
Every one of these audits ends the same way, by asking for real responses to validate against, and preferably responses that were not produced for a study. This corpus offers the field a source of validation. A synthetic persona can be held against a real one whose every claim carries the turns that established it, a confidence and the window over which it held. A claim to reproduce individual difference can be held against people whose differences were recorded. The corpus is released for research and its terms permit nothing else. It may not be used to profile or target anyone, since the participants consented to research and to no other use, and the argument of this subsection is only that the question the field is publicly asking is the one the record answers.

% ---------- H.5 ----------
\subsection{Limitations}
\label{app:disc-limits}

% P1 - the corpus
Ten people used one application over four months. They chose that product, so they are not a sample of companion users and still less of people, and the application frames the relationship in a particular way; how much of what is reported here survives a different product is unknown. Every turn was rewritten during anonymization, so the corpus preserves what was said and not the words used to say it; the rewriting cost 30 chat rows (Appendix~\ref{app:withdraw}), and its effect on retrieval and reply scores is not measured. Timestamps were shifted by a per-subject offset, so intervals are exact while calendar position and day of the week carry no information.

% P2 - the annotation
Three judgments rest on a single reader. The on-screen verdicts of Appendix~\ref{app:onscreen} were made by one person over four passes, with the final rule reapplied to the earliest, and one reader means a judgment without an agreement rate. The judge of Appendix~\ref{app:scoring} agrees with that reader at $\kappa = 0.804$ on the binary use decision and $0.545$ on the graded scale, so every graded score carries the weaker of the two. The question track carries a content defect rate of 3.8\% [3.2, 4.6] after repair, higher than the chat track's, reported in Appendix~\ref{app:qagt}.

% P3 - the results
Two results need their scope stated. Proposition~\ref{prop:nonident} is about a method and holds wherever the demand rate is below one; the 96\% that instantiates it is about this corpus and this pair of conditions, and a different pair would give a different split. Under the strict reading, five subjects contribute no memory-bearing probes at all, so every quantity conditioned on memory demand is a macro-average over five subjects with intervals to match, and none of them should be read as though it rested on ten.
\section{Notation and Terminology}
\label{app:notation}

Every symbol used in this paper is introduced where it is first needed, and it is collected here. Table~\ref{tab:symbols} gives the symbols and Table~\ref{tab:terms} the terms whose meaning in this paper is narrower than in ordinary use.

\begin{table}[h]
\centering\small
\begin{tabular}{llp{6.1cm}}
\toprule
Symbol & Introduced in & Meaning \\
\midrule
\multicolumn{3}{@{}l}{\emph{The record}} \\
$u$ & \ref{app:objects} & A subject; the corpus has ten \\
$m$, $m_i$ & \ref{app:objects} & A message in a subject's stream \\
$M_u$ & \ref{app:objects} & The stream of subject $u$, ordered in time \\
$N_u$ & \ref{app:objects} & Its length \\
$s(m)$, $\tau(m)$, $\iota(m)$ & \ref{app:objects}, \ref{app:medium} & The sender, timestamp and stored identifier of $m$ \\
$\mathcal{S} = \{S_1,\dots,S_L\}$ & \ref{app:objects} & The stores a referent may live in \\
\midrule
\multicolumn{3}{@{}l}{\emph{A probe}} \\
$p$ & \ref{app:objects} & A probe: a message attributable to the subject \\
$H(p)$ & \ref{app:objects} & Its history, every message preceding it \\
$W_k(p)$ & \ref{app:objects} & Its recent window: the $k$ messages preceding it \\
$\overline{W}_k(p)$ & \ref{app:objects} & The distant history, $H(p)$ less the window \\
$k(p)$ & \ref{app:objects} & The window size recorded for the probe \\
$R(p)$ & \ref{app:objects} & Its reference set: the records the gold used \\
$g(p)$ & \ref{app:objects} & Its gold response \\
\midrule
\multicolumn{3}{@{}l}{\emph{Ground truth and a system's output}} \\
$\rho^\star(p)$ & \ref{app:objects} & Whether retrieval beyond the window is required \\
$\lambda^\star(p)$ & \ref{app:objects} & Which stores hold the references \\
$\hat\rho(p)$, $\hat\lambda(p)$ & \ref{app:objects} & A system's predictions of the two \\
$\hat R(p)$, $\hat y(p)$ & \ref{app:scoring} & What it retrieved and what it replied \\
$\mathrm{uses}(\cdot)$ & \ref{app:scoring} & Whether a reply asserts a recorded past detail \\
$J(\cdot,\cdot)$ & \ref{app:scoring} & The judge's verdict on a reply against the gold \\
$\mathbf{1}[\cdot]$ & \ref{app:objects} & The indicator function \\
\midrule
\multicolumn{3}{@{}l}{\emph{Rates and estimation}} \\
$\pi$ & \ref{app:objects} & The demand rate, $\Pr[\rho^\star = 1]$ \\
$\hat\pi$ & \ref{app:lowdemand} & Its estimate from the derivation \\
$r$, $\varepsilon$ & \ref{app:lowdemand} & The derivation's recall and false-positive rate \\
$\theta_u$, $\bar\theta$ & \ref{app:scoring} & A quantity on subject $u$, and its macro-average \\
\midrule
\multicolumn{3}{@{}l}{\emph{Measurement}} \\
$\Delta$ & \ref{app:lowdemand} & The response-score difference between C2 and C1 \\
$\gamma_1$, $\gamma_0$ & \ref{app:lowdemand} & Its means on probes that do and do not require retrieval \\
$\phi$, $D(\phi)$ & \ref{app:lowdemand} & A feature map of the probe alone, and its surface detectability \\
$I(\cdot\,;\cdot)$ & \ref{app:lowdemand} & Mutual information \\
$k_1$, $b$ & \ref{app:baselines} & The \textsc{bm25} saturation and length parameters \\
$\mathrm{RC}(C)$ & \ref{app:scoring} & The response score under condition $C$ \\
\bottomrule
\end{tabular}
\caption{Symbols.}
\label{tab:symbols}
\end{table}

\begin{table}[h]
\centering\small
\begin{tabular}{lp{9.1cm}}
\toprule
Term & Meaning in this paper \\
\midrule
Probe & A user turn taken verbatim from the transcript, never authored \\
Gold response & A reference reply written from the probe and the recorded context and nothing else \\
Reference set & The records the gold response used, disjoint from the recent window by construction \\
Locus & Which store holds the referent; a judgment, unlike the tier \\
Tier & The shape of the three context lists; a function of them, and re-computable \\
Category & A subdivision of a tier, assigned by the same rule \\
Memory-bearing & A probe whose reference set is not empty \\
Recorded reading & Every memory-bearing probe, as the derivation labeled it \\
Strict reading & The same, less the probes whose referent was already on screen \\
Proportional stratum & Probes drawn without regard to memory demand; the only stratum that carries a rate \\
Enriched stratum & A census of the corpus's memory-bearing probes; carries power, not a denominator \\
Abstention stratum & Cold opens built to enlarge the restraint population; selected, so it carries no rate \\
Derivation & The five phases that produce an item's labels \\
Phase & One step of the derivation, recording its decision and its reason together \\
Demotion & The effect of a verification removing evidence, so that a later phase reads a lower tier \\
Withdrawal & Removal of an item from scoring, with its reason shipped \\
Dispute marker & A recorded disagreement that changes no label \\
Gate & A predicate over the released files that the release is blocked on \\
Audit & An independent grading of the labels by a model that built none of them \\
\bottomrule
\end{tabular}
\caption{Terminology.}
\label{tab:terms}
\end{table}

\end{document}